\documentclass[twoside,11pt]{article}

\usepackage{blindtext}

\usepackage{amsmath,amsfonts,bm}

\def\eqref#1{equation~\ref{#1}}

\def\1{\bm{1}}

\def\rvt{{\mathbf{t}}}

\def\rvx{{\mathbf{x}}}
\def\rvy{{\mathbf{y}}}

\def\vzero{{\bm{0}}}
\def\vone{{\bm{1}}}
\def\vmu{{\bm{\mu}}}
\def\vtheta{{\bm{\theta}}}

\def\vg{{\bm{g}}}
\def\vh{{\bm{h}}}

\def\vs{{\bm{s}}}

\def\vv{{\bm{v}}}

\def\vx{{\bm{x}}}
\def\vy{{\bm{y}}}
\def\vz{{\bm{z}}}

\def\mI{{\bm{I}}}

\def\mSigma{{\bm{\Sigma}}}

\DeclareMathAlphabet{\mathsfit}{\encodingdefault}{\sfdefault}{m}{sl}
\SetMathAlphabet{\mathsfit}{bold}{\encodingdefault}{\sfdefault}{bx}{n}

\newcommand{\KL}{D_{\mathrm{KL}}}

\usepackage[abbrvbib]{jmlr2e}

\usepackage{lastpage}
\jmlrheading{0}{0000}{1-\pageref{LastPage}}{00/00; Revised 00/00}{00/00}{00}{Wei Chen, Qibin Zhao, John Paisley, Junmei Yang and Delu Zeng}

\ShortHeadings{One-Step Score-Based Density Ratio Estimation}{Chen, Zhao, Paisley, Yang, and Zeng}
\firstpageno{1}

\usepackage{thm-restate}
\usepackage{hyperref}

\usepackage{algorithm}
\usepackage{algorithmic}
\usepackage{booktabs}
\usepackage{multirow} 
\usepackage{makecell}
\usepackage{listings}
\usepackage{alltt} 
\usepackage[strings]{underscore} 
\usepackage[table]{xcolor}

\usepackage{subcaption}

\usepackage[capitalize]{cleveref} 
\crefname{lemma}{Lemma}{Lemmas}
\Crefname{lemma}{Lemma}{Lemmas}
\crefname{proposition}{Proposition}{Propositions}
\Crefname{proposition}{Proposition}{Propositions}
\crefname{corollary}{Corollary}{Corollaries}
\Crefname{corollary}{Corollary}{Corollaries}
\newtheorem{assumption}{Assumption}
\crefname{assumption}{Assumption}{Assumptions}
\Crefname{assumption}{Assumption}{Assumptions}

\crefname{section}{Sec.}{Secs.}
\Crefname{section}{Section}{Sections}
\crefname{table}{Tab.}{Tabs.}
\Crefname{table}{Table}{Tables}
\crefname{figure}{Fig.}{Figs.}
\Crefname{figure}{Figure}{Figures}
\crefname{equation}{Eq.}{Eqs.}
\Crefname{equation}{Equation}{Equations}

\begin{document}

\title{One-Step Score-Based Density Ratio Estimation}

\author{\name Wei Chen \email weichen001.work@foxmail.com \\
       \addr School of Mathematics, South China University of Technology, Guangzhou 510006, China 
       \AND
       \name Qibin Zhao \email qibin.zhao@riken.jp \\
       \addr Tensor Learning Team, RIKEN Center for Advanced Intelligence Project, Tokyo 103-0027, Japan 
       \AND
       \name John Paisley \email jpaisley@columbia.edu \\
       \addr Department of Electrical Engineering, Columbia University, New York, NY 10027, USA 
       \AND
       \name Junmei Yang \email yjunmei@scut.edu.cn \\
       \addr School of Electronic and Information Engineering, South China University of Technology, Guangzhou 510006, China 
       \AND
       \name Delu Zeng\thanks{Correspondence to Delu Zeng (dlzeng@scut.edu.cn)}  \email dlzeng@scut.edu.cn \\
       \addr School of Electronic and Information Engineering, South China University of Technology, Guangzhou 510006, China 
       }

\editor{Mahdi Soltanolkotabi}

\maketitle

\begin{abstract}
Density ratio estimation (DRE) is a useful tool for quantifying discrepancies between probability distributions, but existing approaches often involve a trade-off between estimation quality and computational efficiency. 
Classical direct DRE methods are usually efficient at inference time, yet their performance can seriously deteriorate when the discrepancy between distributions is large.
In contrast, score-based DRE methods often yield more accurate estimates in such settings, but they typically require considerable repeated function evaluations and numerical integration.
We propose One-step Score-based Density Ratio Estimation (OS-DRE), a partly analytic and numerical solver-free framework designed to combine these complementary advantages. 
OS-DRE decomposes the time score into spatial coefficients and a finite analytic radial basis function (RBF) dictionary, motivated by a Hilbert-space frame representation.
This formulation converts the otherwise intractable temporal integral into a closed-form weighted sum, thereby removing the need for numerical solvers and enabling DRE with \emph{only one} function evaluation.
We establish denseness and approximation error bounds for the finite RBF family in both finitely and infinitely smooth regimes, and give a conditional finite-sample error decomposition for learned coefficients.
Experiments across density estimation, continual Kullback-Leibler and mutual information estimation, and near out-of-distribution detection demonstrate that OS-DRE offers a favorable balance between estimation quality and inference efficiency.
\end{abstract}

\begin{keywords}
  density ratio estimation, score-based methods, probabilistic inference, temporal approximation, radial basis function
\end{keywords}

\clearpage

\newlength{\originalparskip} 
\setlength{\originalparskip}{\parskip} 

\renewcommand{\contentsname}{Contents}
\setcounter{tocdepth}{3} 

\begingroup
\hypersetup{linkcolor=black} 
\setlength{\parskip}{0.14em} 
\tableofcontents
\endgroup

\setlength{\parskip}{\originalparskip} 

\section{Introduction}

Estimating discrepancies between probability distributions is a recurrent problem in machine learning and statistics. A particularly useful quantity is the \emph{density ratio} $r(\vx)=p_1(\vx)/p_0(\vx)$ \citep{sugiyama2012density}, which quantifies how much more likely an event $\vx$ is under one distribution than under another. This quantity supports a broad range of applications, including likelihood-free Bayesian inference \citep{thomas2022likelihood}, 
continuous covariate-shift adaptation~\citep{zhang2023adapting}, 
mutual information (MI) estimation~\citep{letizia2024mutual,chen2025dequantified}, out-of-distribution (OOD) detection \citep{ren2019likelihood,zhang2022ood}, 
contrastive representation learning~\citep{sasaki2022representation}, 
large language model alignment~\citep{higuchi2025direct}, 
and causal inference~\citep{wang2025projection}.
Despite this broad utility, classical density ratio estimation (DRE) can be brittle in modern settings. 
When $p_0$ and $p_1$ overlap only weakly~\citep{srivastava2023estimating,chen2025dequantified} or differ sharply in high-discrepancy settings~\citep{rhodes2020telescoping,wang2025projection}, direct ratio estimators can encounter the density-chasm problem, in which insufficient shared support makes comparison unstable and inaccurate.  

A notable recent development is the emergence of continuous, score-based DRE methods \citep{choi2022density, chen2026dont}, which express the log-density ratio as a temporal integral along an interpolating family of densities $\{p_t\}_{t\in[0,1]}$ connecting $p_0$ and $p_1$:
\begin{equation} \label{eq:path-integral-representation}
\log r(\vx)=\log \frac{p_1(\vx)}{p_0(\vx)}=\log p_1(\vx) - \log p_0(\vx) =\int_0^1 \partial_t \log p_t(\vx)\mathrm{d}t,
\end{equation}
where $\partial_t\log p_t(\vx)$ denotes the time score, $\vx$ is the data variable, and $t$ denotes the temporal point along the integration path.
This formulation can alleviate the density-chasm problem by avoiding direct comparison between poorly overlapping densities and instead traversing a path of intermediate distributions that are better behaved. 
For any sufficiently regular path with endpoints $p_0$ and $p_1$, the log-density ratio admits the path-integral representation in \cref{eq:path-integral-representation}. 
In practice, explicit evaluation of $p_t$ is unnecessary; it is sufficient to specify an interpolation scheme that allows sampling during training. 
For example, the variance-preserving (VP) path \citep{song2020score} induces well-behaved intermediate distributions and thereby supports the integral representation in \cref{fig:illustration-of-interpolated-path}.
\begin{figure}[ht]
    \centering
    \includegraphics[width=0.98\linewidth]{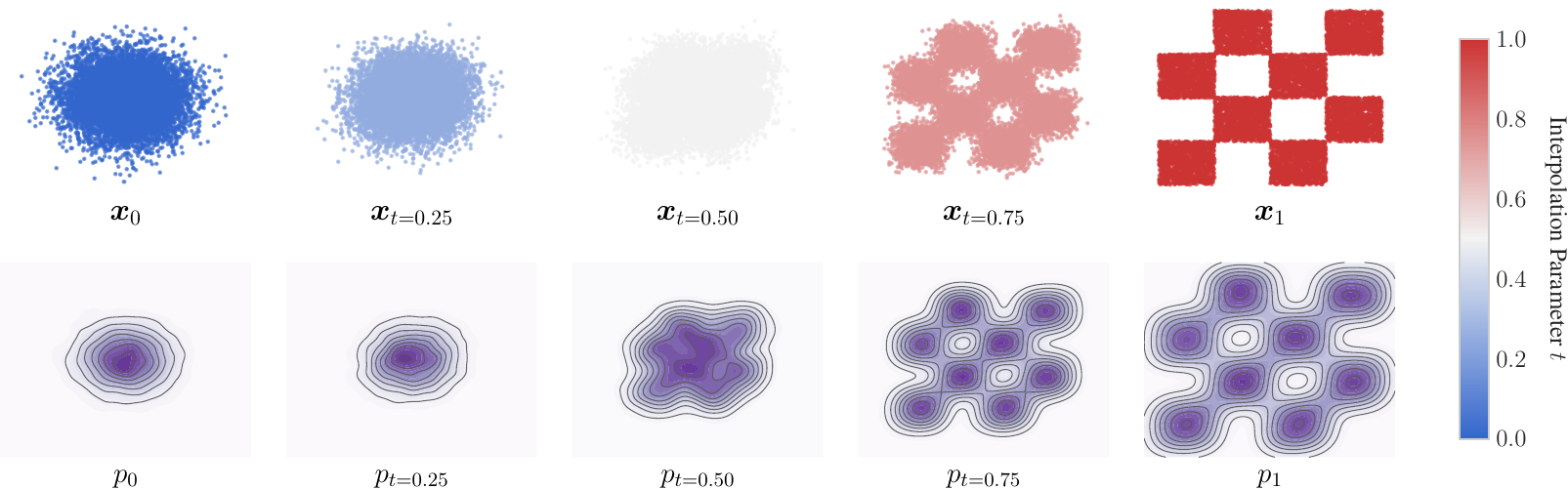}
    \vspace{-2mm}
    \caption{Visualization of the variance-preserving interpolating path $\{p_t\}_{t\in[0,1]}$ connecting a standard Gaussian ($p_0$) to a checkerboard distribution ($p_1$). Top: sample evolution over time. Bottom: corresponding density contours. }
    \label{fig:illustration-of-interpolated-path}
\end{figure}

Although this continuous formulation alleviates density-chasm issues, existing score-based methods still rely on costly numerical integration, such as ordinary differential equation (ODE) solvers \citep{choi2022density} or fine-grained quadrature schemes \citep{norcliffe2023faster}. Obtaining reliable estimates usually requires repeated evaluations of the learned time-score model across many time steps, leading to a large number of function evaluations (NFE) and substantial computational overhead \citep{chen2025dequantified}. In practice, this creates a persistent tension between estimation quality and inference efficiency. 

In this work, we propose One-step Score-based Density Ratio Estimation (OS-DRE), a partly analytic solver-free framework that approximates the path integral in a single step. The central idea is to replace numerical integration of the time score with an analytic temporal representation, where we succeed in separating the time score into spatial coefficients and temporal basis functions, that is:
\begin{equation}
	\log r(\vx)=\int_{0}^{1}\partial_t\log p_{t}(\vx) \mathrm{d}t \approx \sum_{k=1}^{K} h_k(\vx) \bar g_k = \left \langle \vh_{1:K}(\vx),\bar{\vg}_{1:K} \right \rangle,  
\end{equation}
where $\bar{\vg}_{1:K}=\left[\bar{g}_1,\bar{g}_2,\ldots,\bar{g}_K\right]$ with $\bar g_k=\int_0^1 g_k(t) \mathrm{d}t$ are the integrals of predefined temporal basis $\{g_k\}_{k=1}^K$, available in closed form and precomputed.
Then during inference procedure, a single forward pass of a neural network (blue squares in \cref{fig:baseline-illustration}) produces  $\vh_{1:K}(\vx)=[h_1(\vx),\ldots,h_K(\vx)]$,
and the integral would reduce to one inner product, yielding $\text{NFE}=1$ with the proposed OS-DRE in (ii) of \cref{fig:baseline-illustration}.
In addition, analytic temporal bases admit closed-form derivatives, which as well allow efficient computation of the $\partial_t^2 \log p_t(\vx)$  in the training objective (\cref{eq:joint-score-matching-loss}) without higher-order automatic differentiation.  

\begin{figure}[ht]
    \centering
    \includegraphics[width=0.98\linewidth]{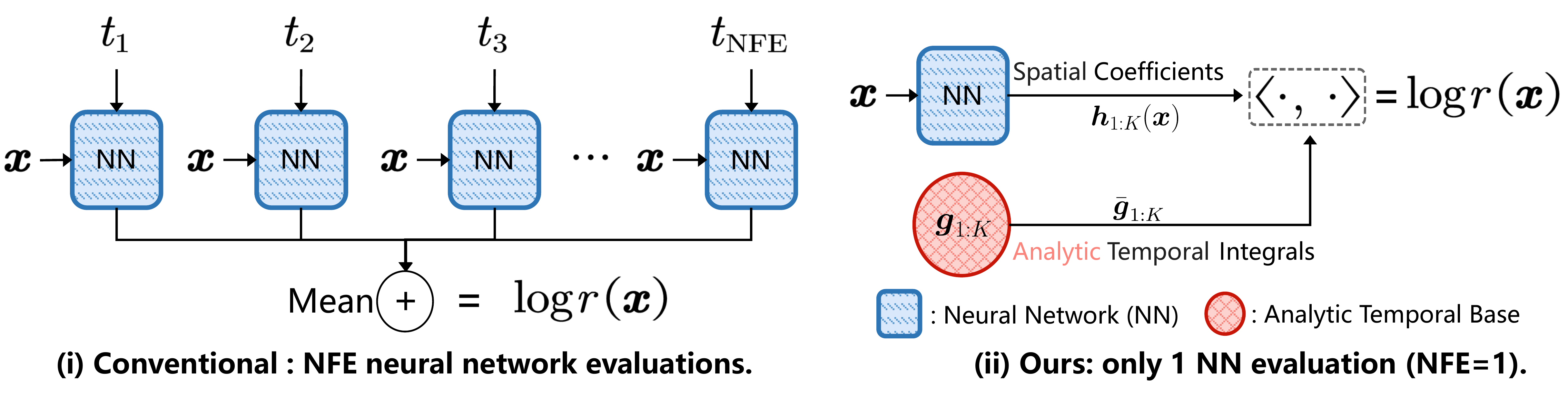}
    \caption{Illustrative comparison of conventional (i) to calculate $\int_{0}^{1}\partial_t\log p_{t}(\vx) \mathrm{d}t$ directly by some traditional numerical solvers and proposed (ii) score-based DRE method (OS-DRE). The NFE of conventional methods depends on the chosen numerical solver, whereas OS-DRE requires only one single neural network evaluation.
    }
    \label{fig:baseline-illustration}
\end{figure}

For standard orthonormal systems on $[0,1]$ containing the constant function as one mode, orthogonality forces every other mode to have zero temporal integral, so only that mode contributes directly to the log-density ratio $\log r(\vx)=\int_0^1 \partial_t\log p_t(\vx)\,\mathrm dt$ (see \cref{subse:one-step-estimation}). This motivates non-orthogonal representations through frame theory \citep{mallat1999wavelet}, which we instantiate with finite RBF dictionaries whose integrals and derivatives are available in closed form.
{Their approximation properties are analyzed in \cref{sec:rbf-framework-and-analysis}: denseness as $K$ grows ensures asymptotic identifiability, finite-term linear independence at a selected $K$ preserves coefficient identifiability, and the resulting error bounds characterize the trade-off between approximation quality and computational cost.}

Empirically, OS-DRE provides a favorable balance between estimation quality and computational efficiency. Across structured and multimodal density-estimation benchmarks, it produces sharp and faithful densities with $\text{NFE}=1$, whereas solver-based baselines, such as DRE-$\infty$ \citep{choi2022density} and D$^3$RE \citep{chen2025dequantified}, can still be unsatisfactory even at $\text{NFE}=10$ (\cref{fig:likelihood-estimation-2d-synthesic-method}). 
On real-world tabular datasets, OS-DRE achieves competitive negative log-likelihood while reducing inference time by up to $68\times$ relative to methods that require $\text{NFE}=75$ (\cref{tab:density-estimation-tabular}).  
In $f$-divergence estimation tasks, including continual Kullback--Leibler (KL) divergence estimation and mutual information estimation, the proposed method supports stable real-time tracking and remains reliable in high-discrepancy settings where baselines can exhibit lag or instability (\cref{fig:distributon-drift-ds,fig:mi_mse_pathological,tab:MI-high-discrepancy}). 
OS-DRE also scales to challenging near OOD detection tasks in high-dimensional image spaces, where it produces clear separation between in-distribution (ID) and OOD samples with competitive AUROC (\cref{tab:ood,fig:density-of-ood-scores-cifar100}). Taken together, these results suggest that OS-DRE can retain much of the estimation quality of score-based models while approaching the inference efficiency of classical neural DRE methods under strict evaluation budgets (\cref{fig:accuracy-vs-efficiency}).

The main contributions of this work are:
\begin{itemize}
    \item We propose OS-DRE, a solver-free score-based DRE framework that approximates the path integral in one step and aims to combine the estimation quality of score-based methods with the efficiency of classical one-pass estimation.
	
    \item We use Hilbert-space frame theory to justify a spatiotemporal representation of the time score and the interchange of temporal integration with norm-convergent expansions. A finite analytic dictionary converts the approximated path integral into one inner product and supports analytic temporal derivatives for training.

    \item We establish denseness of the finite RBF approximation family,
    conditions for unique finite coefficient representations, and analytic temporal
    operators. The construction yields approximation error bounds for both finitely and
    infinitely smooth regimes used by OS-DRE. {A finite-sample analysis further relates the estimation error of the learned estimator to the dictionary size, the sample size and the coefficient class.}

    \item Extensive experiments show that OS-DRE achieves competitive performance with reduced computational cost across density estimation, continual KL divergence estimation, OOD detection, and MI estimation benchmarks.
\end{itemize}

\paragraph{Theory Roadmap.}
{Section~\ref{sec:os-dre} derives the separable score representation and
justifies temporal integration and weak differentiation under stated convergence
conditions. Section~\ref{sec:rbf-framework-and-analysis} constructs finite analytic
RBF dictionaries and establishes their approximation capacity and conditions for
coefficient uniqueness. Section~\ref{sec:approximation-estimation} quantifies temporal
approximation error and relates it to learned log-ratio error under explicit
statistical assumptions. Proofs and supplementary conditions are given in the appendix.}

\vspace{-2mm}
\section{Related Works}
\label{sec:related-works}

\vspace{-2mm}
Existing methods for DRE can be broadly categorized by how directly they compare the endpoint densities. Early direct approaches either estimate the two densities separately, e.g.,, by kernel density estimation \citep{huang2006correcting}, and then form their ratio, or estimate the ratio itself without modeling each density individually \citep{kanamori2010theoretical}. Representative direct estimators include KLIEP \citep{sugiyama2008direct} and uLSIF \citep{kanamori2009least}, together with broader formulations based on Bregman divergence and $f$-divergence \citep{sugiyama2012density}. 
A major conceptual advance was the connection between DRE and class-probability estimation, which showed that minimizing suitable proper composite losses in binary classification yields consistent ratio estimators \citep{menon2016linking}. 
This perspective led to discriminative and neural DRE methods such as NCE \citep{gutmann2012noise} and InfoNCE \citep{oord2018representation}, and also motivated subsequent analyses of truncation \citep[TR-DRE]{liu2017trimmed}, bias in classification-based ratio learning \citep{song2020understanding}, and DRE with missing data \citep{givens2023density}. More recently, theoretical work has begun to clarify the statistical behavior of deep direct estimators. 
For example, \citet{zheng2026error} establish non-asymptotic error bounds for deep ReLU \citep{agarap2018deep} DRE under Bregman divergence \citep{banerjee2005clustering} and extend the analysis to unbounded settings, with connections to KL estimation and telescoping estimators. 
Nevertheless, these methods still compare the two endpoint densities directly, which have been discovered to probably lead to numerical instability and sensitivity to regularization. Particularly, some literature discuss therapies when the distributions have weak overlap \citep{rhodes2020telescoping,chen2025dequantified}, large discrepancy \citep{srivastava2023estimating}, or complex high-dimensional geometry \citep{wang2025projection}.

A substantial body of work therefore seeks to improve the estimation quality of direct DRE while retaining its basic formulation. A central challenge is the density-chasm problem, in which the supports of $p_0$ and $p_1$ overlap poorly, making direct ratio learning unstable and biased \citep{liu2017trimmed,rhodes2020telescoping}. Existing remedies include trimmed DRE \citep{liu2017trimmed}, adaptive and regularized Bregman objectives \citep{zellinger2023adaptive, nguyen2024regularized}, iterated regularization to improve convergence behavior in kernel settings \citep{gruber2024overcoming}, robust objectives such as $\gamma$-DRE \citep{nagumo2024density}, binary-loss refinements that reduce bias toward small density ratios \citep{zellinger2025binary}, and projection-based methods that transform the data into more favorable representation spaces for ratio estimation, either by identifying low-dimensional subspaces \citep{wang2025projection} or by mapping samples into a learned feature space before estimation, as in FDRE \citep{choi2021featurized}. Although these methods substantially strengthen the classical pipeline, they remain variants of direct endpoint-ratio estimation and therefore do not substantially  mitigate the fundamental difficulty of comparing $p_0$ and $p_1$ head-on.

A different line of work addresses this issue by introducing intermediate distributions, thereby replacing a difficult global ratio problem with a sequence or continuum of easier local problems. In the discrete setting, TRE \citep{rhodes2020telescoping} decomposes the target ratio into a telescoping product of intermediate ratios, and this idea has subsequently been extended through adaptive decomposition in AM-DRE \citep{xiao2022adaptive}, multinomial formulations in MDRE \citep{srivastava2023estimating}, and generalized-geodesic interpolation in IMDRE \citep{kimura2025density}. The continuous counterpart is score-based DRE, in which the log-density ratio is expressed as a path integral of a time-dependent score field along an interpolation between $p_0$ and $p_1$ \citep{choi2022density}. Because this formulation avoids direct comparison of the endpoint distributions and instead estimates smoother local changes along the path, it often yields improved estimation quality under high discrepancy \citep{choi2022density, chen2025dequantified}. Within this framework, the choice of interpolation path and training objective can affect both estimation quality and optimization stability. Several studies therefore examine how to improve these aspects, for example by designing lower-variance training objectives \citep{chen2025diffusion}, modifying objective formulations to accelerate training \citep{yu2025density}, or proposing principled criteria for selecting more suitable probability paths \citep{chen2025dequantified,kimura2025density}. More recently, \citet{chen2026dont} propose heuristic path-adaptation strategies motivated by minimum-variance principles. Despite these advances, inference in score-based DRE still typically requires numerical quadrature \citep{chen2026dont} or ODE solvers \citep{choi2022density,yu2025density}, which leads to substantial computational cost \citep{li2025evodiff}.

In summary, direct estimators are efficient but degrade under large discrepancy, while score-based methods are more accurate but computationally expensive. OS-DRE bridges this gap by retaining the estimation quality of score-based DRE while removing numerical integration via partly analytic temporal integration. Closely related, \citet{guth2025learning} study time-varying energy modeling via dual score matching, although their setting differs from general DRE between non-Gaussian distributions. Overall, OS-DRE combines the estimation quality of score-based methods with the efficiency of direct DRE methods.

\section{One-Step Score-Based Density Ratio Estimation}
\label{sec:os-dre}

\subsection{Problem Statement}
Score-based DRE rewrites the density ratio $r(\vx)=p_1(\vx)/p_0(\vx)$ as a temporal integral.
Given a continuously interpolating path of densities $\{p_t(\vx)\}_{t\in[0,1]}$ from $p_0$ to $p_1$, we have
\begin{equation} \label{eq:continuous-log-density-ratio}
	\log r(\vx) = \log \frac{p_1(\vx)}{p_0(\vx)}=\int_0^1 \partial_t \log p_t(\vx)\,\mathrm{d}t.
\end{equation}

The integrand $\partial_t \log p_t(\vx)$, called the \emph{time score}, is unknown in practice and is estimated by a neural network trained via time score matching \citep{choi2022density}.
Evaluating $\log r(\vx)$ then requires numerically integrating the trained network, which is the main computational bottleneck as it demands many forward passes (e.g., via ODE solvers \citep{choi2022density} or quadrature methods \citep{chen2026dont,norcliffe2023faster}).
Our goal is to eliminate this bottleneck by reducing the integral to a finite weighted sum via a temporal dictionary,
so that the weights can be precomputed once and reused across all $\vx$.



\subsection{One-Step Estimation via Temporal Frames}
\label{subse:one-step-estimation}

\paragraph{Assumptions and Setup.}
Our analysis assumes mild regularities on the density $p_t$. Particularly, given a measurable bounded data feature space $\mathcal{X}$, we make

\begin{assumption}
    \label{assumption:logp-positive}
    There exists $ +\infty>C > 0 $ such that $ p_t(\vx) \geq C $ for all $ \vx \in \mathcal{X} $ and $ t \in [0,1] $.
\end{assumption}

\begin{assumption}
\label{assumption:logp-partial-derivative-bounded}
There exists $+\infty>D>0$ such that $\left|\partial_t p_t(\vx)\right|\le D$ for all $\vx\in\mathcal{X}$ and $t\in[0,1]$.
\end{assumption}

Assumption~\ref{assumption:logp-positive} ensures $\log p_t(\vx)$ is well-defined, and
Assumptions~\ref{assumption:logp-positive}--\ref{assumption:logp-partial-derivative-bounded} imply $|s^{(\rvt)}(\vx,t)|\le D/C$. Thus the time score belongs to the Hilbert space $\mathcal{H}_{\vx,t}\triangleq L^2(\mathcal{X}\times[0,1])$, as stated in Lemma~\ref{lemma:space-S-subspace-L2}.
Let $\mathcal{S}\triangleq\{s^{(\rvt)}: s^{(\rvt)}(\vx,t)=\partial_t\log p_t(\vx)\}$ be the set of time score functions.
Throughout this paper, \emph{superscripts $(\rvt)$ and $(\rvx)$ indicate time and data}, respectively.
Under the stated assumptions, the space $\mathcal{S}$ is a subset of   $\mathcal{H}_{\vx,t}$:

\begin{restatable}{lemma}{SpaceSisSubspaceL}
\label{lemma:space-S-subspace-L2}
Under \cref{assumption:logp-positive,assumption:logp-partial-derivative-bounded}, the space $\mathcal{S}$ is a subset of $L^2(\mathcal{X} \times [0,1])$.
\end{restatable}
See proof in \cref{proof:space-S-subspace-L2}. This embedding allows us to leverage the tools of Hilbert space theory to analyze and approximate the time score function $s^{(\rvt)}$. 

We will use the standard identification of the Hilbert space $\mathcal{H}_{\vx, t}$ with the Hilbert tensor product $\mathcal{H}_{\vx} \hat{\otimes} \mathcal{H}_t$ \citep{kadison1986fundamentals}, where $\mathcal{H}_{\vx} \triangleq L^2(\mathcal{X})$ and $\mathcal{H}_t \triangleq L^2([0,1])$. This equivalence guarantees that any time score $s^{(\rvt)} \in \mathcal{S}$ can be represented by separating its spatial and temporal components (see Lemma~\ref{lemma:tensor-product-of-frames}).

\paragraph{A First Attempt: Orthonormal Temporal Bases and Its Degeneracy.}
We first propose to represent the temporal component of the time score using a complete orthonormal basis $\{g_k\}_{k=1}^{+\infty}$ for $\mathcal{H}_t$. This decomposition, analogous to the Karhunen-Loève expansion \citep{karhunen1947lineare,loeve1977elementary}, allows us to express the time score $s^{(\rvt)}(\vx, t)$ for each fixed $\vx$ as a weighted sum of its spatial and temporal components:
\begin{equation}\label{eq:orthonormal-bases-approximationforscore}
	s^{(\rvt)}(\vx, t) = \sum_{k=1}^{\infty} h_k(\vx) g_k(t), 
\end{equation}
where $h_k(\vx) = \langle s^{(\rvt)}(\vx, \cdot), g_k \rangle_{\mathcal{H}_t}\triangleq\int_0^1 s^{(\rvt)}(\vx, t) g_k(t) \mathrm{d}t$.
By integrating this series in \cref{eq:orthonormal-bases-approximationforscore} with respect to $t$, we derive our initial formulation for the $\log r(\vx)$.
\begin{restatable}{lemma}{LogDensityRatioDecomposition}
\label{lemma:log-density-ratio-decomposition}
Let $\{g_k\}_{k=1}^{+\infty}$ be a complete orthonormal basis for $\mathcal{H}_t$. Then
\begin{equation}\label{eq:orthonormal-bases-approximation}
\log r(\vx)=\sum_{k=1}^{\infty} h_k(\vx)\int_0^1 g_k(t)\,\mathrm{d}t.
\end{equation}
\end{restatable}
Proof in  \cref{proof:log-density-ratio-decomposition}.
In a common class of orthonormal systems on $  [0,1]  $ that include the constant function, orthogonality forces all other modes to have zero mean, so only the constant mode contributes to the one-step integral. Formally, suppose the orthonormal temporal system contains the normalized constant mode $g_1(t)\equiv1$. Then, for every $k>1$,
\begin{equation}
 \label{eq:orthonormal-zero-mean}
 \bar g_k
 =\int_0^1g_k(t)\,\mathrm dt
 =\langle g_k,g_1\rangle_{\mathcal H_t}
 =0.
\end{equation}
Hence only the constant mode contributes to the integrated representation.

\paragraph{Non-Orthogonal Temporal Representation.}
To avoid the zero-integral collapse, we relax orthogonality and use a \emph{frame} in $\mathcal{H}_t$.

\begin{definition}[Frame, \citet{mallat1999wavelet}]
	Let $\mathcal{H}$ be a Hilbert space with  inner product $\langle \cdot, \cdot \rangle_{\mathcal{H}}$. 
	A sequence $\{g_k\}_{k=1}^{+\infty}$ in a Hilbert space $\mathcal{H}$ is a frame if there exist constants $0<A\le B<\infty$, called the \textit{frame bounds}, such that for any $g \in \mathcal{H}$:
	\begin{equation}
		A \|g\|_{\mathcal{H}}^2 \leq \sum_{k=1}^{+\infty} |\langle g, g_k \rangle_{\mathcal{H}}|^2 \leq B \|g\|_{\mathcal{H}}^2.
	\end{equation}
\end{definition}

We will use frames in both $\mathcal{H}_{\vx}$ and $\mathcal{H}_t$. Their tensor products form a frame for $\mathcal{H}_{\vx,t}$:
\begin{lemma}
\label{lemma:tensor-product-of-frames}
Let $\{f_l\}_{l=1}^{+\infty}$ be a frame for $\mathcal{H}_{\vx}$ with bounds $A_f,B_f$, and $\{g_k\}_{k=1}^{+\infty}$ be a frame for $\mathcal{H}_t$ with bounds $A_g,B_g$.
Then, the set of elementary tensors $\{f_l\otimes g_k\}_{l,k=1}^{+\infty}$ is a frame for $\mathcal{H}_{\vx,t}$ with bounds $A_fA_g$ and $B_fB_g$.
\end{lemma}
Proof in \cref{proof:tensor-product-of-frames}. This lemma provides the direct theoretical justification for the spatiotemporal expansion used in \cref{theorem:frame-based-representation}, allowing us to represent any time score function $s^{(\rvt)} \in \mathcal{S} \subseteq \mathcal{H}_{\vx,t}$ as a double summation over the elementary tensors $f_l\otimes g_k$.

With frames, we obtain a stable spatiotemporal expansion in the Hilbert-frame sense. For one-step DRE, we additionally choose temporal items whose integrals are nonzero and available in closed form.
\begin{restatable}{theorem}{FrameBasedRepresentation}
\label{theorem:frame-based-representation}
Let $\{f_l\}_{l=1}^{+\infty}$ and $\{g_k\}_{k=1}^{+\infty}$ be frames for $\mathcal{H}_{\vx}$ and $\mathcal{H}_t$, respectively.
Then any time score function $s^{(\rvt)}\in\mathcal{S}\subseteq \mathcal{H}_{\vx,t}$ admits an expansion
\begin{equation}
    s^{(\rvt)}(\vx, t) = \sum_{k=1}^{+\infty} \sum_{l=1}^{+\infty} c_{l, k} f_l(\vx) g_k(t),
\end{equation}
with coefficients $c_{l,k}=\langle s^{(\rvt)},\widetilde f_l\otimes\widetilde g_k\rangle_{\mathcal H_{\vx,t}}$, where $\{\widetilde f_l\}$ and $\{\widetilde g_k\}$ are the canonical dual frames. The expansion converges unconditionally in $\mathcal H_{\vx,t}$.
Define $h_k\triangleq\sum_{l=1}^{+\infty}c_{l,k}f_l\in\mathcal H_{\vx}$, with convergence in $\mathcal H_{\vx}$, and $\bar g_k\triangleq\int_0^1 g_k(t)\mathrm{d}t$.
Then
\begin{align}
s^{(\rvt)}(\vx,t)&=\sum_{k=1}^{+\infty} h_k(\vx)g_k(t), \\
\log r(\vx)&=\int_0^1 s^{(\rvt)}(\vx,t) \mathrm{d} t=\sum_{k=1}^{+\infty} h_k(\vx)\bar g_k. \label{eq:log-ratio-frame-decomposition-single-sum}
\end{align}
The two series converge in $\mathcal H_{\vx,t}$ and $L^2(\mathcal X)$, respectively; the equalities are identities of $L^2$ equivalence classes.
\end{restatable}
Proof in \cref{proof:log-density-ratio-frame-based}.
The frame formulation relaxes the orthogonality constraint responsible for the collapse above. OS-DRE then selects temporal items with analytically tractable, nonzero integrals, allowing multiple temporal components to contribute to the one-step weighted sum.
Crucially, the weights $\bar g_k$ depend only on the chosen temporal frame and can be precomputed once.
The exact identity in \cref{eq:log-ratio-frame-decomposition-single-sum} is an infinite series. The finite approximation constructed next yields a computable weighted sum without numerical quadrature.

Furthermore, Corollary~\ref{corollary:weak-derivative-of-time-score} allows for the computation of derivatives as well (see proof in \cref{proof:weak-derivative-of-time-score}).

\begin{restatable}{corollary}{WeakDerivative}
\label{corollary:weak-derivative-of-time-score}
Under the hypotheses and coefficient choice of Theorem~\ref{theorem:frame-based-representation}, if each $g_k\in \mathcal W^{1,2}([0,1])$ (see \cref{appendix:kernel-conventions} for the Sobolev convention) and the series
$\sum_{k=1}^{+\infty} h_k(\vx)g_k^{\prime}(t)$ converges in $\mathcal{H}_{\vx,t}$,
then the weak temporal derivative exists in $\mathcal H_{\vx,t}$ and satisfies
\begin{equation}
    \partial_t s^{(\rvt)}(\vx, t) = \sum_{k=1}^{+\infty} h_k(\vx) g_k^{\prime}(t).
\end{equation}
\end{restatable}

\paragraph{Roadmap.}
The abstract frame representation $s^{(\rvt)}(\vx,t)=\sum_{k=1}^{+\infty} h_k(\vx)g_k(t)$ motivates a separable temporal model. {The spatial frame in \cref{theorem:frame-based-representation} serves only to expand the coefficient functions $h_k$ in $\mathcal{H}_{\vx}$, and OS-DRE parameterizes them with a neural network.} The separable form and the identity for $\log r$ depend on the temporal frame alone. The next section constructs finite RBF spaces and proves their approximation properties directly.

\section{Analytic RBF Representation for OS-DRE}
\label{sec:rbf-framework-and-analysis}

Motivated by the frame representation in \cref{sec:os-dre}, we construct
finite RBF dictionaries with analytic temporal operators. Their coefficients are
selected anew for each dictionary size. This section establishes approximation
capacity and coefficient uniqueness; Section~\ref{sec:approximation-estimation}
analyzes the resulting approximation and estimation errors.

\subsection{Finite Temporal Representation}
\label{subsec:approximation-scheme}

For each $K$, select temporal functions $\{g_k\}_{k=1}^K\subset\mathcal H_t$
and define $\mathcal V_K\triangleq\operatorname{span}\{g_k\}_{k=1}^K$, whose dimension is
at most $K$. For each fixed $\vx$, approximate $s^{(\rvt)}(\vx,\cdot)$ in this space:
\begin{align}
 s^{(\rvt)}(\vx,t)&\approx s_K^{(\rvt)}(\vx,t)
 \triangleq\sum_{k=1}^K h_k(\vx)g_k(t),
 \label{eq:log-ratio-frame-decomposition-finite_revised}\\
 \log r(\vx)&\approx\log r_K(\vx)
 \triangleq\int_0^1s_K^{(\rvt)}(\vx,t)\,\mathrm dt
 =\sum_{k=1}^K h_k(\vx)\bar g_k.
 \label{eq:one-step-estimation}
\end{align}
Here $\bar g_k=\int_0^1g_k(t)\,\mathrm dt$. In this section, $h_k$ denotes the optimized coefficient of the $K$-term finite expansion, in contrast to the prescribed coefficient of the infinite expansion in Theorem~\ref{theorem:frame-based-representation}.
The next subsection defines the RBFs using fixed centers and learnable shape
parameters, giving the spaces $\mathcal V_K(\boldsymbol\sigma_{1:K})$ in
\cref{eq:trainable-shape-parameter-vector}.

\paragraph{What Makes This Approximation Valid?}

We distinguish the ability to approximate a target to any prescribed accuracy
using finitely many temporal functions from uniqueness of the coefficients at a selected
finite $K$. Under the RBF instantiation below, Lemma~\ref{lemma:rbf-density} establishes
the first property for the chosen kernel, and Proposition~
\ref{proposition:rbf-approximation-scheme} summarizes it together with finite-term linear
independence for the corresponding RBF family.

\subsection{Temporal Approximation with RBFs}
\label{subsec:rbf-construction}

\paragraph{General RBF Construction.}
This subsection constructs the finite temporal dictionary from RBF kernel sections, making each function's center and shape parameter explicit.

Let $\phi:[0,\infty)\to\mathbb R$ be the generating function. The bivariate RBF kernel is
\begin{equation}
\label{eq:rbf-definition}
\kappa_{\phi,\sigma}(s,t)\triangleq
\phi\left(\frac{|s-t|}{\sigma}\right),
\qquad s,t\in[0,1],\ \sigma>0,
\end{equation}
and the temporal functions are its kernel sections
\begin{equation}
\label{eq:rbf-temporal-items}
g_k(t)=\kappa_{\phi,\sigma_k}(t,c_k),\qquad k=1,\ldots,K,
\end{equation}
with centers $c_k\in[0,1]$ and shape parameters $\sigma_k>0$. The integrals $\bar g_k=\int_0^1 g_k(t)\,\mathrm dt$ and derivatives $g_k'(t)$ are computed in closed form from $\phi$.

Fix an optimization domain $\mathcal A_K\subset(0,\infty)^K$ for the shape parameters and define the parameterized space
\begin{equation}
\label{eq:trainable-shape-parameter-vector}
\mathcal V_K(\boldsymbol\sigma_{1:K})\triangleq\operatorname{span}
\{\kappa_{\phi,\sigma_k}(\cdot,c_k)\}_{k=1}^{K},
\qquad \boldsymbol\sigma_{1:K}=(\sigma_1,\ldots,\sigma_K)\in\mathcal A_K.
\end{equation}

For comparison, let $\boldsymbol\sigma_{0,1:K}\triangleq(\sigma_0,\ldots,\sigma_0)$ with $\sigma_0>0$ fixed independently of $K$ and $\boldsymbol\sigma_{0,1:K}\in\mathcal A_K$; the corresponding fixed-scale space is $\mathcal V_K(\boldsymbol\sigma_{0,1:K})$. The native space $\mathcal N_{\phi,\sigma_0}([0,1])$ and other kernel conventions are recorded in \cref{appendix:kernel-conventions}.

Next, we will demonstrate that this family is expressive enough in $\mathcal{H}_t$.

{\sloppy
\begin{lemma}
\label{lemma:rbf-density}
Let $\mathcal C_K=\{c_k\}_{k=1}^{K}\subset[0,1]$ and define its fill distance by
$h_K\triangleq\sup_{t\in[0,1]}\min_{c\in\mathcal C_K}|t-c|$.
Assume that $h_K\to0$ and that there is a $\sigma_0>0$, independent of $K$, for which
$\boldsymbol\sigma_{0,1:K}\in\mathcal A_K$ for every $K$ and
$\kappa_{\phi,\sigma_0}$ is continuous and $c$-universal on
$[0,1]$. Then, for every $f\in\mathcal H_t$ and $\varepsilon>0$, there exists a finite
$K_\varepsilon$ such that, for every $K\ge K_\varepsilon$, one can choose
$\boldsymbol\sigma_{1:K}\in\mathcal A_K$ and
$v_K\in\mathcal V_K(\boldsymbol\sigma_{1:K})$ satisfying
\begin{equation}
\label{eq:rbf-finite-approximation}
\|f-v_K\|_{\mathcal H_t}<\varepsilon.
\end{equation}
\end{lemma}
See the proof in \cref{proof:rbf-density}. Thus, any prescribed approximation accuracy is
attainable with a sufficiently large but finite temporal dictionary. Here $c$-universality
means that finite linear combinations of kernel functions can uniformly approximate every
continuous function on $[0,1]$; the formal definition is given in
\cref{appendix:kernel-conventions}.
}

\paragraph{\texorpdfstring{Finite-Term Linear Independence}{Finite-Term Linear Independence}.}

Denseness alone does not ensure coefficient nonredundancy at a fixed $K$. For OS-DRE, we formalize finite-term linear independence as a separate structural property.

\begin{restatable}{proposition}{RBFApproximationScheme}
\label{proposition:rbf-approximation-scheme}
Under the assumptions of Lemma~\ref{lemma:rbf-density}, the finite RBF family is dense in
$\mathcal H_t$ in the sense of \cref{eq:rbf-finite-approximation}. For any fixed $K$ and
admissible shape vector, every $f\in\mathcal H_t$ has a unique orthogonal projection onto
$\mathcal V_K(\boldsymbol\sigma_{1:K})$, and its coefficient vector is unique if and only if
the selected temporal functions are linearly independent.
\end{restatable}
See the proof in \cref{proof:rbf-approximation-scheme}. Distinct centers with a common
shape parameter and a strictly positive-definite kernel give a sufficient condition for
linear independence \citep{schaback1995error}.

\paragraph{Analytic Formulas for One-Step Estimation.}
In addition to approximation capacity, OS-DRE uses RBF families with analytic temporal integrals and derivatives.
For $g_k=\kappa_{\phi,\sigma_k}(\cdot,c_k)$ in \cref{eq:rbf-temporal-items},
\begin{align}
\label{eq:rbf_general_integral_derivative}
\bar g_k&=\int_0^1 \phi\left(\frac{|t-c_k|}{\sigma_k}\right)\mathrm{d}t, \\
g^{\prime}_k(t)&=\frac{\mathrm{sgn}(t-c_k)}{\sigma_k}
\phi^{\prime}\left(\frac{|t-c_k|}{\sigma_k}\right), \label{eq:rbf_general_integral_derivative2}
\end{align}
so once $\phi$ is fixed, $\bar g_k$ and $g_k^{\prime}(t)$ reduce to closed-form expressions and could be easily precomputed in advance (derived in \cref{appendix:rbf-formulas}).
This property is exactly what enables one-step inference in \cref{eq:one-step-estimation} (via precomputing $\{\bar g_k\}$) and as well partly analytic computation of $\partial_t s^{(\rvt)}(\vx, t)$ in training.

\paragraph{Summary of Basic Approximation Conditions.}
Before introducing specific kernel families, we distinguish the structural properties used in the approximation analysis and implementation.

\begin{remark}[Basic Approximation Conditions]\label{remark:basic-requirements}
The temporal construction separates three properties with different roles:
approximation capacity of the family, coefficient uniqueness at a fixed model, and
analytic temporal evaluation.
\begin{enumerate}
\item \textbf{Denseness:} any target in $\mathcal H_t$ can be approximated to arbitrary
accuracy using a sufficiently large but finite number of temporal RBFs, as established in
Lemma~\ref{lemma:rbf-density};
\item \textbf{Coefficient uniqueness:} when the selected finite functions are linearly
independent, coefficients in $\operatorname{span}\{g_k\}_{k=1}^K$ are unique, as
formalized in Proposition~\ref{proposition:rbf-approximation-scheme};
\item \textbf{Analytic tractability:} each temporal function has a closed-form integral
$\bar g_k$ and derivative $g_k'(t)$, which support one-step inference and analytic temporal
differentiation.
\end{enumerate}
\end{remark}

The following kernels provide denseness at a fixed admissible scale and analytic temporal
operators. Distinct centers at a common scale also provide coefficient uniqueness; for
non-uniform scales this property must be verified separately, and the approximation
bounds do not rely on it. The approximation rates address their different native-space
regularities.

\subsection{A Suite of Analytic RBF Families}
\label{subsec:rbf-kernels}

Herein we introduce a suite of analytic RBF families, results of which could be analytically precomputed by their closed forms of \cref{eq:rbf_general_integral_derivative} \cref{eq:rbf_general_integral_derivative2}. While the basic conditions outlined in Remark~\ref{remark:basic-requirements} ensure that an RBF family is valid, these kernels differ intrinsically in their degree of smoothness. 
As analyzed formally in \cref{subsec:error-analysis}, this regularity determines the theoretical decay rate of the approximation error.

To accommodate different types of target score dynamics, we consider several strictly positive definite kernels from two regularity regimes. 
The first is the \emph{finitely} smooth regime, represented by the Matérn family \citep{lamichhane2017particular}, which offers local adaptability and is therefore suitable for target scores with relatively sharp temporal variations. The second is the \emph{infinitely} smooth regime, which includes the Gaussian \citep{sandberg2003gaussian}, inverse multiquadric (IMQ) \cite{xiang2012trigonometric}, and rational quadratic (RQ) \citep{mure2021propriety} kernels. This class induces broader temporal influence and is associated with faster convergence rates when the target is sufficiently smooth.

\cref{tab:rbf_comparison} summarizes the kernels used in this work. Their
$c$-universality and analytic formulas are verified separately in
\cref{appendix:rbf-formulas}; their quantitative rates depend on the regularity assumptions
of \cref{subsec:error-analysis}.

\begin{table}[ht]
    \centering
    \scriptsize
    \setlength{\tabcolsep}{4pt}
    \caption{
    Representative analytic RBF generating functions used in OS-DRE. 
    \cref{appendix:rbf-formulas} verifies
    $c$-universality and gives closed-form temporal integrals and derivatives for
    each generating function. Quantitative approximation rates require the additional assumptions stated
    in Propositions~\ref{proposition:error-bound} and~\ref{proposition:error-bound-infinitely-smooth}.
    }
    \label{tab:rbf_comparison}
    \begin{tabular}{llllll}
        \toprule
        \textbf{Kernel} & \textbf{Generating} & \textbf{Smooth} & \textbf{Decay} & \textbf{Approximation}  & \textbf{Analytic formulas} \\
        \textbf{name}   & \textbf{function $\phi(r)$} & & \textbf{rate} & \textbf{theory} &  \\
        \midrule
        Matérn ($\nu=3/2$) & $(1 + \sqrt{3}r)\exp(-\sqrt{3}r)$ & Finitely  & Fast & Proposition~\ref{proposition:error-bound} & \cref{eq:integral-of-gk-matern-rbf,eq:derivative-of-gk-matern-rbf} \\
        Gaussian & $\exp(-r^2)$ & Infinitely & Fast & Proposition~\ref{proposition:error-bound-infinitely-smooth}(i) & \cref{eq:integral-of-gk-gaussian-rbf,eq:derivative-of-gk-gaussian-rbf} \\
        Inverse Multiquadric & $(r^2 + 1)^{-1/2}$ & Infinitely & Slow & Proposition~\ref{proposition:error-bound-infinitely-smooth}(ii) & \cref{eq:integral-of-gk-imq-rbf,eq:derivative-of-gk-imq-rbf} \\
        Rational Quadratic & $(r^2 + 1)^{-1}$ & Infinitely & Medium & Proposition~\ref{proposition:error-bound-infinitely-smooth}(iii) & \cref{eq:integral-of-gk-rq-rbf,eq:derivative-of-gk-rq-rbf} \\
        \bottomrule
    \end{tabular}%
\end{table}

To elaborate on the summary in \cref{tab:rbf_comparison}, we present the explicit temporal instantiations below. For each kernel, we specify the resulting temporal item $g_k(t)$ derived from $\phi$ and provide the corresponding closed-form expressions for its integral $\bar g_k$ and derivative $g_k^{\prime}(t)$.

\begin{example}[Matérn RBFs]
Finally, we include a finitely smooth option via the Matérn family.
With $\nu=3/2$, the generating function is $\phi(r)=(1+\sqrt{3}r)\exp(-\sqrt{3}r)$, giving $g_k(t)=\left(1+ \sqrt{3}|t-c_k| / \sigma_k \right)\exp\left(- \sqrt{3}|t-c_k| / \sigma_k\right)$.
Unlike the preceding infinitely smooth kernels, Mat\'ern kernels have finite smoothness and a more local profile.
This makes them the natural representative of the Sobolev-rate-certified regime, and hence Proposition~\ref{proposition:error-bound} (detailed later in \cref{subsec:error-analysis}) will provide an explicit algebraic approximation rate.
They still satisfy the conditions in Remark~\ref{remark:basic-requirements} and admit analytic temporal operators:
\begin{align}
    \bar{g}_k &= \frac{2\sigma_k}{\sqrt{3}} - \sigma_k \left[ \left(\frac{1 - c_k}{\sigma_k} + \frac{2}{\sqrt{3}}\right) e^{-\frac{\sqrt{3}(1 - c_k)}{\sigma_k}} + \left(\frac{c_k}{\sigma_k} + \frac{2}{\sqrt{3}}\right) e^{-\frac{\sqrt{3}c_k}{\sigma_k}} \right], \label{eq:integral-of-gk-matern-rbf} \\
		g_k^{\prime}(t) &= -\frac{3(t - c_k)}{\sigma_k^2} \exp\left(-\frac{\sqrt{3}\,|t - c_k|}{\sigma_k}\right). \label{eq:derivative-of-gk-matern-rbf}
\end{align}

\end{example}

\begin{example}[Gaussian RBFs]
The generating function is $\phi(r)=\exp(-r^2)$, yielding the Gaussian temporal item $g_k(t)=\exp\left(-|t-c_k|^2 / \sigma_k^2\right)$.
Gaussian kernels are strictly positive definite and infinitely smooth.
Together with quasi-uniform centers, they  satisfy the basic conditions  outlined in Remark~\ref{remark:basic-requirements},
while their approximation behavior is governed by the Gaussian-specific estimate in Proposition~\ref{proposition:error-bound-infinitely-smooth}(i) (detailed later in \cref{subsec:error-analysis}).
Moreover, both the integral $\bar g_k$ and derivative $g_k^{\prime}(t)$ admit closed forms:
\begin{align}
    \bar{g}_k  &= \frac{\sigma_k\sqrt{\pi}}{2} \left[ \mathrm{erf}\left(\frac{1 - c_k}{\sigma_k}\right) + \mathrm{erf}\left(\frac{c_k}{\sigma_k}\right) \right], \label{eq:integral-of-gk-gaussian-rbf} \\
    g_k^{\prime}(t) &= -\frac{2(t - c_k)}{\sigma_k^2} g_k(t), \label{eq:derivative-of-gk-gaussian-rbf}
\end{align}
where $\mathrm{erf}(z)= \tfrac{2}{\sqrt{\pi}}\int_0^z \exp(-x^2)\mathrm{d}x$ is the error function.
\end{example}

\begin{example}[Inverse Multiquadric (IMQ) RBFs]
We next consider the IMQ generating function $\phi(r)=(r^2+1)^{-1/2}$, which gives
$g_k(t)=\sigma_k / \sqrt{(t-c_k)^2+\sigma_k^2}$.
This kernel is also strictly positive definite and infinitely smooth, but decays more slowly than the Gaussian kernel.
As a result, it provides broader temporal influence while preserving analytic integrals and derivatives.
Its approximation error is described by the exponential-type IMQ estimate in Proposition~\ref{proposition:error-bound-infinitely-smooth}(ii).
Moreover, OS-DRE can exploit analytic temporal operators:
\begin{align}
    \bar{g}_k &= \sigma_k \ln \left( \frac{(1 - c_k) + \sqrt{(1 - c_k)^2 + \sigma_k^2}}{-c_k + \sqrt{c_k^2 + \sigma_k^2}} \right), \label{eq:integral-of-gk-imq-rbf} \\ 
	g_k^{\prime}(t) &= -\frac{\sigma_k(t - c_k)}{\left( (t - c_k)^2 + \sigma_k^2 \right)^{3/2}}. \label{eq:derivative-of-gk-imq-rbf}
\end{align}

\end{example}

\begin{example}[Rational Quadratic (RQ) RBFs]
The generating function $\phi(r)=(1+r^2)^{-1}$ yields the item $g_k(t)=\sigma_k^2 / \left((t-c_k)^2+\sigma_k^2\right)$.
This kernel can be viewed as a special case of the generalized IMQ family
$\phi(r)=(1+r^2)^{-\alpha}$ with $\alpha=1$.
It is strictly positive definite and infinitely smooth. Its approximation error is covered directly by Proposition~\ref{proposition:error-bound-infinitely-smooth}(iii).
Closed-form temporal operators are:
\begin{align}
    \bar{g}_k &= \sigma_k \left( \arctan\left(\frac{1 - c_k}{\sigma_k}\right) + \arctan\left(\frac{c_k}{\sigma_k}\right) \right), \label{eq:integral-of-gk-rq-rbf} \\
	g_k^{\prime}(t) &= -\frac{2\sigma_k^2(t - c_k)}{\left( (t - c_k)^2 + \sigma_k^2 \right)^2}. \label{eq:derivative-of-gk-rq-rbf}
\end{align}

\end{example}

\section{Approximation and Estimation Error}
\label{sec:approximation-estimation}
\label{subsec:error-analysis}
The analytic dictionary in \cref{sec:rbf-framework-and-analysis} makes temporal
integration exact for the represented score. Its accuracy as a density-ratio
estimator depends on how well that score approximates the target and how well
its coefficients and shape parameters are learned. We first characterize temporal
approximation, then relate it to the error of a learned estimator.

\subsection{Temporal Approximation}
For fixed $\vx$, let $\mathcal E_K(\vx)$ denote the best $L^2([0,1])$ error
attainable over the admissible $K$-term RBF family:
\begin{equation}
\label{eq:best-temporal-error}
\mathcal E_K(\vx)=\inf_{\boldsymbol\sigma_{1:K}\in\mathcal A_K}
\inf_{v\in\mathcal V_K(\boldsymbol\sigma_{1:K})}
\|s^{(\rvt)}(\vx,\cdot)-v\|_{\mathcal H_t}.
\end{equation}
Under Lemma~\ref{lemma:rbf-density}, this error tends to zero, and any sequence of
approximations whose excess over this infimum vanishes gives a consistent log-ratio,
since temporal integration on $[0,1]$ has operator norm one. The next statement
formalizes this consistency.

\begin{lemma}[Approximation Consistency]
\label{lemma:qualitative-approximation-consistency}
Assume the hypotheses of Lemma~\ref{lemma:rbf-density}. {For fixed $\vx$ and each $K$, choose $\boldsymbol\sigma_{1:K}\in\mathcal A_K$ and
$s_K^{(\rvt)}(\vx,\cdot)\in\mathcal V_K(\boldsymbol\sigma_{1:K})$. Suppose that the
approximation error measured in the $\mathcal H_t$-norm is at most $\eta_K$ above the
best error attainable by jointly choosing $\boldsymbol\sigma_{1:K}\in\mathcal A_K$ and
$v\in\mathcal V_K(\boldsymbol\sigma_{1:K})$, where $\eta_K\downarrow0$. Concretely,
\begin{equation}
\|s^{(\rvt)}(\vx,\cdot)-s_K^{(\rvt)}(\vx,\cdot)\|_{\mathcal H_t}
\le
\inf_{\boldsymbol\sigma_{1:K}\in\mathcal A_K}
\inf_{v\in\mathcal V_K(\boldsymbol\sigma_{1:K})}
\|s^{(\rvt)}(\vx,\cdot)-v\|_{\mathcal H_t}+\eta_K.
\end{equation}}
Define $\log r_K(\vx)=\int_0^1s_K^{(\rvt)}(\vx,t)\,\mathrm dt$. Then, as $K\to\infty$,
\begin{align}
\|s^{(\rvt)}(\vx,\cdot)-s_K^{(\rvt)}(\vx,\cdot)\|_{\mathcal H_t}
&\longrightarrow0,\label{eq:qualitative-approximation-consistency}\\
|\log r(\vx)-\log r_K(\vx)|
&\longrightarrow0.\label{eq:qualitative-logratio-consistency}
\end{align}
\end{lemma}
See \cref{proof:qualitative-approximation-consistency} for a proof.
The following results quantify the finite-$K$ error using target regularity.
The Sobolev and fixed-scale native-space conventions are given in
Appendix~\ref{appendix:kernel-conventions}.

\begin{restatable}[Sobolev Approximation Rate]{proposition}{ApproximationErrorBound}
\label{proposition:error-bound}
Let $\sigma_0>0$ be independent of $K$, assume
$\boldsymbol\sigma_{0,1:K}\in\mathcal A_K$ for every $K$, and suppose that
$\kappa_{\phi,\sigma_0}$ is strictly positive definite.
Suppose that, for some
$\tau>1/2$, its native space $\mathcal N_{\phi,\sigma_0}(\mathbb R)$ is norm-equivalent
to $\mathcal W^{\tau,2}(\mathbb R)$. Let
$s^{(\rvt)}(\vx,\cdot)\in\mathcal W^{\beta,2}([0,1])$ with
$1/2<\beta\le\tau$. If $(\mathcal C_K)$ is a quasi-uniform sequence, then there exist
$C>0$, depending on $\phi$, $\sigma_0$, $\beta$, $\tau$, $[0,1]$
and the uniform mesh-ratio bound of the center sequence, but independent of
$\vx$, $s^{(\rvt)}(\vx,\cdot)$ and $K$, and a threshold $K_0$, depending only on
the same quantities, such that for all $K\ge K_0$,
\begin{equation}
\label{eq:sobolev-algebraic-rate}
\inf_{\boldsymbol\sigma_{1:K}\in\mathcal A_K}\inf_{v\in\mathcal V_K(\boldsymbol\sigma_{1:K})}
\|s^{(\rvt)}(\vx,\cdot)-v\|_{\mathcal H_t}
\le C K^{-\beta}\|s^{(\rvt)}(\vx,\cdot)\|_{\mathcal W^{\beta,2}([0,1])}.
\end{equation}
\end{restatable}
The proof is in Appendix~\ref{proof:error-bound}. For the Mat\'ern-$3/2$ kernel,
$\tau=2$, so the rate uses target smoothness up to order two. Both the exponent
and the target norm matter: steep but smooth temporal variation can increase the
error bound without lowering $\beta$.

\begin{proposition}[Native-Space Approximation Rates]
\label{proposition:error-bound-infinitely-smooth}
Let $(\mathcal C_K)$ be quasi-uniform and write $h_K$ for its fill distance, as defined
in \cref{appendix:kernel-conventions}. Assume $h_K\le c_\ast K^{-1}$, where $c_\ast$
is independent of $K$. Fix $\sigma_0>0$ independently of $K$ and assume
$\boldsymbol\sigma_{0,1:K}\in\mathcal A_K$ for every $K$.
For each kernel below, assume
$s^{(\rvt)}(\vx,\cdot)\in\mathcal N_{\phi,\sigma_0}([0,1])$.
Then, for all sufficiently large $K$:

\textbf{(i) Gaussian:} for $\phi(r)=e^{-r^2}$,
\begin{equation}
\label{eq:gaussian-k-bound}
\mathcal E_K(\vx)
\le \widetilde C\exp(-\widetilde cK\log K)
\|s^{(\rvt)}(\vx,\cdot)\|_{\mathcal N_{\phi,\sigma_0}([0,1])}.
\end{equation}
\textbf{(ii) Standard IMQ:} for $\phi(r)=(1+r^2)^{-1/2}$, under
Assumption~\ref{assumption:fixed-scale-exponential-interpolation},
\begin{equation}
\label{eq:imq-k-bound}
\mathcal E_K(\vx)
\le \widehat C\exp(-\widehat cK)
\|s^{(\rvt)}(\vx,\cdot)\|_{\mathcal N_{\phi,\sigma_0}([0,1])}.
\end{equation}
\textbf{(iii) Generalized IMQ:} for $\phi(r)=(1+r^2)^{-\alpha}$ with $\alpha>1/2$,
\begin{equation}
\label{eq:gimq-k-bound}
\mathcal E_K(\vx)
\le \widetilde C\exp(-\widetilde cK)
\|s^{(\rvt)}(\vx,\cdot)\|_{\mathcal N_{\phi,\sigma_0}([0,1])}.
\end{equation}
All constants are positive and may depend on the fixed kernel, scale and center
geometry, but not on $K$, $\vx$ or the target function, and the threshold on $K$
may be chosen to depend only on the same quantities.
\end{proposition}
{For the endpoint-inclusive uniform grid $c_k=(k-1)/(K-1)$, $K\ge2$, used in
the reported experiments, $h_K=1/[2(K-1)]\le K^{-1}$ for every $K\ge2$, so one may take
$c_\ast=1$.}
Appendix~\ref{proof:error-bound-infinitely-smooth} gives the fill-distance estimates
and their conversion to these rates; Appendix~\ref{appendix:imq-rbfs} verifies the
standard IMQ assumption. RQ is the case $\alpha=1$ in part (iii).
These faster rates require membership in the corresponding native space;
smoothness alone does not establish that condition.

The bounds use the admissible common scale $\sigma_0$ as a comparison, so they
control the best error over learnable scales.

The same comparison yields joint convergence. For a probability measure $\mu$ on
$\mathcal X$, the fixed-scale projections converge in
$L^2(\mu\otimes\mathrm dt)$ whenever the target lies in that space.
In the Sobolev regime, if its squared temporal Sobolev norm is $\mu$-integrable, their squared error is
$\mathcal O(K^{-2\beta})$, which also bounds the squared $L^2(\mu)$ log-ratio error.
Corollary~\ref{corollary:joint-expressive-capacity} proves this with one scale
shared across $\vx$ and unrestricted measurable coefficient functions.

\subsection{Finite-Sample Estimation}
A learned model also incurs error from its neural coefficient class, finite data
and incomplete optimization. For a score $\widehat s$ trained on $N$ i.i.d. tuples,
write $\log\widehat r(\vx)=\int_0^1\widehat s(\vx,t)\,\mathrm dt$.
Let $P$ be the measure defining the time-score error
and $\mu$ the measure for log-ratio evaluation.

{Four} terms quantify the resulting error.
{Here $\mathcal F_K$ denotes the coefficient class, and
$\mathcal S_{K,\mathcal F}$ denotes the resulting class of $K$-term temporal scores;
Appendix~\ref{appendix:finite-sample-oracle} defines both.} Write $A_K$ for the best temporal
approximation error in $L^2(P)$ with unrestricted coefficients and global scales;
for $\eta>0$, choose such an approximation with error at most $A_K+\eta$, and let
$B_{K,\mathcal F,\eta}$ be its approximation error using the neural coefficient
class $\mathcal F_K$ at those scales. Let $\Gamma_N(\mathcal S_{K,\mathcal F},\delta)$
denote the uniform deviation between empirical and population loss differences
relative to a fixed reference model in the class.
Appendix~\ref{appendix:finite-sample-oracle} defines these quantities precisely.

Suppose the population excess loss is bounded above and below by
$c_+$ and $c_-$ times the squared $L^2(P)$ score error, respectively, with
$0<c_-\le c_+<\infty$; temporal integration has norm at most $C_T$ from
$L^2(P)$ score errors to $L^2(\mu)$ log-ratio errors.
{Define $\varepsilon_{\rm opt}$ as the empirical loss of the returned score
minus the infimum over the chosen model class. This optimization term measures the
remaining empirical loss gap and enters the total error bound additively.}
Under Assumptions~\ref{assumption:risk-calibration}--\ref{assumption:integration-norm},
the log-ratio error satisfies
\begin{equation}
\label{eq:oracle-inequality}
\|\log\widehat r-\log r\|_{L^2(\mu)}^2
\le
\frac{C_T^2}{c_-}
\left[
2c_+\{(A_K+\eta)^2+B_{K,\mathcal F,\eta}^2\}
+2\Gamma_N(\mathcal S_{K,\mathcal F},\delta)
+\varepsilon_{\rm opt}
\right].
\end{equation}

The terms separate temporal approximation, neural coefficient approximation,
generalization and optimization. The proof and the bound on
$\Gamma_N(\mathcal S_{K,\mathcal F},\delta)$ are in
Appendix~\ref{appendix:finite-sample-oracle}.

This decomposition explains the role of dictionary size. Suppose the calibration
and integration constants are uniform in $K,N$. In the Sobolev regime,
$A_K^2=\mathcal O(K^{-2\beta})$ when
$P\le p_{\max}\mu\otimes\mathrm dt$ and the squared temporal Sobolev norm is
$\mu$-integrable, with constants independent of $K,N$.
Under the parameter-regularity conditions of
Corollary~\ref{corollary:sample-size-capacity}, a fixed backbone and feature dimension
with uniformly bounded loss differences and polynomially growing parameter radius and
Lipschitz constant give a generalization term
$\mathcal O(\sqrt{K\log N/N})$ at fixed confidence.
Balancing this term against $A_K^2$ on the scale of the bound above yields
$K\asymp(N/\log N)^{1/(4\beta+1)}$.
Consistency additionally requires vanishing neural approximation and optimization
errors and $\eta\to0$. Thus temporal capacity alone does not determine the error
of the learned estimator.

The experiments that follow evaluate the chosen kernels and dictionary sizes under the
specified training protocols.

\section{Experimental Settings and Results}
\label{section:experiment}

All experiments were conducted on four NVIDIA TITAN X (Pascal) 12GB and three NVIDIA RTX A6000 48GB GPUs using PyTorch (2.1.2) and PyTorch-Lightning (2.1.2).
Our code is developed based on the official code for both DRE-$\infty$ at \url{https://github.com/ermongroup/dre-infinity} and OpenOOD \citep{zhang2023openood} at \url{https://github.com/Jingkang50/OpenOOD}. {The implementation and experiment entry points are publicly available at \url{https://github.com/Hoemr/OpenDRE}.} 

We compare OS-DRE with DRE-$\infty$~\citep{choi2022density} and D$^3$RE~\citep{chen2025dequantified}.
For baselines, the log-ratio is obtained by numerical quadrature along $t\in[0,1]$
(using the trapezoidal rule by default), where the NFE controls inference cost.

\subsection{Model, Objectives, and Training Protocol}
\label{subsec:reproducibility}

This subsection summarizes the implementation choices needed to reproduce OS-DRE:
(i) the temporal basis and neural network parameterization, (ii) the training objectives, and
(iii) the sampling schedules and training/inference procedures.

{We implement the finite RBF model from
Section~\ref{sec:rbf-framework-and-analysis}, with neural coefficients and global
learnable scales. The choices below specify the model class and training protocol
whose approximation and estimation errors were distinguished in
Section~\ref{sec:approximation-estimation}.}

\paragraph{Neural Network Parameterization.}
We parameterize the spatial coefficients with a neural network $\text{NN}(\cdot;\vtheta)$, which outputs the $K$ coefficients
\begin{equation} \label{eq:one-pass-forwardNN}
\vh_{1:K}(\vx;\vtheta)=[h_1(\vx;\vtheta), \dots, h_K(\vx;\vtheta)] = \text{NN}(\vx;\vtheta).
\end{equation}

For brevity, we denote the vector of temporal items and their derivatives evaluated at time $t$ as
\begin{equation}
\vg_{1:K}(t) = [g_1(t), \dots, g_K(t)], \quad
\vg_{1:K}^{\prime}(t) = [g_1^{\prime}(t), \dots, g_K^{\prime}(t)].
\end{equation}

Then, for a given pre-specified temporal items $\{g_k\}_{k=1}^{K}$, the time score, its derivative, and its integral admit closed forms:
\begin{equation}\label{eq:score-model-and-derivative-integral}
\begin{aligned}
    s^{(\rvt)}(\vx,t;\vtheta) &= \sum_{k=1}^{K}h_k(\vx;\vtheta)g_{k}(t)=\left \langle \vh_{1:K}(\vx;\vtheta),\vg_{1:K}(t) \right \rangle, \\
    \partial_t s^{(\rvt)}(\vx,t;\vtheta) &= \sum_{k=1}^{K}h_k(\vx;\vtheta)g_k^{\prime}(t)=\left \langle \vh_{1:K}(\vx;\vtheta),\vg_{1:K}^{\prime}(t) \right \rangle, \\
    \int_{0}^{1}s^{(\rvt)}(\vx,t;\vtheta)\mathrm{d}t &=\sum_{k=1}^{K}h_k(\vx;\vtheta)\bar{g}_k=\left \langle \vh_{1:K}(\vx;\vtheta),\bar{\vg}_{1:K} \right \rangle,
\end{aligned}
\end{equation}
where $\bar g_k$ is precomputed from the chosen kernel family (cf. \cref{subsec:rbf-kernels}).

{The trainable parameters are the network weights $\vtheta$ and the shape parameters $\{\sigma_k\}_{k=1}^{K}$. The temporal centers are fixed on the endpoint-inclusive uniform grid $c_k\triangleq(k-1)/(K-1)$, $k=1,\ldots,K$, which satisfies the center assumptions of Propositions~\ref{proposition:error-bound} and~\ref{proposition:error-bound-infinitely-smooth}. We reparameterize $\sigma_k=\exp(\log\sigma_k)$, with $\log\sigma_k$ initialized at $-\log K$ \citep{chen2026dont}, and clamp the log-shape parameters below before exponentiation.}

The network consists of a shared backbone and task-specific heads. The backbone extracts a feature embedding from $\vx$. For time-score matching and neural DRE tasks we use a ResNet backbone (e.g., ResNet-18) \citep{he2016deep}. For joint estimation of both the time score $s^{(\rvt)}$ and the data score $\vs^{(\rvx)}(\vx,t)\triangleq\nabla_{\vx}\log p_t(\vx)$, we adopt a U-Net backbone \citep{ronneberger2015u} to better capture spatial structure in high-dimensional data. The coefficient head maps the feature embedding to the $K$ coefficients. When joint score matching (see \cref{eq:joint-score-matching-loss}) is used, an additional head predicts the data score $\vs^{(\rvx)}(\vx,t;\vtheta)$ using time positional encoding \citep{vaswani2017attention} and residual blocks.

\paragraph{Training Objective.}
We train $\vtheta$ and $\{\sigma_k\}_{k=1}^{K}$ using a task-dependent score-matching objective.
When both the time score $s^{(\rvt)}(\vx,t)$ and the data score $\vs^{(\rvx)}(\vx,t)$ are modeled, we use the joint score matching objective \citep{choi2022density}:
\begin{equation}\label{eq:joint-score-matching-loss}
	\begin{aligned}
		\mathcal{L}_{\text{joint}}(\vtheta)&=2\mathbb{E}_{p_{0}(\vx_0)p_{1}(\vx_1)}\left[\lambda(0) s^{(\rvt)}(\vx_0,0;\vtheta)-\lambda(1)s^{(\rvt)}(\vx_1,1;\vtheta)\right] \\
		&\quad +\mathbb{E}_{p(t)p_t(\vx)}\mathbb{E}_{ p(\vv)}\left[2\lambda(t)\partial_t s^{(\rvt)}(\vx,t;\vtheta)
		+2\lambda^{\prime}(t)s^{(\rvt)}(\vx,t;\vtheta) \right.\\
		&\quad \left.+\lambda(t)\|\vs^{(\rvx)}(\vx,t;\vtheta)\|_2^2
		+\lambda(t)|s^{(\rvt)}(\vx,t;\vtheta)|^2
		+2\lambda(t)\vv^{\mathsf{T}}\nabla_{\vx}\vs^{(\rvx)}(\vx,t;\vtheta)\vv\right],
	\end{aligned}
\end{equation}
where
$p(\vv)=\mathcal{N}(\vzero,\mI_d)$, and $\lambda(t)$ is a weighting function.
The term $\partial_t s^{(\rvt)}(\vx,t;\vtheta)$ is an approximation of $\partial_t^2 \log p_t(\vx)$.
The training-time distribution $p(t)$ used for each task family is specified in the
``Sampling from $p_t(\vx)$ with Path Schedules'' paragraph below.

Our framework provides computational advantages in both training and inference.
The derivative $\partial_t s^{(\rvt)}(\vx,t;\vtheta)$ is computed via \cref{eq:score-model-and-derivative-integral}, eliminating the need for automatic differentiation over $t$ and the second-order gradients required by prior methods (e.g., DRE-$\infty$), thereby reducing optimization to a first-order problem with improved stability.
For large-scale tasks such as OOD detection, we adopt conditional score matching \citep{yu2025density}, which is more efficient while retaining the one-step inference property of OS-DRE.

\paragraph{Boundary Dynamics.}
The DDBI construction uses Gaussian dequantization to smooth the endpoint densities and
keep the conditional bridge variance positive at the endpoints;
\citet[Corollary~3.8]{chen2025dequantified} establish a time-uniform bound on the
expected absolute time score for their linear-schedule DDBI under stated moment
conditions. This provides theoretical support for using Gaussian dequantization to
alleviate boundary instability. Our trigonometric DDBI schedule likewise retains a
positive Gaussian variance at both endpoints.

OS-DRE approximates the resulting time score with the RBF temporal basis. For a fixed
sample $\vx$, finite $K$, and positive scales, the four RBF families give a bounded,
continuously differentiable time-score approximation on $[0,1]$, with the widths
controlling temporal resolution.
The dictionary provides this smooth approximation, while Gaussian dequantization
smooths the endpoint densities of the target path.

{\paragraph{Complexity and Measured Efficiency.}
We compare OS-DRE against an $M$-step score-based baseline that evaluates the trained
score model $M$ times along the probability path. OS-DRE requires a single forward pass at
inference ($\text{NFE}=1$), whereas the baseline requires $\text{NFE}=M$. Let $n_{\mathrm b}$
be the minibatch size, $K$ the number of temporal RBF items, $d_{\mathrm{feat}}$ the dimension
of the shared feature embedding, and $c_{\mathrm{net}}$, $m_{\mathrm{net}}$ the compute cost and
peak activation memory of one forward pass producing that embedding. The asymptotic
inference comparison is summarized in \cref{tab:inference-complexity}.
\begin{table}[ht]
\centering
\caption{Per-batch inference complexity: OS-DRE versus score-based baseline.
The OS-DRE coefficient matrix has shape $n_{\mathrm b}\times K$ and is not a $K$-channel
dense image tensor.}
\label{tab:inference-complexity}
\begin{tabular}{lll}
\toprule
Quantity & OS-DRE (ours) & Score-based baseline \\
\midrule
$\text{NFE}$ & $1$ & $M$ \\
Compute cost &
$c_{\mathrm{net}}+\mathcal{O}(n_{\mathrm b}d_{\mathrm{feat}}K+n_{\mathrm b}K)$ &
$M\,c_{\mathrm{net}}+\mathcal{O}(M\,n_{\mathrm b}d_{\mathrm{feat}})$ \\
Peak activation memory &
$m_{\mathrm{net}}+\mathcal{O}(n_{\mathrm b}d_{\mathrm{feat}}+n_{\mathrm b}K)$ &
$m_{\mathrm{net}}+\mathcal{O}(n_{\mathrm b}d_{\mathrm{feat}}+n_{\mathrm b})$ \\
\bottomrule
\end{tabular}
\end{table}

Thus OS-DRE trades a $K$-dimensional coefficient head for removing the factor $M$ from
inference cost; the baseline memory expression assumes sequential time evaluations with
discarded intermediate activations.

On CIFAR-$10$/$100$, OS-DRE has $11,440,762$ trainable parameters versus $11,235,675$ for
DRE-$\infty$/D$^3$RE, an increment of $205,087$ ($\approx 0.78$ MiB in FP32 weights). At
$n_{\mathrm b}=768$ and $K=400$, the extra batch-level coefficient matrix is $1.17$ MiB.
{At training time OS-DRE uses roughly half the peak memory of the $M$-step
baseline on both datasets, while the measured inference peaks remain within $12$ MiB
of each other.} On an
NVIDIA RTX A6000, OS-DRE inference takes $37.6\pm 0.7$ s on CIFAR-$10$ and $36.6\pm 0.6$ s on
CIFAR-$100$, versus $284.4\pm0.4$ s and $283.6\pm0.9$ s for DRE-$\infty$. On CIFAR-$10$,
training times per epoch are $38.38$ s (interquartile range, IQR, $38.29$--$38.47$) for
OS-DRE versus $66.35$ s (IQR $66.23$--$66.48$) for DRE-$\infty$; on CIFAR-$100$ they are
$38.58$ s (IQR $38.51$--$38.63$) versus $65.96$ s (IQR $65.86$--$66.02$).
\cref{tab:measured-efficiency} summarizes these timings together with the measured peak
GPU memory for training and inference.

\begin{table}[ht]
\centering
\small
\setlength{\tabcolsep}{4pt}
\caption{CIFAR efficiency comparison on an NVIDIA RTX A6000. Training times are medians
per epoch; inference times are means $\pm$ standard deviations from the reported evaluation.
Peak VRAM is the maximum allocated GPU memory over the run.}
\label{tab:measured-efficiency}
\begin{tabular}{llrrrr}
\toprule
Dataset & Method & \shortstack{Train peak\\VRAM (MiB){$\downarrow$}} &
\shortstack{Infer peak\\VRAM (MiB){$\downarrow$}} & \shortstack{Train\\(s/epoch){$\downarrow$}} &
\shortstack{Inference\\(s){$\downarrow$}} \\
\midrule
CIFAR-$10$ & OS-DRE & $12781$ & $5739$ & $38.38$ & $37.6\pm0.7$ \\
 & DRE-$\infty$ & $25716$ & $5751$ & $66.35$ & $284.4\pm0.4$ \\
CIFAR-$100$ & OS-DRE & $19101$ & $5739$ & $38.58$ & $36.6\pm0.6$ \\
 & DRE-$\infty$ & $38510$ & $5751$ & $65.96$ & $283.6\pm0.9$ \\
\bottomrule
\end{tabular}
\end{table}

The two methods use the same backbone and data pipeline, with $K=400$ for OS-DRE.
Peak allocated GPU memory is recorded separately for training and inference; the peak
counter is reset before each phase, and GPU timings are synchronized.
}

\paragraph{Sampling from $p_t(\vx)$ with Path Schedules.}
To evaluate the training objective, we sample from the interpolated distribution $p_t(\vx)$ by constructing $\vx_t$ from independent samples $\vx_0\sim p_0$, $\vx_1\sim p_1$, and $t\sim p(t)$ via a task-dependent interpolation scheme.
For simple settings we use fixed schedules such as the linear or variance-preserving (VP) path; for DRE, common choices include the diffusion interpolant (DI) and dequantified diffusion bridge interpolant (DDBI).
{For the non-conditional experiments, we sample $t\sim\mathrm{Unif}(0,1)$ with the VP path \citep{song2020score}:
$\alpha_t=\exp[-0.25(20-0.1)t^2-0.05t]$,
$\beta_t=\sqrt{1-\alpha_t^2}$, and
$\vx_t=\alpha_t \vx_0+\beta_t \vx_1$, with weighting $\lambda(t)=\alpha_t^2$.}
{For the reported CIFAR-$10$/$100$ OOD experiments, we use the DDBI construction
\citep{chen2025dequantified} with the fixed trigonometric schedule
$\alpha_t=\cos(\pi t/2)$, $\beta_t=\sin(\pi t/2)$, and bridge variance
$v_t=0.01t(1-t)+0.005$ ($0.01$ scales the diffusion-bridge term, $0.005$ the
Gaussian-dequantization contribution). The path samples are
\begin{equation}
\label{eq:ood-path-sampling}
\vx_t=\alpha_t\vx_0+\beta_t\vx_1+\sqrt{v_t}\,\vz,
\qquad \vz\sim\mathcal N(\vzero,\mI_d),
\end{equation}
and $t$ follows a logit-normal law \citep{geng2025meanflow}: draw $u\sim\mathcal N(-0.4,1)$
and set $t=(1+e^{-u})^{-1}$, with numerical clipping to $[\varepsilon,1-\varepsilon]$ for
$\varepsilon=10^{-5}$, following the dequantification construction of
\citet{chen2025dequantified}. The schedule is fixed throughout training.

The OOD objective uses conditional joint score matching with variance-based normalization \citep{chen2026dont}.
Writing $\dot\alpha_t$, $\dot\beta_t$ and $\dot v_t$ for temporal derivatives, the effective
weights multiplying the squared time-score residual and the mean squared data-score
residual over the $d$ coordinates are \citep{chen2026dont}, respectively,
\begin{equation}
\label{eq:ood-score-weights}
\lambda^{(\rvt)}(\vx_0,\vx_1,t)
=\frac{v_t}{\bigl\||\dot\alpha_t|\vx_0+|\dot\beta_t|\vx_1\bigr\|_2^2
+\dfrac{d\dot v_t^2}{2(v_t+10^{-5})}},
\qquad \lambda^{(\rvx)}(t)=v_t.
\end{equation}

The numerical offset $10^{-5}$ matches the variance-floor used at sampling time. Training
uses the sampling law and loss weights above; inference uses the analytic temporal
integral.}

\paragraph{Inference Protocol.}
After training, OS-DRE evaluates the log-ratio by a single pass:
\begin{equation}\label{eq:estimated-log-density-ratio}
	\log \hat{r}(\vx)= \int_{0}^{1}s^{(\rvt)}(\vx,t;\vtheta)\mathrm{d}t
        =\sum_{k=1}^{K} h_k(\vx;\vtheta^\star) \bar{g}_k
        =\left \langle \vh_{1:K}(\vx;\vtheta^\star), \bar{\vg}_{1:K} \right \rangle,
\end{equation}
where $\vh_{1:K}(\vx;\vtheta^\star)=\text{NN}(\vx;\vtheta^\star)$ (see \cref{eq:one-pass-forwardNN}).
Because $\bar{\vg}_{1:K}$ are precomputed analytic constants, evaluation requires only one network evaluation ($\text{NFE}=1$), regardless of the particular time-sampling strategy or path family used during training.

{Finally, we separate the training and inference procedures into
\Cref{alg:os-dre-training,alg:os-dre-inference}.}
\begin{algorithm}[ht]
\caption{OS-DRE Training}
\label{alg:os-dre-training}
\begin{algorithmic}[1]
\renewcommand{\algorithmicrequire}{\textbf{Input:}}
\renewcommand{\algorithmicensure}{\textbf{Output:}}
\REQUIRE Distributions $p_0$, $p_1$; fixed centers $\{c_k\}_{k=1}^K$; learnable shape parameters $\{\sigma_k\}_{k=1}^K$; weighting $\lambda(t)$; time distribution $p(t)$; path schedule for $\vx_t$; neural coefficient model $\mathrm{NN}(\cdot;\vtheta)$.
\ENSURE Trained parameters $\vtheta^\star$, $\{\sigma_k^\star\}_{k=1}^K$; log-ratio estimator for $\log \hat r(\vx)$.
\STATE Initialize $\vtheta$ and $\{\sigma_k\}_{k=1}^K$.
\FOR{each optimization step}
    \STATE Sample $\vx_0\sim p_0$, $\vx_1\sim p_1$, $t\sim p(t)$; construct $\vx_t$ via the path schedule.
    \STATE Neural forward pass $\vh_{1:K}(\vx_t;\vtheta)=\mathrm{NN}(\vx_t;\vtheta)$; when joint score matching is enabled, additionally obtain $\vs^{(\rvx)}(\vx_t,t;\vtheta)$ from the data-score head.
    \STATE Evaluate the analytic temporal basis and derivative via \cref{eq:rbf_general_integral_derivative2}.
    \STATE Form the time score and derivative via \cref{eq:score-model-and-derivative-integral} evaluated at $\vx_t$.
    \STATE Evaluate the selected score-matching loss (e.g., $\mathcal{L}_{\text{joint}}$ in \cref{eq:joint-score-matching-loss}).
    \IF{the auxiliary analytic density-ratio loss is enabled}
    \STATE Evaluate the closed-form temporal integrals in \cref{eq:rbf_general_integral_derivative} at $\{\sigma_k\}_{k=1}^K$, form the analytic endpoint log-ratios, and add the logistic Bregman term to the loss.
    \ENDIF
    \STATE Perform one optimizer step on $\vtheta$ and $\{\sigma_k\}_{k=1}^K$.
\ENDFOR
\end{algorithmic}
\end{algorithm}

\begin{algorithm}[ht]
\caption{OS-DRE One-Step Inference}
\label{alg:os-dre-inference}
\begin{algorithmic}[1]
\renewcommand{\algorithmicrequire}{\textbf{Input:}}
\renewcommand{\algorithmicensure}{\textbf{Output:}}
\REQUIRE Trained parameters $\vtheta^\star$, $\{\sigma_k^\star\}_{k=1}^K$.
\ENSURE One-step log-ratio estimator $\log\widehat r(\vx)$.
\STATE Precompute $\bar g_k^\star$ via \cref{eq:rbf_general_integral_derivative}.
\STATE For sample $\vx$, evaluate $\vh_{1:K}(\vx;\vtheta^\star)=\mathrm{NN}(\vx;\vtheta^\star)$ once.
\STATE Return $\log\widehat r(\vx)=\langle \vh_{1:K}(\vx;\vtheta^\star),\bar{\vg}^\star_{1:K}\rangle$
with $\bar{\vg}^\star_{1:K}=[\bar g_1^\star,\ldots,\bar g_K^\star]$.
\end{algorithmic}
\end{algorithm}

\subsection{Application to Density Estimation}
\label{subsection:density-estimation}

\paragraph{Problem Setup and Evaluation Metric.}
We consider density estimation where $p_0(\vx)$ is a simple reference distribution, e.g., $\mathcal{N}(\vzero,\mI_d)$, and $p_1(\vx)$ is a complex, intractable data distribution.
The log-likelihood of $p_1$ can be written as $\log p_1(\vx)= \log r(\vx)+\log p_0(\vx)$, where $ r(\vx)={p_1(\vx)}/{p_0(\vx)}$ is the density ratio.
After training, OS-DRE provides a one-step estimate of the log-ratio
$\log \hat r(\vx)$
(cf. \cref{eq:estimated-log-density-ratio}), and we evaluate log-likelihood as
$\log p_1(\vx)\approx \log \hat r(\vx)+\log p_0(\vx)$.
We report negative log-likelihood (NLL; lower is better) on tabular datasets, and bits-per-dimension (BPD; lower is better) for energy-based modeling.

\paragraph{Density Estimation for Structured and Multimodal Benchmarks.}
We first evaluate on nine synthetic benchmarks~\citep{karras2024guiding,chen2025dequantified} covering diverse geometric challenges.
OS-DRE employs $K=400$ temporal items with a RQ kernel and performs inference with a single function evaluation ($\text{NFE}=1$). 
Baselines rely on numerical quadrature with varying budgets ($\text{NFE}\in\{2,5,10\}$). 
\begin{figure}[ht]
    \centering
    \includegraphics[width=\linewidth]{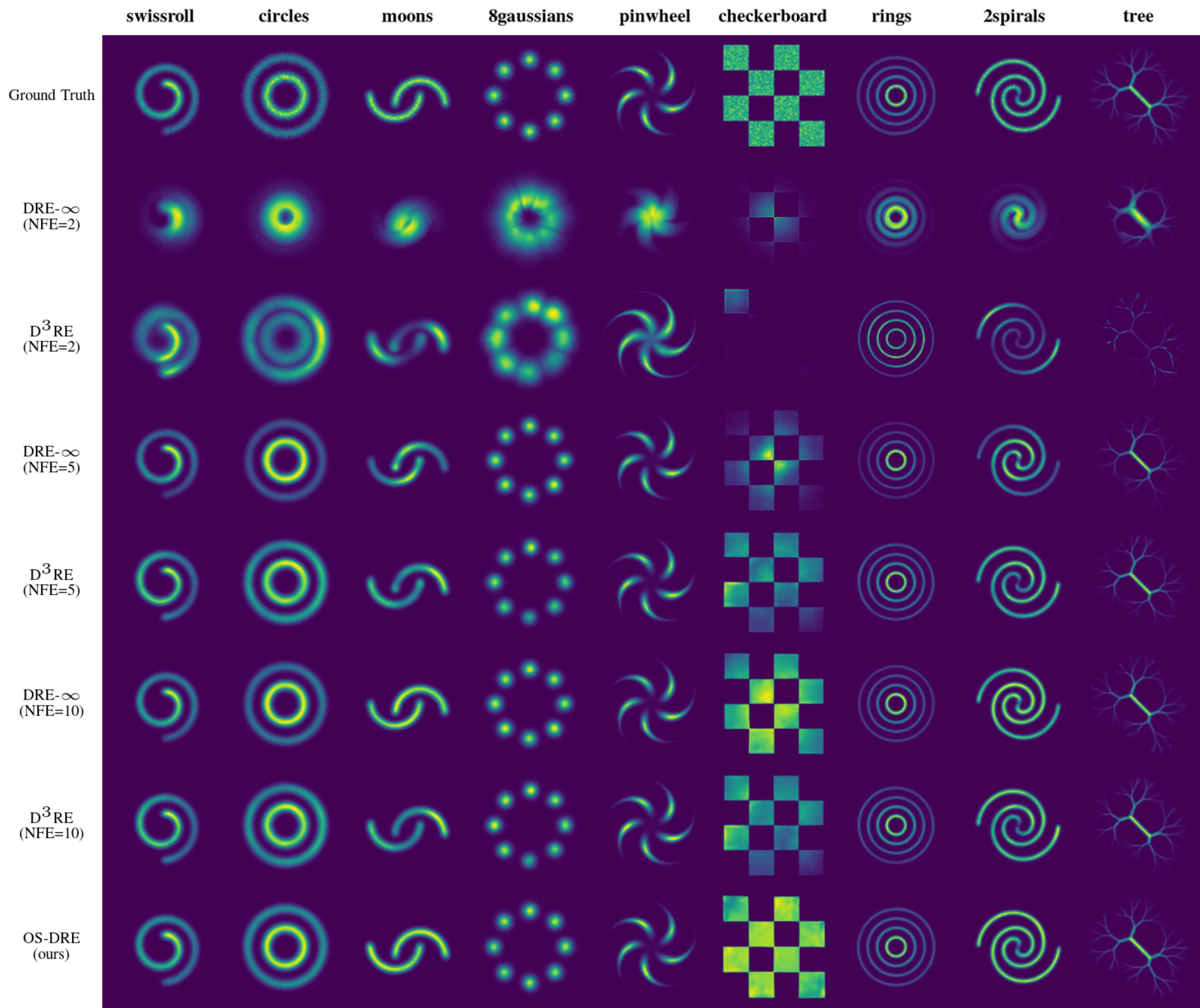}
    \caption{Comparison of density estimates on nine structured and multimodal 2D datasets. 
    DRE-$\infty$ and D$^3$RE use numerical quadrature with $\text{NFE}\in\{2,5,10\}$, whereas OS-DRE uses partly analytic one-step inference with $\text{NFE}=1$.
}
    \label{fig:likelihood-estimation-2d-synthesic-method}
\end{figure}

As visualized in \cref{fig:likelihood-estimation-2d-synthesic-method}, OS-DRE produces sharp, faithful density estimates that match or exceed baseline quality despite the lower computational cost. 
Specifically, the analytic temporal representation prevents density leakage in disconnected topologies (e.g., \textsc{circles}, \textsc{rings}) and preserves fine-grained details in thin manifolds (e.g., \textsc{swissroll}, \textsc{2spirals}) where solver-based methods often collapse structures into diffuse clouds. 
Furthermore, OS-DRE resolves sharp discontinuities (e.g., \textsc{checkerboard}) and branching topologies (e.g., \textsc{tree}) with crisp boundaries, whereas baselines exhibit excessive smoothing even at higher NFE. 
These results confirm that OS-DRE effectively reconciles accuracy with efficiency, eliminating the need for iterative integration to capture complex distributional structures.

\paragraph{Density Estimation for Real-World Tabular Benchmarks.} 
We evaluate on five tabular benchmarks: POWER, GAS, HEPMASS, MINIBOONE, and BSDS300, following \citet{papamakarios2017masked,grathwohl2018ffjord}.
OS-DRE uses $K=400$ with fixed $\text{NFE}=1$, while baselines rely on numerical quadrature and are reported across a range of NFE values (cf. \cref{tab:density-estimation-tabular}).
All timings are measured on a single  GPU.

\cref{tab:density-estimation-tabular} shows that OS-DRE achieves consistently strong NLL with minimal inference cost.
In particular, OS-DRE with IMQ/RQ kernels outperforms baselines across all datasets even when baselines use up to two orders of magnitude more function evaluations.
The advantage is especially pronounced on high-dimensional datasets (e.g., MINIBOONE and BSDS300), highlighting that replacing numerical quadrature with analytic temporal operators is crucial for stable and efficient likelihood estimation.

A notable phenomenon is that baseline performance does not necessarily improve monotonically with larger NFE
(e.g., on MINIBOONE, DRE-$\infty$ can be better at $\text{NFE}=10$ than at $\text{NFE}=50$),
suggesting practical instability when integrating complex high-dimensional scores numerically.
By contrast, OS-DRE’s one-step evaluation is deterministic given learned parameters and avoids this quadrature-induced sensitivity.
These results confirm that our partly analytic framework can reconcile estimation quality with computational efficiency.
\begin{table}[htbp]
	\centering
	\setlength{\tabcolsep}{1.3pt} 
    \caption{Density estimation performance (NLL $\downarrow$) and inference time (seconds $\downarrow$) across five tabular benchmarks. OS-DRE achieves competitive likelihoods with minimal inference cost. \textbf{Best} and \underline{fastest} results are highlighted. Timings are measured on a single NVIDIA TITAN X GPU.}
	\label{tab:density-estimation-tabular}
    \footnotesize
	\begin{tabular}{@{}lcccccc@{}} 
		\toprule
		\multirow{2}{*}{\textbf{Method}} & \multirow{2}{*}{\textbf{NFE$\downarrow$}} & \textbf{POWER} & \textbf{GAS} & \textbf{HEPMASS} & \textbf{MINIBOONE} & \textbf{BSDS300} \\
		& & \textbf{NLL$\downarrow$/Time$\downarrow$} & \textbf{NLL$\downarrow$/Time$\downarrow$} & \textbf{NLL$\downarrow$/Time$\downarrow$} & \textbf{NLL$\downarrow$/Time$\downarrow$} & \textbf{NLL$\downarrow$/Time$\downarrow$} \\
		\midrule
	\multirow{6}{*}{DRE-$\infty$} 
            & 2 & $0.05 \scriptstyle\pm 1.84$ \scriptsize/0.26 & $-4.37 \scriptstyle\pm 1.44$ \scriptsize/0.18 & $19.30 \scriptstyle\pm 1.31$ \scriptsize/0.26 & $41.55 \scriptstyle\pm 2.07$ \scriptsize/0.09 & $-130.68 \scriptstyle\pm 4.17$ \scriptsize/0.47 \\
        & 5 & $0.35 \scriptstyle\pm 0.50$ \scriptsize/0.47 & $-3.63 \scriptstyle\pm 0.78$ \scriptsize/0.29 & $20.24 \scriptstyle\pm 0.47$ \scriptsize/0.46 & $20.90 \scriptstyle\pm 0.84$ \scriptsize/0.10 & $-83.70 \scriptstyle\pm 1.35$ \scriptsize/1.02 \\
        & 10 & $0.03 \scriptstyle\pm 0.17$ \scriptsize/0.85 & $-4.34 \scriptstyle\pm 0.60$ \scriptsize/0.45 & $20.43 \scriptstyle\pm 0.52$ \scriptsize/0.80 & $20.57 \scriptstyle\pm 0.93$ \scriptsize/0.11 & $-87.65 \scriptstyle\pm 2.24$ \scriptsize/1.93 \\
        & 50 & $0.25 \scriptstyle\pm 0.28$ \scriptsize/3.50 & $-4.33 \scriptstyle\pm 0.71$ \scriptsize/1.66 & $20.67 \scriptstyle\pm 0.57$ \scriptsize/3.48 & $20.97 \scriptstyle\pm 0.51$ \scriptsize/0.19 & $-90.24 \scriptstyle\pm 2.14$ \scriptsize/9.10 \\
        & 100 & $0.26 \scriptstyle\pm 0.21$ \scriptsize/6.90 & $-4.32 \scriptstyle\pm 0.69$ \scriptsize/3.16 & $20.67 \scriptstyle\pm 0.57$ \scriptsize/6.83 & $21.03 \scriptstyle\pm 0.36$ \scriptsize/0.30 & $-90.31 \scriptstyle\pm 1.91$ \scriptsize/18.12 \\
        & 200 & $0.25 \scriptstyle\pm 0.14$ \scriptsize/13.65 & $-4.33 \scriptstyle\pm 0.70$ \scriptsize/6.20 & $20.67 \scriptstyle\pm 0.57$ \scriptsize/13.50 & $21.04 \scriptstyle\pm 0.36$ \scriptsize/0.45 & $-90.31 \scriptstyle\pm 1.91$ \scriptsize/36.19 \\
        \midrule
	\multirow{6}{*}{D$^3$RE} 
            & 2 & $3.57 \scriptstyle\pm 1.84$ \scriptsize/0.27 & $5.74 \scriptstyle\pm 15.28$ \scriptsize/0.18 & $23.90 \scriptstyle\pm 0.36$ \scriptsize/0.27 & $49.37 \scriptstyle\pm 6.85$ \scriptsize/0.10 & $-149.53 \scriptstyle\pm 9.06$ \scriptsize/0.50 \\
		& 5 & $1.26 \scriptstyle\pm 0.38$ \scriptsize/0.48 & $-1.15 \scriptstyle\pm 4.20$ \scriptsize/0.28 & $21.05 \scriptstyle\pm 0.52$ \scriptsize/0.49 & $22.90 \scriptstyle\pm 2.66$ \scriptsize/0.10 & $-101.97 \scriptstyle\pm 1.67$ \scriptsize/1.03 \\
		& 10 & $0.49 \scriptstyle\pm 0.39$ \scriptsize/0.80 & $-3.27 \scriptstyle\pm 2.00$ \scriptsize/0.43 & $20.30 \scriptstyle\pm 0.55$ \scriptsize/0.83 & $23.03 \scriptstyle\pm 2.85$ \scriptsize/0.12 & $-102.01 \scriptstyle\pm 2.43$ \scriptsize/1.91 \\
		& 50 & $0.89 \scriptstyle\pm 0.33$ \scriptsize/3.49 & $-3.16 \scriptstyle\pm 0.62$ \scriptsize/1.63 & $20.05 \scriptstyle\pm 0.35$ \scriptsize/3.50 & $22.85 \scriptstyle\pm 2.54$ \scriptsize/0.20 & $-78.26 \scriptstyle\pm 0.96$ \scriptsize/9.09 \\
        & 100 & $-0.24 \scriptstyle\pm 0.27$ \scriptsize/6.87 & $-4.10 \scriptstyle\pm 0.69$ \scriptsize/3.18 & $19.45 \scriptstyle\pm 0.46$ \scriptsize/6.86 & $22.91 \scriptstyle\pm 2.58$ \scriptsize/0.28 & $-74.61 \scriptstyle\pm 1.79$ \scriptsize/18.07 \\
        & 200 & $-0.24 \scriptstyle\pm 0.29$ \scriptsize/13.60 & $-4.09 \scriptstyle\pm 0.72$ \scriptsize/6.26 & $19.45 \scriptstyle\pm 0.48$ \scriptsize/13.54 & $22.90 \scriptstyle\pm 2.84$ \scriptsize/0.48 & $-74.47 \scriptstyle\pm 1.42$ \scriptsize/36.27 \\
		\midrule
	\makecell[c]{OS-DRE \\ \scriptsize(Matérn)} & 1 & $0.57 \scriptstyle\pm 0.11$ \scriptsize/0.08 & $-3.49 \scriptstyle\pm 0.01$ \scriptsize/\underline{0.03} & $23.66 \scriptstyle\pm 0.02$ \scriptsize/0.06 & $31.71 \scriptstyle\pm 0.11$ \scriptsize/\underline{0.003} & $-52.38 \scriptstyle\pm 0.42$ \scriptsize/0.07 \\
	\makecell[c]{OS-DRE \\ \scriptsize(Gaussian)} & 1 & $-0.35 \scriptstyle\pm 0.10$ \scriptsize/0.10 & $-16.39 \scriptstyle\pm 0.17$ \scriptsize/0.04 & $17.44 \scriptstyle\pm 0.00$ \scriptsize/0.12 & $10.95 \scriptstyle\pm 0.33$ \scriptsize/0.005 & $-191.22 \scriptstyle\pm 3.19$ \scriptsize/0.08 \\
	\makecell[c]{OS-DRE \\ \scriptsize(IMQ)} & 1 & $\mathbf{-0.69} \scriptstyle\pm 0.18$ \scriptsize/0.08 & $\mathbf{-18.33} \scriptstyle\pm 0.04$ \scriptsize/0.04 & $17.45 \scriptstyle\pm 0.05$ \scriptsize/0.07 & $\mathbf{9.97} \scriptstyle\pm 0.37$ \scriptsize/0.005 & $\mathbf{-217.99} \scriptstyle\pm 3.39$ \scriptsize/\underline{0.07} \\
	\makecell[c]{OS-DRE \\ \scriptsize(RQ)} & 1 & $-0.66 \scriptstyle\pm 0.17$ \scriptsize/\underline{0.08} & $-17.86 \scriptstyle\pm 0.03$ \scriptsize/0.04 & $\mathbf{16.88} \scriptstyle\pm 0.03$ \scriptsize/\underline{0.05} & $11.34 \scriptstyle\pm 0.28$ \scriptsize/\underline{0.003} & $-201.37 \scriptstyle\pm 2.21$ \scriptsize/0.07 \\
		\bottomrule 
	\end{tabular}
\end{table}




\subsection{Application to Continual Kullback-Leibler Divergence Estimation}  
\paragraph{Problem Setup and Evaluation Metric.}
We evaluate OS-DRE on continual Kullback-Leibler (KL) divergence estimation under online distribution-shift scenarios, a setting highly relevant to real-time change-point detection~\citep{chen2021continual} and continuous covariate-shift adaptation~\citep{zhang2023adapting}. 
In this setup, the target distribution evolves over discrete steps $t=0,1,2,\dots$, while the source distribution $p_0$ remains fixed. 
The goal is to continuously estimate the KL divergence from $p_0$ to the current target $p_t$ in real time: $\KL(p_t\|p_0)=\mathbb{E}_{\vx\sim p_t}\left[\log\frac{p_t(\vx)}{p_0(\vx)}\right]$, where the inner log-ratio is directly approximated by DRE. 

We consider three controlled benchmarks with a fixed source $p_0(\vx)=\mathcal{N}(\vzero,\mI_d)$ that isolate different types of nonstationarity: 
(i) linear drift, where the target $p_t=\mathcal{N}(\vmu_t,\mSigma_t)$ has linearly shifting mean and covariance, $\vmu_t = t\Delta\vmu$ and $\mSigma_t = (1 - t\Delta\sigma)\mI_d$ with constant drift rates $\Delta\vmu$ and $\Delta\sigma$; 
(ii) noise corruption, motivated by~\citet{hendrycks2018benchmarking}, where variance inflates progressively with $\vmu_t=\vzero$ and $\mSigma_t=(1 + t\Delta\sigma^2)\mI_d$; 
and (iii) controlled divergence, which enforces a prescribed KL increase $\KL(p_t\|p_0)=t$ by setting $\|\vmu_t\|=\sqrt{2t}$ with random directions~\citep{zhang2023adapting}.

\paragraph{Continual KL Divergence Estimation.}
\cref{fig:distributon-drift-ds} shows KL divergence estimation results across three benchmarks.
With only $\text{NFE}=1$, OS-DRE produces stable, accurate estimates that closely follow the ground truth, whereas solver-based baselines remain unreliable even at $\text{NFE}=50$.
In particular, OS-DRE tracks the drift horizon smoothly (\cref{fig:linearly-drifting-gaussian}), detects discrete change points clearly (\cref{fig:controlled-divergence-shift}), and maintains competitive accuracy as corruption increases, where baselines suffer from systematic underestimation and high variance (\cref{fig:progressive-noise-corruption}).
The key reason is that OS-DRE computes the density ratio analytically rather than relying on numerical solvers, avoiding the error accumulation that makes solver-based methods unstable under tight inference budgets.
\begin{figure}[!ht]
	\centering
	\begin{subfigure}[b]{0.32\linewidth}
		\centering
		\includegraphics[width=\linewidth]{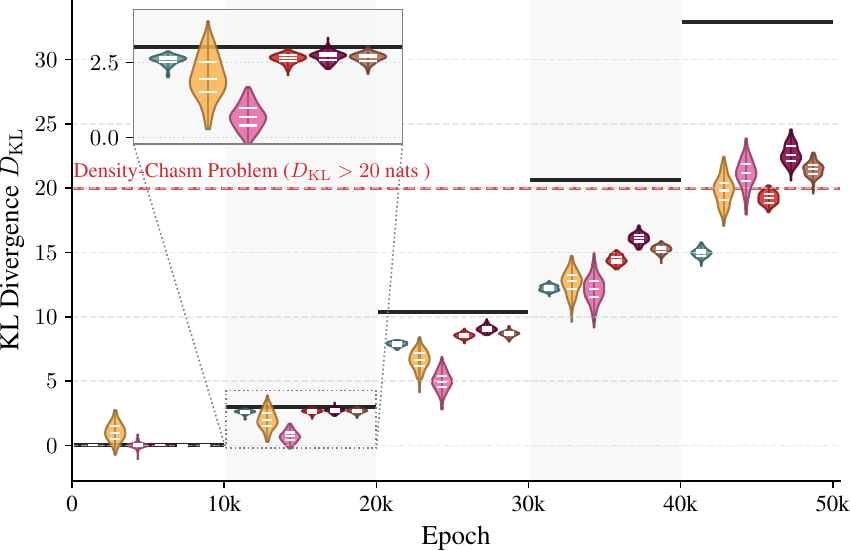}
		\caption{Linear drift.}
		\label{fig:linearly-drifting-gaussian}
	\end{subfigure}
	\hfill
	\begin{subfigure}[b]{0.32\linewidth}
		\centering
		\includegraphics[width=\linewidth]{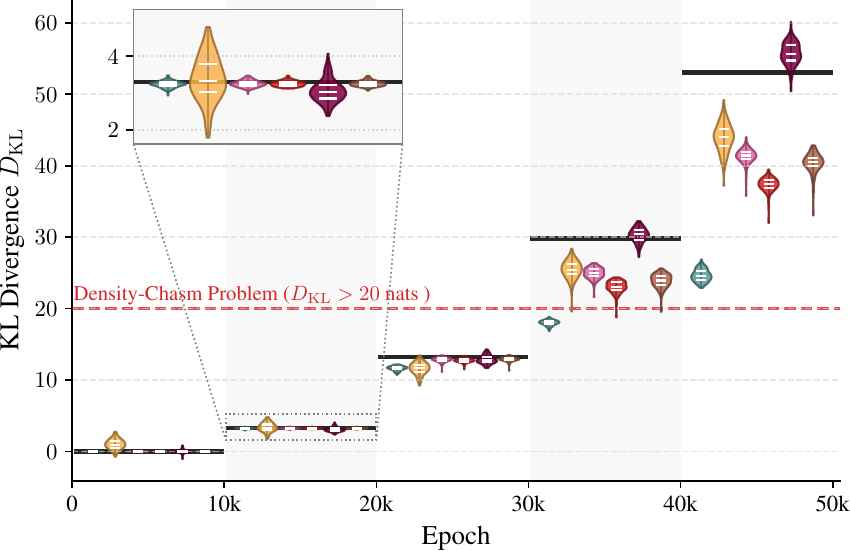}
		\caption{Noise corruption.}
		\label{fig:progressive-noise-corruption}
	\end{subfigure}
	\hfill
	\begin{subfigure}[b]{0.32\linewidth}
		\centering
		\includegraphics[width=\linewidth]{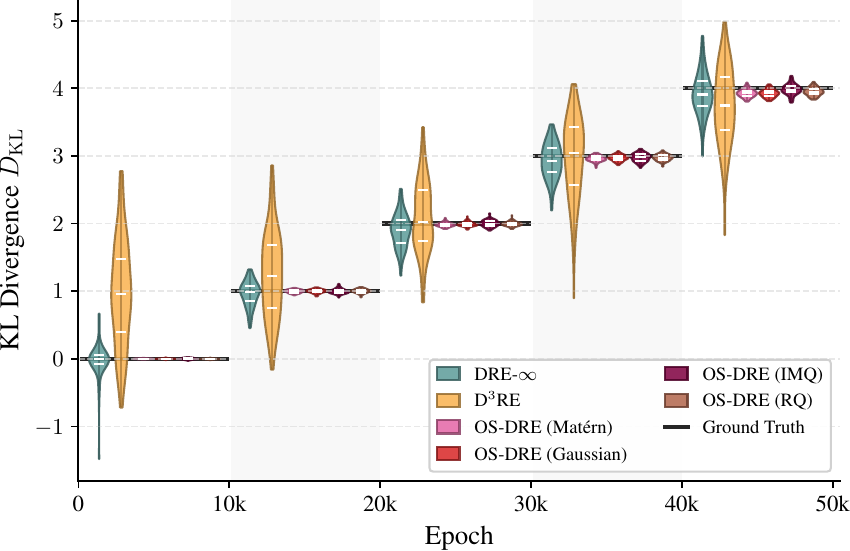}
		\caption{Controlled divergence.}
		\label{fig:controlled-divergence-shift}
	\end{subfigure}
    \caption{Real-time KL divergence estimation under distribution shifts. OS-DRE ($\mathrm{NFE}=1$) stably and accurately tracks the ground truth (black), whereas quadrature-based baselines ($\mathrm{NFE}=50$) suffer from noticeable lag during large distributional changes.}
	\label{fig:distributon-drift-ds}
\end{figure}

\vspace{-5mm}
\subsection{Application to Mutual Information Estimation}
\label{subsec:mi}

\paragraph{Problem Setup and Evaluation Metric.}
We estimate mutual information (MI) between two random variables $\rvx\sim p(\vx)$ and $\rvy\sim q(\vy)$ via DRE by treating the joint distribution $p_1(\vx,\vy)\triangleq p(\vx,\vy)$ as the target and the product of marginals $p_0(\vx,\vy)\triangleq p(\vx)\,q(\vy)$ as the reference. 
The MI is then $\text{MI}(\rvx,\rvy)=\mathbb{E}_{(\vx,\vy)\sim p_1}[\log r(\vx,\vy)]$, where $r(\vx,\vy)=p_1(\vx,\vy) / p_0(\vx,\vy)$. After training OS-DRE to obtain $\log \hat r(\vx,\vy)$, we approximate MI by Monte Carlo over held-out joint samples. Negative samples from $p_0$ are obtained by shuffling within minibatches, while positive pairs are drawn from the true joint \citep{choi2022density}. We report mean squared error (MSE) against ground-truth MI.


\paragraph{Geometrically Pathological Distributions.}
We evaluate OS-DRE on four MI estimation benchmarks with geometrically pathological distributions~\citep{czyz2023beyond}, featuring heavy tails (Half-Cube Map), sharp curvature (Asinh Mapping), non-smooth boundaries (Additive Noise), and asymmetric nonlinear dependencies (Gamma--Exponential). 
These properties violate the local smoothness assumptions of many standard estimators. 

As shown in \cref{fig:mi_mse_pathological}, OS-DRE achieves consistently lower MSE than solver-based baselines across all dependency levels while using only $\text{NFE}=1$. Notably, the Mat\'ern kernel often provides the best fit for asymmetric structures like Gamma--Exponential, while infinitely smooth kernels (Gaussian/RQ) remain competitive on moderate dependencies. These results confirm that OS-DRE's partly analytic temporal representation is robust to geometric pathologies that destabilize numerical quadrature.
\begin{figure}[ht]
    \centering
    \includegraphics[width=\linewidth]{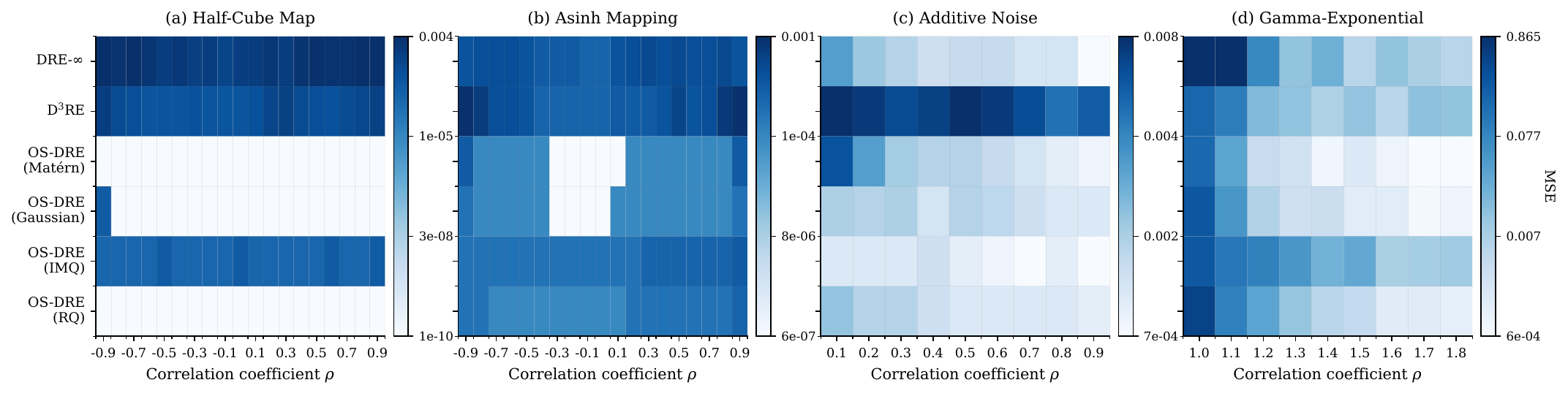}
    \vspace{-2mm}
    \caption{MI estimation error (MSE $\downarrow$) on four geometrically pathological distributions \citep{czyz2023beyond} across varying dependency levels. OS-DRE ($\mathrm{NFE}=1$) consistently yields favorable performance compared to solver-based baselines ($\mathrm{NFE}=100$), demonstrating robust estimation capabilities under complex geometric topologies.}
    \label{fig:mi_mse_pathological}
\end{figure}

Beyond DRE-based baselines, we compare OS-DRE with dedicated MI estimators (KSG \citep{kraskov2004estimating}, MINE~\citep{belghazi2018mutual}, InfoNet~\citep{hu2024infonet}) on the complex settings curated in~\citet{hu2024infonet}. 
Results are presented in \cref{fig:mutual_information_noise_2d}.
On the remaining Gaussian and pathological distributions (\cref{fig:mutual_information_noise_2d}), OS-DRE yields estimates closely matching the ground truth, achieving lower mean absolute error (MAE). 
This demonstrates that our partly analytic, one-step ratio estimation provides a robust alternative where existing estimators struggle with stability in nontrivial geometries.
\begin{figure}[ht]
	\centering
	\begin{subfigure}[b]{0.24\linewidth}
		\centering
		\includegraphics[width=\linewidth]{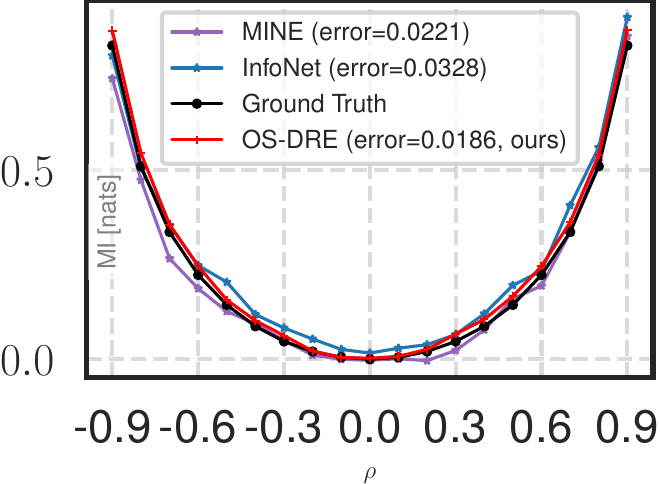}
		\caption{Gauss}
		\label{fig:mutual_information_noise_gauss}
	\end{subfigure}
        \begin{subfigure}[b]{0.24\linewidth}
		\centering
		\includegraphics[width=\linewidth]{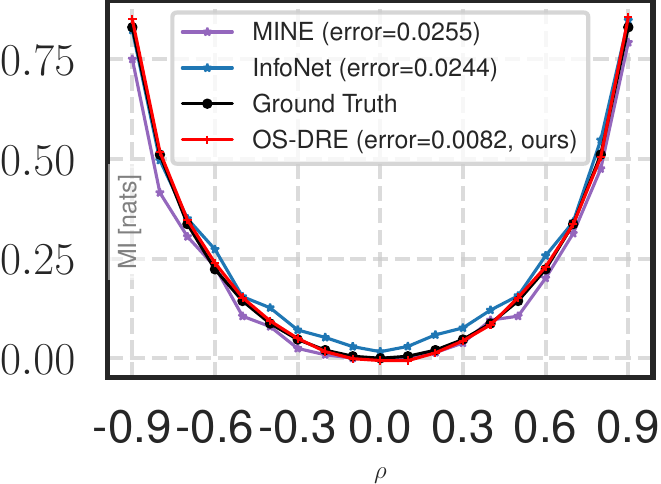}
		\caption{Half-Cube Map}
		\label{fig:mutual_information_noise_half_cube_map}
	\end{subfigure}
	\begin{subfigure}[b]{0.24\linewidth}
		\centering
		\includegraphics[width=\linewidth]{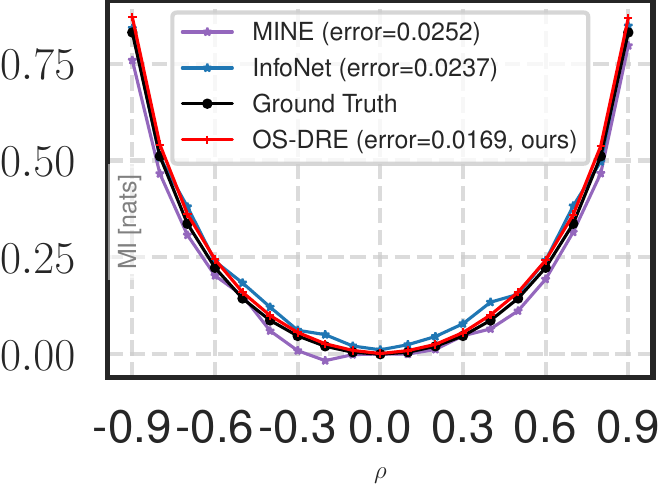}
		\caption{Asinh Mapping}
		\label{fig:mutual_information_noise_asinh_mapping}
	\end{subfigure}
	\begin{subfigure}[b]{0.24\linewidth}
		\centering
		\includegraphics[width=\linewidth]{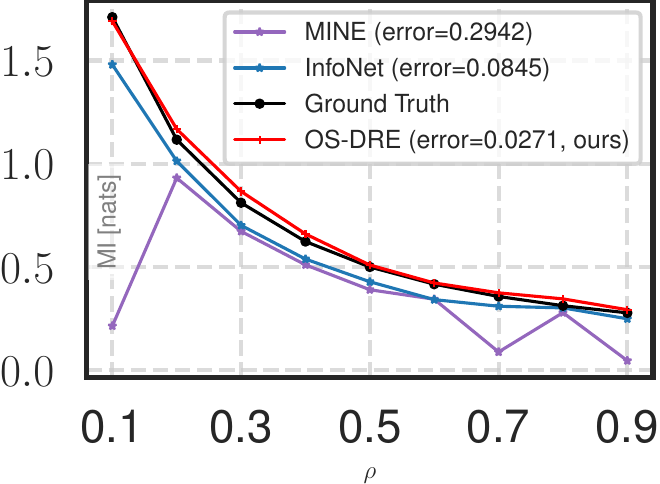}
		\caption{Additive Noise}
		\label{fig:mutual_information_noise_additive_noise}
	\end{subfigure}
    \vspace{-2mm}
        \caption{MI estimates and MAE ($\downarrow$) under the complex dependencies of \citet{hu2024infonet}. OS-DRE provides highly accurate and robust estimates.}
        \label{fig:mutual_information_noise_2d}
\end{figure}

\paragraph{MI Estimation under High-Discrepancy Settings.}
We further evaluate OS-DRE under high-discrepancy settings known to induce the ``density-chasm'' problem ( $\ge 20$ nats~\citep{rhodes2020telescoping}). Following~\citet{choi2022density}, we consider $p_0=\mathcal{N}(\vzero,\mI_d)$ and $p_1=\mathcal{N}(\vzero,\mSigma)$ with block-diagonal $\mSigma$ (pairwise correlation $\rho=0.5$), varying dimension $d\in\{40,80,120,160\}$ to obtain ground-truth MI $\in\{10,20,30,40\}$ nats. 
\begin{table}[ht]
	\centering
	\setlength{\tabcolsep}{2pt}  
	\small
        \caption{MI estimation (MSE $\downarrow$) and inference time (seconds $\downarrow$) in high-discrepancy regimes ($\ge 20$ nats) \citep{rhodes2020telescoping}. OS-DRE successfully mitigates the density-chasm problem, delivering competitive estimation quality with low inference cost.
        \textbf{Best} and \underline{fastest} results are highlighted. Timings are measured on a single NVIDIA TITAN X GPU.
        }
    \begin{tabular}{@{}lccccccccccc@{}} 
        \toprule
        \multirow{2}{*}{\textbf{Method}} & \multirow{2}{*}{\textbf{NFE$\downarrow$}} & \multirow{2}{*}{\textbf{Kernel}} & \multicolumn{2}{c}{\textbf{MI = 10}} & \multicolumn{2}{c}{\textbf{MI = 20}} & \multicolumn{2}{c}{\textbf{MI = 30}} & \multicolumn{2}{c}{\textbf{MI = 40}} \\
        \cmidrule(lr){4-5} \cmidrule(lr){6-7} \cmidrule(lr){8-9} \cmidrule(lr){10-11}
        & & & \textbf{MSE$\downarrow$} & \textbf{Time$\downarrow$} & \textbf{MSE$\downarrow$} & \textbf{Time$\downarrow$} & \textbf{MSE$\downarrow$} & \textbf{Time$\downarrow$} & \textbf{MSE$\downarrow$} & \textbf{Time$\downarrow$} \\
        \midrule
        \multirow{6}{*}{DRE-$\infty$} 
            & $2$ & -- & $73.91$ & $0.045$ & $283.52$ & $0.045$ & $614.62$ & $0.045$ & $1215.69$ & $0.046$ \\
            & $5$ & -- & $2.86$ & $0.055$ & $7.09$ & $0.060$ & $25.29$ & $0.068$ & $70.36$ & $0.057$ \\
            & $10$ & -- & $0.27$ & $0.075$ & $0.54$ & $0.100$ & $2.66$ & $0.085$ & $7.08$ & $0.083$ \\
            & $50$ & -- & $0.03$ & $0.226$ & $0.04$ & $0.249$ & $0.48$ & $0.228$ & $3.77$ & $0.271$ \\
            & $100$ & -- & $0.02$ & $0.475$ & $0.01$ & $0.493$ & $0.50$ & $0.478$ & $3.31$ & $0.554$ \\
            & $200$ & -- & $0.02$ & $0.819$ & $0.01$ & $0.954$ & $0.50$ & $0.879$ & $3.31$ & $0.956$ \\
        \midrule
        \multirow{6}{*}{D$^3$RE} 
            & $2$ & -- & $2.58$ & $0.048$ & $3.65$ & $0.047$ & $6.21$ & $0.046$ & $500.04$ & $0.044$ \\
            & $5$ & -- & $0.01$ & $0.056$ & $0.29$ & $0.061$ & $7.50$ & $0.058$ & $76.80$ & $0.058$ \\
            & $10$ & -- & $0.02$ & $0.075$ & $0.21$ & $0.094$ & $7.72$ & $0.076$ & $59.70$ & $0.080$ \\
            & $50$ & -- & $0.01$ & $0.234$ & $0.09$ & $0.256$ & $8.94$ & $0.256$ & $58.19$ & $0.260$ \\
            & $100$ & -- & $\mathbf{0.00}$ & $0.498$ & $0.08$ & $0.546$ & $9.00$ & $0.487$ & $58.20$ & $0.514$ \\
            & $200$ & -- & $\mathbf{0.00}$ & $0.816$ & $0.05$ & $0.922$ & $9.36$ & $0.907$ & $57.28$ & $0.955$ \\
        \midrule
        \multirow{4}{*}{\shortstack[c]{OS-DRE \\ (ours)}} 
            & $1$ & Matérn & $0.09$ & $0.024$ & $18.30$ & $0.028$ & $208.98$ & $0.032$ & $456.11$ & $0.028$ \\
            & $1$ & Gaussian & $0.01$ & $0.025$ & $\mathbf{0.00}$ & $0.027$ & $\mathbf{0.07}$ & $0.013$ & $2.30$ & $\underline{0.014}$ \\
            & $1$ & IMQ & $0.11$ & $0.035$ & $1.56$ & $0.030$ & $5.86$ & $0.029$ & $\mathbf{0.47}$ & $0.028$ \\
            & $1$ & RQ & $0.03$ & $\underline{0.022}$ & $0.83$ & $\underline{0.012}$ & $1.52$ & $\underline{0.012}$ & $1.41$ & $0.019$ \\
        \bottomrule
    \end{tabular}
\label{tab:MI-high-discrepancy}
\end{table}

As reported in \cref{tab:MI-high-discrepancy}, solver-based baselines  underestimate MI at low inference budgets ($\text{NFE}=2$) and exhibit instability even at $\text{NFE}=50$. In contrast, OS-DRE remains accurate and stable with $\text{NFE}=1$. Kernel choice reveals a clear inductive bias: infinitely smooth kernels (Gaussian/IMQ) are robust to increasing discrepancy, while the limited-smoothness Mat\'ern kernel degrades at higher dimensions. Notably, Gaussian case matches or surpasses DRE-$\infty$ at $\text{NFE}=50$ using a single partly analytic step for MI $\le 30$ nats. These results confirm that obviating the need for iterative integration mitigates error accumulation in the density-chasm regime, enabling reliable MI estimate under strict inference budgets.

\subsection{Application to Near Out-of-Distribution Detection}

\paragraph{Problem Setup and Evaluation Metrics.}
We study image OOD detection via DRE, following \citep{ren2019likelihood,zhang2022ood}.
Let $p_1$ be the in-distribution (ID) and $p_0$ an auxiliary reference.
The OOD score $\log \hat r(\vx)\approx \log p_1(\vx)-\log p_0(\vx)$ is obtained by a single forward pass at inference, with smaller values indicating more OOD-like inputs.
Our detector is trained end-to-end in pixel space without generative pretraining or external features.
We report area under the receiver operating characteristic (AUROC), false positive rate at $95$\% true positive rate (FPR@$95$) and area under the precision-recall curve (AUPR)-IN/OUT.
Unless stated otherwise, $p_0$ is a $300$K cleaned subset of the $80$ Million Tiny Images dataset \citep{torralba200880,zhang2022ood}.
Baselines include neural DRE methods (ChiSq \citep{kanamori2009least}, Logistic \citep{nowozin2016f}, NCE \citep{gutmann2012noise}, InfoNCE \citep{oord2018representation}, PW-DRE \citep{zellinger2025binary}) and score-based DRE methods.

\paragraph{Challenging Near-OOD Detection.}
Following OpenOOD \citep{yang2022openood,zhang2023openood}, we evaluate on CIFAR-$10$ and CIFAR-$100$ \citep{krizhevsky2009learning}. For CIFAR-$10$, CIFAR-$100$ and TinyImageNet \citep[TIN]{le2015tiny} serve as Near-OOD, and MNIST \citep{deng2012mnist} and SVHN \citep{svhn} as Far-OOD. For CIFAR-$100$, CIFAR-$10$ and TIN are Near-OOD, with MNIST and SVHN as Far-OOD. 
OS-DRE consistently outperforms neural DRE baselines: while methods such as ChiSq, Logistic, and InfoNCE fail to exceed $0.60$ AUROC on CIFAR-$10$, OS-DRE achieves $0.89$ overall ($0.86$ Near-OOD, $0.92$ Far-OOD), with similar gains on CIFAR-$100$.
\begin{table}[ht]
\centering
\caption{OOD detection performance on CIFAR-10 and CIFAR-100 (ID). Metrics are reported as mean $\pm$ std over 3 random seeds. Arrows indicate whether higher ($\uparrow$) or lower ($\downarrow$) values are better. Inference time is measured on a single NVIDIA RTX A6000 GPU.}
\label{tab:ood}
{\small  
\setlength{\tabcolsep}{2pt}
\begin{tabular}{lrrrrrrrr}
\toprule
Method & \rotatebox{45}{Near-AUROC$\uparrow$} & \rotatebox{45}{Far-AUROC$\uparrow$} & \rotatebox{45}{AUROC$\uparrow$} & \rotatebox{45}{FPR@$95$$\downarrow$} & \rotatebox{45}{AUPR-IN$\uparrow$} & \rotatebox{45}{AUPR-OUT$\uparrow$} & \rotatebox{45}{Params$\downarrow$} & \rotatebox{45}{Time$\downarrow$} \\
\midrule
\multicolumn{9}{c}{\textit{CIFAR-10}} \\
\midrule
ChiSq & $0.53_{\scriptstyle\pm0.00}$ & $0.59_{\scriptstyle\pm0.16}$ & $0.56_{\scriptstyle\pm0.08}$ & $0.87_{\scriptstyle\pm0.13}$ & $0.47_{\scriptstyle\pm0.09}$ & $0.67_{\scriptstyle\pm0.03}$ & $11.2$M & $38.4_{\scriptstyle\pm0.6}$ \\
Logistic & $0.53_{\scriptstyle\pm0.02}$ & $0.55_{\scriptstyle\pm0.05}$ & $0.54_{\scriptstyle\pm0.02}$ & $0.92_{\scriptstyle\pm0.04}$ & $0.41_{\scriptstyle\pm0.02}$ & $0.67_{\scriptstyle\pm0.02}$ & $11.2$M & $38.9_{\scriptstyle\pm1.3}$ \\
NCE & $0.53_{\scriptstyle\pm0.02}$ & $0.55_{\scriptstyle\pm0.05}$ & $0.54_{\scriptstyle\pm0.02}$ & $0.92_{\scriptstyle\pm0.04}$ & $0.41_{\scriptstyle\pm0.02}$ & $0.67_{\scriptstyle\pm0.02}$ & $11.2$M & $39.7_{\scriptstyle\pm1.7}$ \\
InfoNCE & $0.54_{\scriptstyle\pm0.01}$ & $0.65_{\scriptstyle\pm0.06}$ & $0.59_{\scriptstyle\pm0.03}$ & $0.93_{\scriptstyle\pm0.01}$ & $0.49_{\scriptstyle\pm0.01}$ & $0.69_{\scriptstyle\pm0.01}$ & $11.2$M & $40.1_{\scriptstyle\pm0.7}$ \\
PW-DRE & $0.54_{\scriptstyle\pm0.01}$ & $0.56_{\scriptstyle\pm0.10}$ & $0.55_{\scriptstyle\pm0.05}$ & $0.93_{\scriptstyle\pm0.05}$ & $0.45_{\scriptstyle\pm0.04}$ & $0.67_{\scriptstyle\pm0.03}$ & $11.2$M & $39.9_{\scriptstyle\pm0.7}$ \\
DRE-$\infty$ & $0.59_{\scriptstyle\pm0.00}$ & $0.73_{\scriptstyle\pm0.02}$ & $0.66_{\scriptstyle\pm0.01}$ & $0.89_{\scriptstyle\pm0.05}$ & $0.47_{\scriptstyle\pm0.05}$ & $0.72_{\scriptstyle\pm0.01}$ & $11.2$M & $284.4_{\scriptstyle\pm0.4}$ \\
D$^3$RE & $0.85_{\scriptstyle\pm0.00}$ & $0.91_{\scriptstyle\pm0.03}$ & $0.88_{\scriptstyle\pm0.01}$ & $0.46_{\scriptstyle\pm0.04}$ & $0.76_{\scriptstyle\pm0.05}$ & $\mathbf{0.92}_{\scriptstyle\pm0.00}$ & $11.2$M & $285.5_{\scriptstyle\pm0.8}$ \\
\rowcolor{gray!10} OS-DRE & $\mathbf{0.86}_{\scriptstyle\pm0.02}$ & $\mathbf{0.92}_{\scriptstyle\pm0.01}$ & $\mathbf{0.89}_{\scriptstyle\pm0.01}$ & $\mathbf{0.44}_{\scriptstyle\pm0.06}$ & $\mathbf{0.79}_{\scriptstyle\pm0.03}$ & $0.92_{\scriptstyle\pm0.01}$ & $11.4$M & \underline{$37.6_{\scriptstyle\pm0.7}$} \\
\midrule
\multicolumn{9}{c}{\textit{CIFAR-100}} \\
\midrule
ChiSq & $0.52_{\scriptstyle\pm0.03}$ & $0.51_{\scriptstyle\pm0.18}$ & $0.51_{\scriptstyle\pm0.10}$ & $0.90_{\scriptstyle\pm0.06}$ & $0.41_{\scriptstyle\pm0.07}$ & $0.65_{\scriptstyle\pm0.03}$ & $11.2$M & $42.4_{\scriptstyle\pm0.8}$ \\
Logistic & $0.52_{\scriptstyle\pm0.01}$ & $0.51_{\scriptstyle\pm0.04}$ & $0.52_{\scriptstyle\pm0.02}$ & $0.96_{\scriptstyle\pm0.00}$ & $0.41_{\scriptstyle\pm0.01}$ & $0.65_{\scriptstyle\pm0.00}$ & $11.2$M & $41.8_{\scriptstyle\pm0.7}$ \\
NCE & $0.52_{\scriptstyle\pm0.01}$ & $0.51_{\scriptstyle\pm0.04}$ & $0.52_{\scriptstyle\pm0.02}$ & $0.96_{\scriptstyle\pm0.00}$ & $0.41_{\scriptstyle\pm0.01}$ & $0.65_{\scriptstyle\pm0.00}$ & $11.2$M & $42.7_{\scriptstyle\pm0.7}$ \\
InfoNCE & $0.52_{\scriptstyle\pm0.01}$ & $0.57_{\scriptstyle\pm0.12}$ & $0.55_{\scriptstyle\pm0.06}$ & $0.93_{\scriptstyle\pm0.04}$ & $0.46_{\scriptstyle\pm0.04}$ & $0.65_{\scriptstyle\pm0.03}$ & $11.2$M & $42.1_{\scriptstyle\pm1.5}$ \\
PW-DRE & $0.53_{\scriptstyle\pm0.01}$ & $0.56_{\scriptstyle\pm0.10}$ & $0.54_{\scriptstyle\pm0.05}$ & $0.91_{\scriptstyle\pm0.05}$ & $0.42_{\scriptstyle\pm0.02}$ & $0.65_{\scriptstyle\pm0.02}$ & $11.2$M & $42.2_{\scriptstyle\pm1.7}$ \\
DRE-$\infty$ & $0.51_{\scriptstyle\pm0.00}$ & $\mathbf{0.71}_{\scriptstyle\pm0.02}$ & $0.61_{\scriptstyle\pm0.01}$ & $\mathbf{0.84}_{\scriptstyle\pm0.09}$ & $\mathbf{0.51}_{\scriptstyle\pm0.05}$ & $0.67_{\scriptstyle\pm0.00}$ & $11.2$M & $283.6_{\scriptstyle\pm0.9}$ \\
D$^3$RE & $\mathbf{0.61}_{\scriptstyle\pm0.00}$ & $0.69_{\scriptstyle\pm0.01}$ & $0.65_{\scriptstyle\pm0.01}$ & $0.90_{\scriptstyle\pm0.01}$ & $0.47_{\scriptstyle\pm0.02}$ & $\mathbf{0.73}_{\scriptstyle\pm0.01}$ & $11.2$M & $285.4_{\scriptstyle\pm1.9}$ \\
\rowcolor{gray!10} OS-DRE & $0.61_{\scriptstyle\pm0.01}$ & $0.70_{\scriptstyle\pm0.01}$ & $\mathbf{0.65}_{\scriptstyle\pm0.00}$ & $0.90_{\scriptstyle\pm0.01}$ & $0.49_{\scriptstyle\pm0.02}$ & $0.72_{\scriptstyle\pm0.01}$ & $11.4$M & \underline{$36.6_{\scriptstyle\pm0.6}$} \\
\bottomrule
\end{tabular}
}
\end{table}

Neural DRE methods, such as ChiSq, Logistic, and InfoNCE, struggle to exceed $0.60$ in AUROC on CIFAR-$10$, whereas OS-DRE achieves an AUROC of $0.89$, alongside $0.86$ on Near-OOD and $0.92$ on Far-OOD tasks. This performance advantage persists on the more challenging CIFAR-$100$ dataset. 
Furthermore, score-based methods (DRE-$\infty$, D$^3$RE) incur over $283$s inference latency due to numerical solvers; OS-DRE completes evaluation in $\sim$$37$s (${\sim}7.5\times$ speedup) while matching their estimation quality. 

\paragraph{Score Distribution Characteristics.}
The behavior of OS-DRE can be further understood by examining the distribution of OOD scores. 

\cref{fig:density-of-ood-scores-cifar100} shows kernel density estimates on CIFAR-$100$.
For Near-OOD, neural DRE methods and DRE-$\infty$ produce limited score shifts with considerable ID/OOD overlap. D$^3$RE and OS-DRE yield clearer separation, with ID scores concentrated in higher regions and Near-OOD scores shifted lower.
Specifically, ID samples concentrate in a higher-score region, whereas Near-OOD samples are shifted toward lower-score ranges.
For Far-OOD, all methods show improved separation, but D$^3$RE and OS-DRE maintain more consistent patterns.
Overall, the difference in score distribution centers and shapes provides a reliable basis for anomaly detection across both regimes.
\begin{figure}[!ht]
    \centering
    \includegraphics[width=0.9\textwidth]{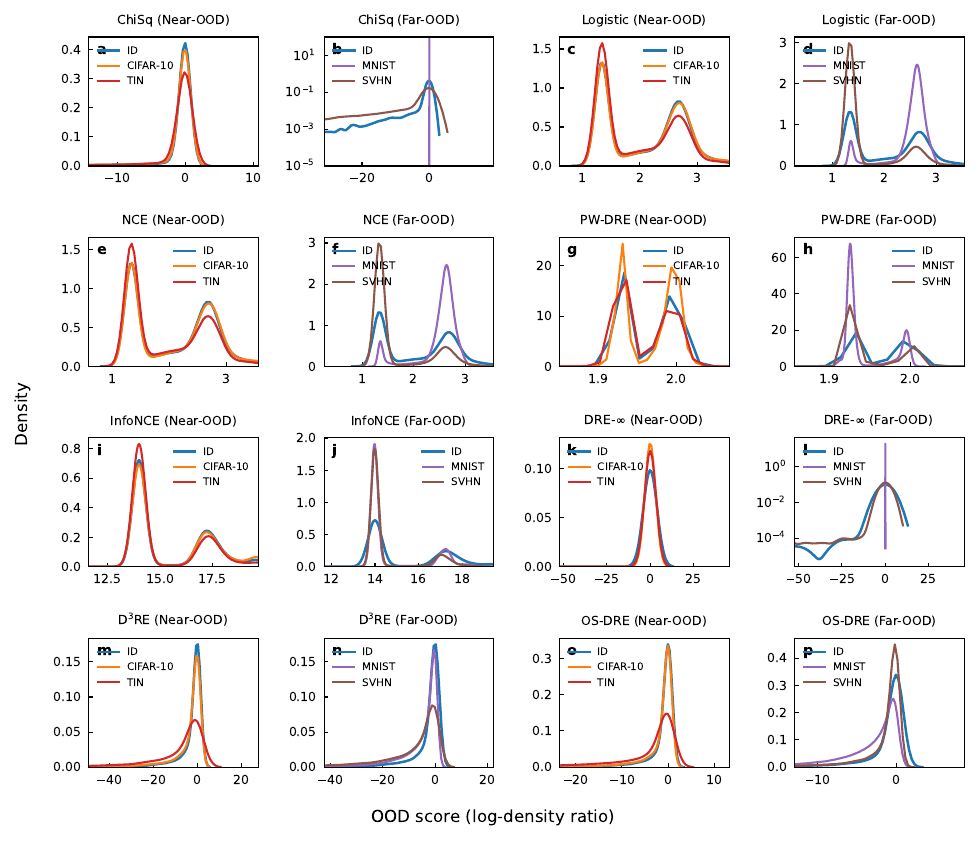}
    \vspace{-3mm}
    \caption{Densities of OOD scores for ID CIFAR-100 (blue) vs. Near- and Far-OOD datasets. 
    OS-DRE and D$^3$RE exhibit clearer distributional discrepancies.}
    \label{fig:density-of-ood-scores-cifar100}
\end{figure}

\subsection{Ablation Studies}
\label{subsec:ablation}
Proposition~\ref{proposition:error-bound} highlights two key terms governing the trade-off between estimation quality and computational efficiency: the number of basis functions $K$ and the kernel $\phi$.
\paragraph{Choice of $K$ and Sensitivity Analysis.}
{The parameter $K$ controls temporal capacity, and OS-DRE retains $\text{NFE}=1$
for every choice of $K$. The main experiments fix $K$ within each task family: $K=400$ for
structured 2D density, tabular density, continual KL and distribution shift,
high-discrepancy Gaussian MI, and CIFAR-$10$/$100$ OOD, and $K=200$ for the geometrically
pathological MI benchmarks, where this capacity already suffices in the reported
experiments.}

{To measure the sensitivity to $K$ directly, we report the high-discrepancy
Gaussian MI benchmarks in dimensions $40$, $80$, $120$ and $160$ --- the four settings of
\cref{tab:MI-high-discrepancy} --- at $K\in\{200,300,400,500,600,800\}$.
\Cref{tab:k-sensitivity-mi} reports the mean absolute error (MAE) between the estimated
and the true $f$-divergence at the last evaluation.
A dash marks a configuration for which no result is available.}

\begin{table}[ht]
\centering\small
\setlength{\tabcolsep}{6pt}
\caption{High-discrepancy Gaussian MI: $K$-sensitivity of the estimation
error. Mean absolute error (MAE$\downarrow$) between the estimated and the true
$f$-divergence at the last evaluation. Dashes mark configurations for
which no result is available.}
\label{tab:k-sensitivity-mi}
\begin{tabular}{@{}rrrrrrrr@{}}
\toprule
Dimension & True MI & \multicolumn{6}{c}{MAE$\downarrow$} \\
\cmidrule(lr){3-8}
 & & $K=200$ & $K=300$ & $K=400$ & $K=500$ & $K=600$ & $K=800$ \\
\midrule
$40$ & $10$ & $0.3740$ & $0.1149$ & $0.4641$ & $0.2574$ & $0.2842$ & --- \\
$80$ & $20$ & $0.6852$ & $0.5209$ & $0.5867$ & $0.2514$ & $0.0982$ & --- \\
$120$ & $30$ & $0.2007$ & $0.0262$ & $0.3900$ & $0.1970$ & $0.3016$ & $0.0867$ \\
$160$ & $40$ & $1.5023$ & $1.4194$ & $1.6605$ & $0.7152$ & $0.5049$ & $0.2426$ \\
\bottomrule
\end{tabular}
\end{table}

{The effect of $K$ differs across the four settings rather than following a
single trend. For $d=80$ the MAE falls from $0.6852$ at $K=200$ to $0.0982$ at $K=600$,
and for $d=160$ it falls from $1.6605$ at $K=400$ to $0.2426$ at $K=800$. Other entries
move in the opposite direction over part of their range: for $d=40$ the MAE at $K=400$
($0.4641$) exceeds its values at $K=200$ ($0.3740$) and $K=300$ ($0.1149$), and for $d=120$
it rises from $0.0262$ at $K=300$ to $0.3900$ at $K=400$ before falling again at larger
$K$. Neither a uniform improvement nor a uniform degradation with $K$ describes these
results, and they support selecting $K$ together with a validation criterion and the
computational budget rather than increasing it uniformly.}

{Varying $K$ also changes the initial RBF scale ($\sigma_k^{(0)}=1/K$), the center
spacing ($\Delta c=1/(K-1)$ under $c_k=(k-1)/(K-1)$), and the output dimension of the
coefficient head. Thus changing $K$ adjusts temporal resolution and coefficient-head size
together.
Lemma~\ref{lemma:rbf-density} establishes asymptotic denseness as $K\to\infty$, but because
the fixed-scale temporal spaces need not be nested, this denseness does not imply a
monotonic reduction of the finite-$K$ approximation error.
Corollary~\ref{corollary:joint-expressive-capacity} lifts the denseness to the joint
data-time setting, and Corollary~\ref{corollary:sample-size-capacity} separates the
resulting temporal approximation from a capacity-dependent generalization term, giving a
statistical reason to balance $K$ against the available data. Because that bound couples
$K$ with the target regularity, the sample size and the coefficient-class complexity, it
does not predict a monotone dependence on $K$.}

\paragraph{Choice of RBF Kernel ($\phi$).}
{Kernel usage follows the same task-family rule. Results that compare all four RBF
families (Matérn, Gaussian, IMQ, RQ) form the within-task kernel comparison; task families
with a single reported kernel use RQ for structured 2D and IMQ for CIFAR-$10$/$100$ OOD.
For a new task, $K$ and the kernel family can be selected with a held-out validation
criterion, keeping the study protocol independent of the reported benchmark settings.}

Among four kernels (Gaussian, IMQ, RQ, Matérn), IMQ and RQ are suitable general-purpose choices, performing competitively across density estimation (\cref{tab:density-estimation-tabular}), MI estimation (\cref{tab:MI-high-discrepancy}), and continual KL divergence estimation (\cref{fig:distributon-drift-ds}).
The Gaussian kernel performs well on GAS and BSDS300.
The Matérn kernel, with its limited smoothness, excels on geometrically pathological tasks (e.g., Gamma--Exponential) where its inductive bias better matches the target regularity.

\paragraph{Quality-Efficiency Trade-off.}
In score-based DRE, overall error is often dominated by numerical integration bias at low NFE rather than score estimation error.
OS-DRE eliminates this bottleneck via closed-form estimation: as shown in \cref{fig:accuracy-vs-efficiency}, it matches or surpasses DRE-$\infty$ and D$^3$RE at $\text{NFE}=1$.
On density estimation (\cref{tab:density-estimation-tabular,fig:accuracy-vs-efficiency-tabular}), it achieves comparable or better NLL across all five tabular datasets, while baselines require $\text{NFE}\in[2,50]$ for similar performance (${\sim}50\times$ reduction).
A similar pattern holds for MI estimation (\cref{tab:MI-high-discrepancy,fig:accuracy-vs-efficiency-high-discrepancy}), where OS-DRE attains near-zero MSE at $\text{NFE}=1$ for $\text{MI}\in\{10,20,30\}$.
These results indicate that, on the reported benchmarks, OS-DRE maintains high estimation quality without costly numerical integration under strict evaluation budgets.

\begin{figure}[htbp]
	\centering
	\begin{subfigure}[b]{0.95\linewidth}
		\centering
		\includegraphics[width=\linewidth]{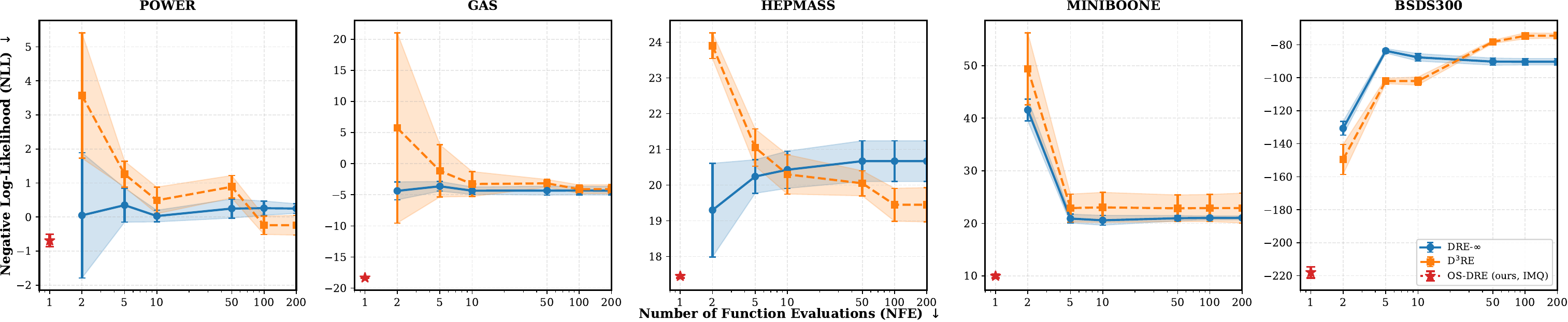}
		\caption{NLL (Qual.) vs. NFE (Effi.) for density estimation on five tabular datasets (cf. \cref{tab:density-estimation-tabular}).}
		\label{fig:accuracy-vs-efficiency-tabular}
	\end{subfigure}
	
	\begin{subfigure}[b]{0.95\linewidth}
		\centering
		\includegraphics[width=\linewidth]{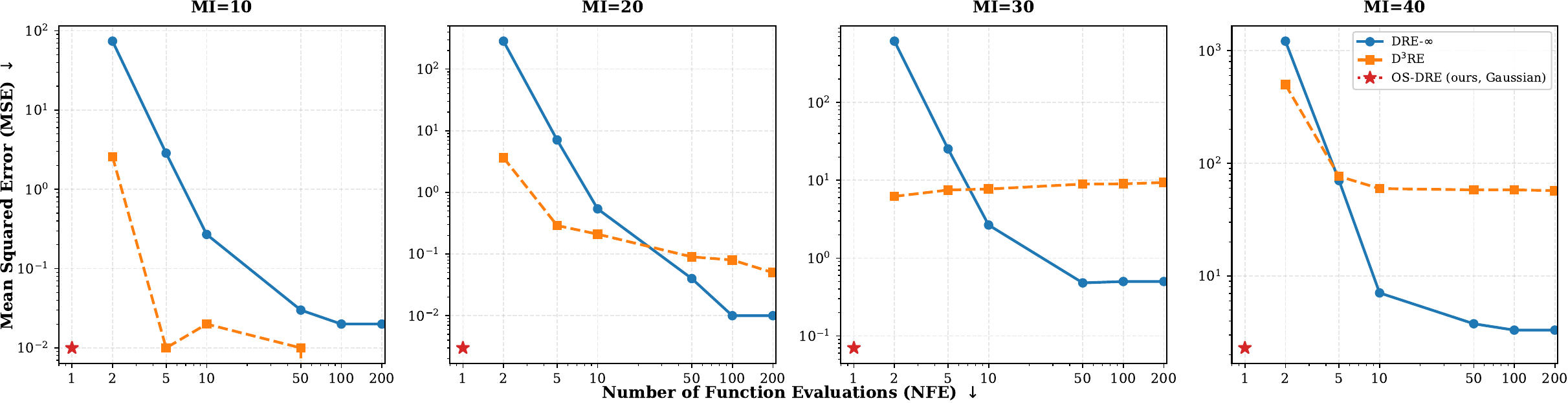}
		\caption{MSE (Qual.) vs. NFE (Effi.) for MI estimation under high-discrepancy settings (cf. \cref{tab:MI-high-discrepancy}).}
		\label{fig:accuracy-vs-efficiency-high-discrepancy}
	\end{subfigure}
        \caption{Trade-off between estimation quality and computational efficiency. }
	\label{fig:accuracy-vs-efficiency}
\end{figure}

\section{Conclusion and Future Work}
We have presented OS-DRE, a partly analytic,  solver-free framework for score-based density ratio estimation that offers a favorable balance between estimation quality and computational efficiency. 
By decomposing the time score into spatial coefficients and a finite analytic RBF temporal dictionary, the proposed method converts the path integral into a closed-form weighted sum.
This formulation obviates the need for numerical solvers and enables density ratio estimation with one single function evaluation. 
The theoretical analysis studies conditions related to approximation capacity, finite-term linear independence, and analytic tractability of the finite temporal approximation.
By examining specific kernels in both finitely and infinitely smooth regimes, we further obtain approximation error bounds with algebraic and exponential decay. 
These results clarify the approximation behavior of the finite-dimensional subspaces and provide guidance for selecting temporal kernels in practice. 
Empirically, OS-DRE shows competitive performance across density estimation, continual KL divergence estimation, mutual information estimation, and challenging near out-of-distribution (OOD) detection benchmarks, while offering substantial speedups relative to solver-based baselines. 
Taken together, these findings suggest that OS-DRE is a practical option for probabilistic inference under strict time or resource constraints.

The current framework also has limitations and natural extensions.
At a fixed $K$, the temporal approximation bound depends on the regularity and the Sobolev
norm of the scalar time-score integrand. Sharp temporal transitions can increase this norm, and
when transition locations and widths vary with $\vx$, a shared temporal dictionary must
accommodate differing temporal profiles. Corollary~\ref{corollary:sample-size-capacity}
quantifies a capacity--sample-size trade-off under its stated assumptions. An adaptive-center
extension could concentrate the same budget of temporal atoms on rapidly varying regions
while preserving the closed-form temporal operators of \cref{subsec:rbf-construction},
provided its center geometry satisfies the conditions of \cref{subsec:error-analysis}.

Looking ahead, although this work focuses on density ratio estimation, the design generalizes to any target that admits a path-integral representation of a learned integrand: an analytic temporal representation can replace numerical quadrature with a single inner product. From this perspective, OS-DRE provides a concrete example of solver-free analytic temporal integration and suggests a broader design principle for efficient path-integral estimators beyond DRE.


\acks{This work was supported in part by grants from National Natural Science Foundation of China (52539005), the China Scholarship Council (202306150167), the fundamental research program of Guangdong, China (2023A1515011281), Guangdong Basic and Applied Basic Research Foundation (24202107190000687), Foshan Science and Technology Research Project (2220001018608).}


\appendix
\section*{Appendix}


\section{Kernel Conventions and Analytic Formulas}
\label{appendix:rbf-formulas}
This section collects the fixed-scale kernel conventions and the analytic formulas for the
RBF families used by OS-DRE. \cref{appendix:kernel-conventions} defines the kernel, its
RKHS, $c$-universality, and the conventions for center sets; the subsequent subsections
verify $c$-universality separately for every family and give closed-form temporal integrals
and derivatives. Coefficient uniqueness follows under the common-scale conditions of
Proposition~\ref{proposition:rbf-approximation-scheme}.

{\subsection{Kernel and Approximation-Space Conventions}
\label{appendix:kernel-conventions}
\sloppy

\paragraph{Fixed-Scale Kernels and Native Spaces.}
Let $\phi:[0,\infty)\to\mathbb R$ be a generating function and fix a reference shape parameter
$\sigma_0>0$, independent of $K$. The corresponding fixed-scale kernel specializes
\cref{eq:rbf-definition} to
\begin{equation}
\label{eq:fixed-shape-kernel}
\kappa_{\phi,\sigma_0}(s,t)=\phi\left(\frac{|s-t|}{\sigma_0}\right),\qquad s,t\in\mathbb R.
\end{equation}

When $\kappa_{\phi,\sigma_0}$ is positive definite, it induces a reproducing kernel Hilbert
space $\mathcal N_{\phi,\sigma_0}(\mathbb R)$ with the reproducing property
$f(t)=\langle f,\kappa_{\phi,\sigma_0}(\cdot,t)\rangle_{\mathcal N_{\phi,\sigma_0}(\mathbb R)}$
for every $f\in\mathcal N_{\phi,\sigma_0}(\mathbb R)$. Its restriction to $[0,1]$ is
\begin{equation}
\label{eq:restricted-native-space-manuscript}
\begin{aligned}
\mathcal N_{\phi,\sigma_0}([0,1])
&\triangleq\{f|_{[0,1]}:f\in\mathcal N_{\phi,\sigma_0}(\mathbb R)\},\\
\|u\|_{\mathcal N_{\phi,\sigma_0}([0,1])}
&\triangleq\inf_{f|_{[0,1]}=u}\|f\|_{\mathcal N_{\phi,\sigma_0}(\mathbb R)}.
\end{aligned}
\end{equation}

The infimum is attained by the minimum-norm extension. Indeed, the functions in the
full-space RKHS that vanish on $[0,1]$ form a closed subspace: this subspace is the
intersection of the kernels of the continuous point-evaluation maps. Orthogonal
projection onto its orthogonal complement gives the unique minimum-norm extension.
The restricted space is itself the RKHS of the restricted kernel, since restricting
a kernel section preserves the reproducing identity with this extension.

\paragraph{Practical and Comparison Spaces.}
The main text defines the optimization domain $\mathcal A_K\subset(0,\infty)^K$ for the shape
parameters, the vector of learnable shape parameters $\boldsymbol\sigma_{1:K}$, and the
corresponding RBF span $\mathcal V_K(\boldsymbol\sigma_{1:K})$ in
\cref{eq:trainable-shape-parameter-vector}. Setting $\sigma_k=\sigma_0$ for every $k$ gives the
comparison vector $\boldsymbol\sigma_{0,1:K}$, whose entries all equal $\sigma_0$, and the
comparison space $\mathcal V_K(\boldsymbol\sigma_{0,1:K})$. This common configuration is assumed
to belong to $\mathcal A_K$ whenever it is used as a comparison, so
$\mathcal V_K(\boldsymbol\sigma_{0,1:K})$ is a member of the trainable family
$\mathcal V_K(\boldsymbol\sigma_{1:K})$.

\paragraph{Admissible Comparison.}
Since $\boldsymbol\sigma_{0,1:K}\in\mathcal A_K$, the best-error inequality
\begin{equation}
\label{eq:admissible-comparison}
\inf_{\boldsymbol\sigma_{1:K}\in\mathcal A_K}
\inf_{v\in\mathcal V_K(\boldsymbol\sigma_{1:K})}
\|s^{(\rvt)}(\vx,\cdot)-v\|_{\mathcal H_t}
\le
\inf_{v\in\mathcal V_K(\boldsymbol\sigma_{0,1:K})}
\|s^{(\rvt)}(\vx,\cdot)-v\|_{\mathcal H_t}
\end{equation}
holds, so the space obtained by setting every shape parameter to $\sigma_0$ provides an
admissible comparison.

\paragraph{Sobolev-Space Convention.}
For a nonnegative integer $m$, the Sobolev space $\mathcal W^{m,2}([0,1])$ consists of the
functions $u\in L^2([0,1])$ whose weak derivatives $u^{(j)}$ belong to $L^2([0,1])$ for
$j=1,\ldots,m$, with norm
\begin{equation}
\label{eq:sobolev-convention-integer}
\|u\|_{\mathcal W^{m,2}([0,1])}^{2}
\triangleq\sum_{j=0}^{m}\|u^{(j)}\|_{L^2([0,1])}^{2}.
\end{equation}

For a noninteger order $s=m+\alpha$, with $m\in\mathbb N_0$ and $0<\alpha<1$, we use the
Sobolev--Slobodeckij norm \citep[Section~2]{dinezza2012hitchhiker}
\begin{equation}
\label{eq:sobolev-convention-fractional}
\|u\|_{\mathcal W^{s,2}([0,1])}^{2}
\triangleq\|u\|_{\mathcal W^{m,2}([0,1])}^{2}
+\int_0^1\int_0^1
\frac{|u^{(m)}(t)-u^{(m)}(r)|^{2}}{|t-r|^{1+2\alpha}}\,\mathrm dt\,\mathrm dr.
\end{equation}

Since $[0,1]$ is a bounded Lipschitz domain, standard Sobolev extension results give a
bounded linear extension from $\mathcal W^{s,2}([0,1])$ to
$\mathcal W^{s,2}(\mathbb R)$ \citep{adams2003sobolev,dinezza2012hitchhiker}; the extension
constant is absorbed into the theorem constants. In the present one-dimensional setting
the condition $\beta>1/2$ serves both as the refined sampling threshold
$\beta>d/2$ with $d=1$ for the interpolation result of
Proposition~\ref{proposition:error-bound} and as a sufficient condition for continuous
representatives via bounded point evaluation.

\paragraph{\texorpdfstring{$c$-Universality}{c-Universality}.}
The restriction of $\kappa_{\phi,\sigma_0}$ to $[0,1]^2$ is called $c$-universal if
$\mathcal N_{\phi,\sigma_0}([0,1])$ is dense in $C([0,1])$ under the uniform norm:
for every $q\in C([0,1])$ and $\delta>0$, there exist finitely many centers
$y_1,\ldots,y_m\in[0,1]$ and coefficients $a_1,\ldots,a_m$ such that
$\|q-\sum_{j=1}^m a_j\kappa_{\phi,\sigma_0}(\cdot,y_j)\|_{C([0,1])}<\delta$.
Here and in the density results we assume the kernel is continuous. To justify the
finite-section formulation, their span is dense in the restricted RKHS: an element
orthogonal to every section vanishes at every point by the reproducing property.
Moreover, for any RKHS element $u$,
$\|u\|_{C([0,1])}\le
\sup_{t\in[0,1]}\sqrt{\kappa_{\phi,\sigma_0}(t,t)}\,
\|u\|_{\mathcal N_{\phi,\sigma_0}([0,1])}$.
The supremum is finite by continuity and compactness, so RKHS-norm approximation by
finite sections also implies uniform approximation.

\paragraph{Geometry of Center Sets.}
For a finite center set $\mathcal C_K\triangleq\{c_k\}_{k=1}^{K}\subset[0,1]$, define its fill
distance, separation radius, and mesh ratio by
\begin{equation}
\label{eq:center-geometry-manuscript}
\begin{aligned}
h_K&\triangleq\sup_{t\in[0,1]}\min_{c\in\mathcal C_K}|t-c|,\\
\operatorname{sep}_K&\triangleq\frac12\min_{i\ne j}|c_i-c_j|,\qquad
\rho_K\triangleq\frac{h_K}{\operatorname{sep}_K}.
\end{aligned}
\end{equation}

The center sets $\mathcal C_K$ are quasi-uniform if $\sup_K\rho_K<\infty$. The condition
$h_K\to0$ ensures that every ideal center can be replaced by a nearby center from the
implemented finite grid. For the endpoint-inclusive uniform grid
$c_k=(k-1)/(K-1)$, $K\ge2$, used in the reported experiments, the fill distance and
separation radius both equal $h_K=\operatorname{sep}_K=1/[2(K-1)]$, and the mesh ratio
equals $\rho_K=1$.

\fussy
}

\subsection{Matérn RBFs}
\label{appendix:matern-rbfs}

The Matérn family of RBFs is widely used in machine learning, particularly in Gaussian processes, as their smoothness is controlled by a parameter $\nu$. We consider the common case where $\nu=3/2$, which corresponds to a once-differentiable function. The generating function is strictly positive definite and is given by $\phi(r) = (1 + \sqrt{3}r)\exp(-\sqrt{3}r)$.
The basis functions, which are in the Sobolev space $\mathcal W^{2,2}(\mathbb R)$, are:
\begin{equation}\label{eq:appendix-matern}
	g_k(t) = \left(1 + \frac{\sqrt{3}|t-c_k|}{\sigma_k}\right) \exp\left(-\frac{\sqrt{3}|t-c_k|}{\sigma_k}\right).
\end{equation}

\paragraph{Universality on $[0,1]$.}
Fix a reference shape parameter $\sigma_0>0$ and write
$a\triangleq\sqrt{3}/\sigma_0$. The associated
translation-invariant kernel is
$\kappa_{\phi,\sigma_0}(s,t)=(1+a|s-t|)\exp(-a|s-t|)$. Under the convention
$\widehat\psi(\omega)=\int_{\mathbb R}\psi(h)e^{-i\omega h}\,\mathrm dh$, its spectral
density is
{\begin{equation}
\label{eq:appendix-matern-spectrum}
\widehat\psi(\omega)=\frac{4a^3}{(a^2+\omega^2)^2}>0,
\qquad \omega\in\mathbb R.
\end{equation}}
For a direct calculation, the Fourier transform of $e^{-a|h|}$ is
\begin{equation}
\label{eq:laplace-profile-spectrum}
\int_0^\infty e^{-(a+i\omega)h}\,\mathrm dh+
\int_0^\infty e^{-(a-i\omega)h}\,\mathrm dh
=\frac{1}{a+i\omega}+\frac{1}{a-i\omega}
=\frac{2a}{a^2+\omega^2}.
\end{equation}
Since $(1+a|h|)e^{-a|h|}=(1-a\partial_a)e^{-a|h|}$, apply
$1-a\partial_a$ to this transform to obtain \cref{eq:appendix-matern-spectrum}.
Differentiation under the integral is justified by an integrable exponential
envelope for $a$ in any compact subinterval of $(0,\infty)$.
{Hence the spectral support is all of $\mathbb R$, and
\citet[Section~3.1]{sriperumbudur2011universality} then implies that this Mat\'ern-$3/2$
kernel is $c_0$-universal on $\mathbb R$: its RKHS is uniformly dense in
$C_0(\mathbb R)$, the continuous functions that vanish at infinity.

For the restriction to $[0,1]$, extend any $q\in C([0,1])$ to a continuous compactly
supported function on $\mathbb R$, for example by linear interpolation from its endpoint
values to zero at $t=-1$ and $t=2$ and by zero outside $[-1,2]$. Approximate this
extension uniformly on $\mathbb R$ by an element of
$\mathcal N_{\phi,\sigma_0}(\mathbb R)$. Restriction to $[0,1]$ gives a function in the
restricted RKHS, which the finite-section density argument of
\cref{appendix:kernel-conventions} replaces by a finite linear combination of sections
centered in $[0,1]$. The fixed-scale kernel is therefore $c$-universal on $[0,1]$:
finite linear combinations of $\kappa_{\phi,\sigma_0}(\cdot,y)$ with $y\in[0,1]$ are
uniformly dense in $C([0,1])$.}

\paragraph{Closed-Form Expression for the Temporal Integral.}
The integral $\bar{g}_k = \int_0^1 g_k(t) \mathrm{d}t$ is computed by splitting the integral at the center $c_k$ due to the absolute value. The indefinite integral of the generating function is $\int \phi(r) \mathrm{d}r = -r e^{-\sqrt{3}r} - \frac{2}{\sqrt{3}}e^{-\sqrt{3}r}$. Evaluating this over the respective intervals yields the final closed form.

\paragraph{Closed-Form Expression for the Temporal Derivative.}
The derivative of the generating function is $\phi^{\prime}(r) = -3r\exp(-\sqrt{3}r)$. Applying the chain rule, we find the derivative:
\begin{align*}
	g_k^{\prime}(t) &= \phi^{\prime}\left(\frac{|t - c_k|}{\sigma_k}\right) \cdot \frac{\mathrm{sgn}(t-c_k)}{\sigma_k} \\
	&= -3 \frac{|t-c_k|}{\sigma_k} \exp\left(-\frac{\sqrt{3}|t-c_k|}{\sigma_k}\right) \cdot \frac{t-c_k}{|t-c_k|} \cdot \frac{1}{\sigma_k} \\
	&= -\frac{3(t-c_k)}{\sigma_k^2} \exp\left(-\frac{\sqrt{3}|t-c_k|}{\sigma_k}\right).
\end{align*}

\subsection{Gaussian RBFs}
\label{appendix:gaussian-rbfs}

The Gaussian RBF is defined by the generating function $\phi(r) = \exp(-r^2)$. The basis functions are therefore given by:
\begin{equation}\label{eq:appendix-gaussian}
	g_k(t) = \exp\left(-\frac{|t - c_k|^2}{\sigma_k^2}\right),
\end{equation}
where $ c_k $ and $ \sigma_k > 0 $ are the center and shape parameters of $g_k$. 

\paragraph{Universality on $[0,1]$.}
For a fixed reference shape parameter $\sigma_0>0$,
$\kappa_{\phi,\sigma_0}(s,t)=\exp(-|s-t|^2/\sigma_0^2)$ is the Gaussian kernel
$\exp(-\alpha|s-t|^2)$ with $\alpha=\sigma_0^{-2}$. {Theorem~17 and the examples
immediately following it in \citet{micchelli2006universal} establish universality of this
kernel on every compact subset of $\mathbb R$. In particular, finite linear combinations of
$\kappa_{\phi,\sigma_0}(\cdot,y)$, $y\in[0,1]$, are uniformly dense in $C([0,1])$.}

\paragraph{Closed-Form Expression for the Temporal Integral.}
The integral $\bar{g}_k = \int_0^1 g_k(t) \mathrm{d}t$ is calculated as follows. We use the substitution $u = (t - c_k)/\sigma_k$, which implies $\mathrm{d}t = \sigma_k \mathrm{d}u$.
\begin{align*}
	\bar{g}_k = \int_0^1 \exp\left(-\frac{|t - c_k|^2}{\sigma_k^2}\right) \mathrm{d}t 
	= \sigma_k \int_{-c_k/\sigma_k}^{(1 - c_k)/\sigma_k} \exp(-u^2) \mathrm{d}u.
\end{align*}

This integral can be expressed using the error function, $\mathrm{erf}(z) = \frac{2}{\sqrt{\pi}} \int_0^z \exp(-x^2) \mathrm{d}x$. Since $\int_a^b \exp(-u^2)\mathrm{d}u = \frac{\sqrt{\pi}}{2}(\mathrm{erf}(b) - \mathrm{erf}(a))$, we have:
\begin{align*}
	\bar{g}_k &= \sigma_k \frac{\sqrt{\pi}}{2} \left[ \mathrm{erf}\left(\frac{1 - c_k}{\sigma_k}\right) - \mathrm{erf}\left(-\frac{c_k}{\sigma_k}\right) \right] = \frac{\sigma_k\sqrt{\pi}}{2} \left[ \mathrm{erf}\left(\frac{1 - c_k}{\sigma_k}\right) + \mathrm{erf}\left(\frac{c_k}{\sigma_k}\right) \right],
\end{align*}
where the last step uses the property $\mathrm{erf}(-z) = -\mathrm{erf}(z)$.

\paragraph{Closed-Form Expression for the Temporal Derivative.}
The derivative $g_k^{\prime}(t)$ is found by applying the chain rule:
\begin{align*}
	g_k^{\prime}(t) &= \frac{\mathrm{d}}{\mathrm{d}t} \exp\left(-\frac{(t - c_k)^2}{\sigma_k^2}\right) \\
	&= \exp\left(-\frac{(t - c_k)^2}{\sigma_k^2}\right) \cdot \frac{\mathrm{d}}{\mathrm{d}t}\left(-\frac{(t - c_k)^2}{\sigma_k^2}\right) \\
	&= \exp\left(-\frac{(t - c_k)^2}{\sigma_k^2}\right) \cdot \left(-\frac{2(t - c_k)}{\sigma_k^2}\right).
\end{align*}


\subsection{Inverse Multiquadric RBFs}
\label{appendix:imq-rbfs}

The Inverse Multiquadric (IMQ) RBF is defined by the generating function $\phi(r) = (r^2 + 1)^{-1/2}$. The basis functions are therefore given by:
\begin{equation}\label{eq:appendix-imq}
	g_k(t) = \left(\frac{(t - c_k)^2}{\sigma_k^2} + 1\right)^{-\frac{1}{2}} = \frac{\sigma_k}{\sqrt{(t - c_k)^2 + \sigma_k^2}}.
\end{equation}

\paragraph{Universality on $[0,1]$.}
For a fixed reference shape parameter $\sigma_0>0$,
\begin{equation}
\label{eq:appendix-imq-fixed-kernel}
\kappa_{\phi,\sigma_0}(s,t)
=\left(1+\frac{|s-t|^2}{\sigma_0^2}\right)^{-1/2}
=\sigma_0(\sigma_0^2+|s-t|^2)^{-1/2}.
\end{equation}
{Theorem~17 and the examples immediately following it in \citet{micchelli2006universal}
establish universality of $(\beta+|s-t|^2)^{-\alpha}$ for all $\alpha,\beta>0$.
Taking $\alpha=1/2$ and $\beta=\sigma_0^2$ verifies universality of
Eq.~\eqref{eq:appendix-imq-fixed-kernel}; the positive factor $\sigma_0$ does not change
the linear span. Thus the fixed-scale standard IMQ kernel is $c$-universal on $[0,1]$ even
though its radial profile is not integrable on $[0,\infty)$.}

\sloppy
\paragraph{Fixed-Scale Exponential Interpolation Estimate.}
{\begin{assumption}[Fixed-scale exponential interpolation estimate]
\label{assumption:fixed-scale-exponential-interpolation}
For the fixed-scale standard IMQ kernel defined above, there exist constants
$C>0$, $\lambda\in(0,1)$, and $h_0>0$, depending on the fixed kernel and $[0,1]$
but not on $K$ or $f$, such that, for every
$f\in\mathcal N_{\phi,\sigma_0}([0,1])$ and every center set $\mathcal C_K$ of distinct
centers with $h_K\le h_0$,
\begin{equation}
\label{eq:imq-classical-interpolation-assumption}
\|f-I_{\mathcal C_K}^{\sigma_0}f\|_{\mathcal H_t}
\le C\lambda^{1/h_K}\|f\|_{\mathcal N_{\phi,\sigma_0}([0,1])}.
\end{equation}
\end{assumption}}

We verify this assumption for distinct centers in $[0,1]$ using the
fixed-scale interpolation estimate below.
Write $\mathcal N=\mathcal N_{\phi,\sigma_0}([0,1])$ and let $\Pi_{\mathcal C_K}$
be the $\mathcal N$-orthogonal projector onto the span of the kernel sections
at the centers. This projector uses the native inner product. By the reproducing
property, $f-\Pi_{\mathcal C_K}f$ vanishes at every center. Since the interpolation
matrix is positive definite for distinct centers, the interpolant is unique;
hence $\Pi_{\mathcal C_K}f=I_{\mathcal C_K}^{\sigma_0}f$.
Define the power function by
\begin{equation}
\label{eq:imq-power-function-definition}
\mathcal P_{\mathcal C_K}(t)
\triangleq\|(I-\Pi_{\mathcal C_K})\kappa_{\phi,\sigma_0}(\cdot,t)\|_{\mathcal N}.
\end{equation}
The reproducing property, orthogonality and Cauchy--Schwarz give, for each $t$,
\begin{equation}
\label{eq:imq-power-function-error}
\begin{aligned}
|f(t)-I_{\mathcal C_K}^{\sigma_0}f(t)|
&=|\langle (I-\Pi_{\mathcal C_K})f,
       (I-\Pi_{\mathcal C_K})\kappa_{\phi,\sigma_0}(\cdot,t)\rangle_{\mathcal N}|\\
&\le\|(I-\Pi_{\mathcal C_K})f\|_{\mathcal N}\mathcal P_{\mathcal C_K}(t)\\
&\le\|f\|_{\mathcal N}\mathcal P_{\mathcal C_K}(t).
\end{aligned}
\end{equation}
{\citet[Eq.~(5)]{wenzel2023greedy} records exponential power-function decay
for inverse multiquadrics on cubes. The standard exponent $1/2$ is explicitly covered by
\citet[Eq.~(1.4) and Theorem~2.4]{luh2017crucial}. To match notations, take their dimension
$d=1$, their exponent parameter $p=-1$, corresponding to $\alpha=1/2$ in our notation, and
their length-scale parameter $\beta=\sigma_0^2$. The result is a positive constant multiple
of \cref{eq:appendix-imq-fixed-kernel}. Its order of conditional positive definiteness is
zero, so the interpolant contains no polynomial term. The fixed multiplicative factor
is absorbed into the native-norm constant.

Theorem~2.4 requires every subinterval of length $\delta$ in $[0,1]$ to contain a center.
This holds for every $\mathcal C_K$ with fill distance $h_K$ when $\delta=3h_K$:
the center nearest the midpoint lies within distance $h_K<\delta/2$.
Apply the theorem to a minimum-norm native-space extension of $f$ to $\mathbb R$
(Appendix~\ref{appendix:kernel-conventions}). Taking $h_0=\delta_0/3$, where $\delta_0$
is the theorem's threshold, gives constants $C_1,C_2>0$ such that, for every finite
set of distinct centers with $h_K\le h_0$ and every
$f\in\mathcal N_{\phi,\sigma_0}([0,1])$,
\begin{equation}
\label{eq:imq-uniform-interpolation-bound}
\|f-I_{\mathcal C_K}^{\sigma_0}f\|_{L^\infty([0,1])}
\le C_1e^{-C_2/h_K}\|f\|_{\mathcal N_{\phi,\sigma_0}([0,1])}.
\end{equation}
The constants and threshold depend only on the fixed scale and interval, not on the
particular center set, $K$, or $f$. Because $[0,1]$ has unit measure, the same estimate
holds in $L^2([0,1])$, verifying
Assumption~\ref{assumption:fixed-scale-exponential-interpolation} with
$\lambda=e^{-C_2}$.
A concrete class of admissible targets consists of finite linear combinations of
fixed-scale IMQ kernel sections $\kappa_{\phi,\sigma_0}(\cdot,y)$ with $y\in[0,1]$,
which belong to $\mathcal N_{\phi,\sigma_0}([0,1])$; one example is
$f(t)=\kappa_{\phi,\sigma_0}(t,1/2)$. The endpoint-inclusive
uniform grid $c_k=(k-1)/(K-1)$ has $h_K=1/(2(K-1))$, so these targets and centers
satisfy the interpolation conditions for all sufficiently large $K$.}
\fussy

\paragraph{Closed-Form Expression for the Temporal Integral.}
The temporal integral $\bar{g}_k = \int_0^1 g_k(t) \mathrm{d}t$ is calculated using the standard integral for the inverse hyperbolic sine function. We use the substitution $u = t - c_k$, which implies $\mathrm{d}t = \mathrm{d}u$.
\begin{align*}
	\bar{g}_k = \int_0^1 \frac{\sigma_k}{\sqrt{(t - c_k)^2 + \sigma_k^2}} \mathrm{d}t 
	= \sigma_k \int_{-c_k}^{1-c_k} \frac{1}{\sqrt{u^2 + \sigma_k^2}} \mathrm{d}u.
\end{align*}

The integral of $1/\sqrt{u^2+a^2}$ is $\ln(u + \sqrt{u^2+a^2})$. Applying this, we get:
\begin{align*}
	\bar{g}_k &= \sigma_k \left[ \ln\left( u + \sqrt{u^2 + \sigma_k^2} \right) \right]_{-c_k}^{1 - c_k} \\
	&= \sigma_k \left( \ln\left( (1-c_k) + \sqrt{(1-c_k)^2 + \sigma_k^2} \right) - \ln\left( -c_k + \sqrt{c_k^2 + \sigma_k^2} \right) \right) \\
	&= \sigma_k \ln \left( \frac{(1 - c_k) + \sqrt{(1 - c_k)^2 + \sigma_k^2}}{-c_k + \sqrt{c_k^2 + \sigma_k^2}} \right).
\end{align*}

\paragraph{Closed-Form Expression for the Temporal Derivative.}
The derivative $g_k^{\prime}(t)$ is found by applying the chain rule to $g_k(t) = \sigma_k \left( (t - c_k)^2 + \sigma_k^2 \right)^{-1/2}$:
\begin{align*}
	g_k^{\prime}(t) &= \sigma_k \cdot \frac{\mathrm{d}}{\mathrm{d}t} \left( (t - c_k)^2 + \sigma_k^2 \right)^{-1/2} \\
	&= \sigma_k \cdot \left(-\frac{1}{2}\right) \left( (t - c_k)^2 + \sigma_k^2 \right)^{-3/2} \cdot \frac{\mathrm{d}}{\mathrm{d}t}\left( (t - c_k)^2 + \sigma_k^2 \right) \\
	&= \sigma_k \cdot \left(-\frac{1}{2}\right) \left( (t - c_k)^2 + \sigma_k^2 \right)^{-3/2} \cdot 2(t - c_k) \\
	&= -\frac{\sigma_k (t - c_k)}{\left( (t - c_k)^2 + \sigma_k^2 \right)^{3/2}}.
\end{align*}

\subsection{Rational Quadratic RBFs}
\label{appendix:rq-rbfs}

The Rational Quadratic (RQ) kernel can be viewed as an infinite sum of Gaussian kernels of different scales. This property makes it a robust choice, capable of modeling data at multiple scales. It is strictly positive definite, and its generating function is $\phi(r) = (1 + r^2)^{-1}$.
The basis functions are therefore given by:
\begin{equation}\label{eq:appendix-rq}
	g_k(t) = \left(1 + \frac{(t - c_k)^2}{\sigma_k^2}\right)^{-1} = \frac{\sigma_k^2}{(t - c_k)^2 + \sigma_k^2}.
\end{equation}

\paragraph{Universality on $[0,1]$.}
For a fixed reference shape parameter $\sigma_0>0$,
\begin{equation}
\label{eq:appendix-rq-fixed-kernel}
\kappa_{\phi,\sigma_0}(s,t)
=\left(1+\frac{|s-t|^2}{\sigma_0^2}\right)^{-1}
=\sigma_0^2(\sigma_0^2+|s-t|^2)^{-1}.
\end{equation}
{This is the family $(\beta+|s-t|^2)^{-\alpha}$ covered by
\citet[Theorem~17 and the subsequent examples]{micchelli2006universal}, with
$\alpha=1$ and $\beta=\sigma_0^2$, up to the positive factor $\sigma_0^2$. It is therefore
$c$-universal on $[0,1]$: finite linear combinations of its sections centered in $[0,1]$
are uniformly dense in $C([0,1])$.}

\paragraph{Closed-Form Expression for the Temporal Integral.}
The temporal integral $\bar{g}_k = \int_0^1 g_k(t) \mathrm{d}t$ is calculated using the standard integral for the arctangent function. We use the substitution $u = t - c_k$, which implies $\mathrm{d}t = \mathrm{d}u$.
\begin{equation}
	\bar{g}_k = \int_0^1 \frac{\sigma_k^2}{(t - c_k)^2 + \sigma_k^2} \mathrm{d}t = \sigma_k^2 \int_{-c_k}^{1-c_k} \frac{1}{u^2 + \sigma_k^2} \mathrm{d}u. \notag
\end{equation}

The integral of $1/(u^2+a^2)$ is $\frac{1}{a}\arctan(\frac{u}{a})$. Applying this, we get:
\begin{align*}
	\bar{g}_k = \sigma_k^2 \left[ \frac{1}{\sigma_k}\arctan\left(\frac{u}{\sigma_k}\right) \right]_{-c_k}^{1 - c_k} = \sigma_k \left( \arctan\left(\frac{1 - c_k}{\sigma_k}\right) + \arctan\left(\frac{c_k}{\sigma_k}\right) \right),
\end{align*}
where the last step uses the property $\arctan(-z) = -\arctan(z)$.

\paragraph{Closed-Form Expression for the Temporal Derivative.}
The derivative $g_k^{\prime}(t)$ is found by applying the chain rule to $g_k(t) = \sigma_k^2 \left( (t - c_k)^2 + \sigma_k^2 \right)^{-1}$:
\begin{align*}
	g_k^{\prime}(t) = \sigma_k^2 \cdot \frac{\mathrm{d}}{\mathrm{d}t} \left( (t - c_k)^2 + \sigma_k^2 \right)^{-1} = -\frac{2\sigma_k^2 (t - c_k)}{\left( (t - c_k)^2 + \sigma_k^2 \right)^2}.
\end{align*}

\section{Proofs}

\subsection{Proof of Lemma~\ref{lemma:space-S-subspace-L2}}
\label{proof:space-S-subspace-L2}


\begin{proof}
	We want to show that for $\forall s^{(\rvt)}\in\mathcal{S} $. It satisfies:
	\begin{equation}
		\|s^{(\rvt)}\|_{L^2(\mathcal{X} \times [0,1])}^2 = \int_{\mathcal{X}} \int_0^1 \left| s^{(\rvt)}(\vx,t) \right|^2  \mathrm{d}t  \mathrm{d}\vx < \infty.
	\end{equation}
	
	Based on \cref{assumption:logp-positive} and \cref{assumption:logp-partial-derivative-bounded}, we have: (1) $ p_t(\vx) \geq C > 0 $, so $ \frac{1}{p_t(\vx)} \leq \frac{1}{C} $; (2) $ \left| \partial_t p_t(\vx) \right| \leq D $. Then, using the chain rule, we write $ s^{(\rvt)}(\vx, t) $ as:
	\begin{equation}
		\left| s^{(\rvt)}(\vx,t) \right|^2=\left| \partial_t \log p_t(\vx) \right|^2 = \left|\frac{\partial_t p_t(\vx)}{p_t(\vx)} \right|^2 \leq \left(\frac{D}{C}\right)^2.
	\end{equation}
	
	We now integrate over $ \mathcal{X} \times [0,1] $:
	\begin{equation}
		\begin{aligned}
			\|s^{(\rvt)}\|_{L^2(\mathcal{X} \times [0,1])}^2 &= \int_{\mathcal{X}} \int_0^1 \left| s^{(\rvt)}(\vx,t) \right|^2  \mathrm{d}t  \mathrm{d}\vx \\
			&\leq \int_{\mathcal{X}} \int_0^1 \left( \frac{D}{C} \right)^2  \mathrm{d}t \mathrm{d}\vx\\
			&=\left( \frac{D}{C} \right)^2\int_{\mathcal{X}} \int_0^1   \mathrm{d}t \mathrm{d}\vx =\left( \frac{D}{C} \right)^2\int_{\mathcal{X}} \mathrm{d}\vx.
		\end{aligned}
	\end{equation}
	
	Since $ \mathcal{X} $ is  measurable and bounded, then $\int_{\mathcal{X}}  \mathrm{d}\vx<\infty$ holds, and $\|s^{(\rvt)}\|_{L^2(\mathcal{X} \times [0,1])}^2<\infty$.
	Therefore, the space $ \mathcal{S} $ of functions $ s^{(\rvt)}$ is a subset of $ L^2(\mathcal{X} \times [0,1]) $. This complete the proof.
	
\end{proof}

\subsection{Proof of Lemma~\ref{lemma:log-density-ratio-decomposition}}
\label{proof:log-density-ratio-decomposition}


\begin{proof}
The target log-density ratio is defined as the temporal integral of the time score function: $\log r(\vx) = \int_{0}^{1} s^{(\rvt)}(\vx, t) \mathrm{d}t$.

For a fixed $\vx$, we can express this integral as an inner product in the Hilbert space $\mathcal{H}_t = L^2([0,1])$ between the function $s^{(\rvt)}(\vx, \cdot)$ and the constant function $\vone(t) \equiv 1$.
\begin{equation}
    \log r(\vx) = \langle s^{(\rvt)}(\vx, \cdot), \vone \rangle_{\mathcal{H}_t}.
\end{equation}
    
	Since $\{g_k\}_{k=1}^{+\infty}$ is a complete orthonormal basis for $\mathcal{H}_t$, the time score has the series expansion $s^{(\rvt)}(\vx, t) = \sum_{k=1}^{+\infty} h_k(\vx) g_k(t)$, which converges in the $L^2$-norm. Due to the continuity of the inner product in a Hilbert space, we can interchange the inner product and the infinite summation:
	\begin{equation}
		\begin{aligned}
			\log r(\vx) &= \left\langle \sum_{k=1}^{+\infty} h_k(\vx) g_k(t), \vone(t) \right\rangle_{\mathcal{H}_t} \\
			&= \sum_{k=1}^{+\infty} \langle h_k(\vx) g_k(t), \vone(t) \rangle_{\mathcal{H}_t} \\
			&= \sum_{k=1}^{+\infty} h_k(\vx) \langle g_k(t), \vone(t) \rangle_{\mathcal{H}_t} \\
			&= \sum_{k=1}^{+\infty} h_k(\vx) \int_{0}^{1} g_k(t) \mathrm{d}t.
		\end{aligned}
	\end{equation}
	This completes the proof.
\end{proof}


\subsection{Proof of Lemma~\ref{lemma:tensor-product-of-frames}}
\label{proof:tensor-product-of-frames}

\begin{proof}
We begin with the \textit{algebraic tensor product} $\mathcal{H}_{\vx} \otimes \mathcal{H}_t$, consisting of finite sums of elementary tensors $f \otimes g$ with $f \in \mathcal{H}_{\vx}$ and $g \in \mathcal{H}_t$. Although dense in  $\mathcal{H}_{\vx,t}$, this space is not complete. Its completion under the induced inner product is the \textit{Hilbert tensor product}, $\mathcal{H}_{\vx} \hat{\otimes} \mathcal{H}_t$, which is isometrically isomorphic to the space  $\mathcal{H}_{\vx,t}$ \citep{kadison1986fundamentals}.
This isomorphism identifies the tensor product frame with a frame for $\mathcal{H}_{\vx,t}$.

Let $F$ be an arbitrary element in $\mathcal{H}_{\vx} \hat{\otimes} \mathcal{H}_t \cong \mathcal{H}_{\vx,t}$. Let $h_k(\vx) \triangleq \langle F(\vx, \cdot), g_k \rangle_{\mathcal{H}_t}$. The sum of the squared frame coefficients for $F$ can be bounded as follows:
	\begin{equation}
		\begin{aligned}
			\sum_{l=1}^{+\infty} \sum_{k=1}^{+\infty} |\langle F, f_l \otimes g_k \rangle|^2 &= \sum_{k=1}^{+\infty} \left( \sum_{l=1}^{+\infty} |\langle h_k, f_l \rangle_{\mathcal{H}_{\vx}}|^2 \right) \quad (\boldsymbol{\star}) \\
			&\le \sum_{k=1}^{+\infty} B_f \|h_k\|_{\mathcal{H}_{\vx}}^2 = B_f \sum_{k=1}^{+\infty} \|h_k\|_{\mathcal{H}_{\vx}}^2 \quad (\boldsymbol{\star\star}) \\
			&= B_f \int_{\mathcal{X}} \left( \sum_{k=1}^{+\infty} |\langle F(\vx, \cdot), g_k \rangle_{\mathcal{H}_t}|^2 \right) \, \mathrm{d}\vx \quad (\boldsymbol{\star\star\star}) \\
			&\le B_f \int_{\mathcal{X}} B_g \|F(\vx, \cdot)\|_{\mathcal{H}_t}^2 \, \mathrm{d}\vx \quad (\boldsymbol{\star\star\star\star}) \\
			&= B_f B_g \|F\|_{\mathcal{H}_{\vx,t}}^2.
		\end{aligned}
	\end{equation}
    
An analogous derivation provides the lower bound, $A_f A_g \|F\|_{\mathcal{H}_{\vx,t}}^2 \le \sum_{l,k} |\langle F, f_l \otimes g_k \rangle|^2$. The key steps are:
	$(\boldsymbol{\star})$ Rewriting the sum by substituting the definition of $h_k$.
	$(\boldsymbol{\star\star})$ Applying the upper frame bound for the spatial frame $\{f_l\}$ for each fixed $k$.
	$(\boldsymbol{\star\star\star})$ Using Fubini's theorem to swap the summation and integration.
	$(\boldsymbol{\star\star\star\star})$ Applying the upper frame bound for the temporal frame $\{g_k\}$ for each fixed $\vx$.Rec
	This completes the proof.
\end{proof}

\subsection{Proof of Theorem~\ref{theorem:frame-based-representation}}
\label{proof:log-density-ratio-frame-based}


\begin{proof}
Let $\{\widetilde f_l\}$ and $\{\widetilde g_k\}$ be the canonical dual frames.
Since $s^{(\rvt)}\in\mathcal H_{\vx,t}$, Fubini's theorem gives
$s^{(\rvt)}(\vx,\cdot)\in\mathcal H_t$ for almost every $\vx$.
{For almost every $\vx$, define
$h_k(\vx)=\langle s^{(\rvt)}(\vx,\cdot),\widetilde g_k\rangle_{\mathcal H_t}$.
These scalar functions are measurable, as parameter integrals of measurable
functions; their defining integrals are finite by Cauchy--Schwarz in $t$.
The upper frame bound $B_{\widetilde g}\le 1/A_g$ of the temporal dual frame gives
\begin{equation}
\label{eq:temporal-coefficient-summability}
\begin{aligned}
\sum_{k=1}^{\infty}\|h_k\|_{\mathcal H_{\vx}}^2
&=\int_{\mathcal X}\sum_{k=1}^{\infty}
|\langle s^{(\rvt)}(\vx,\cdot),\widetilde g_k\rangle_{\mathcal H_t}|^2\,\mathrm d\vx
\\&\le B_{\widetilde g}\|s^{(\rvt)}\|_{\mathcal H_{\vx,t}}^2.
\end{aligned}
\end{equation}

The equality uses Tonelli's theorem for the nonnegative summands; the inequality
applies the dual frame bound to almost every temporal slice and then integrates.}
Thus every $h_k$ belongs to $\mathcal H_{\vx}$. Spatial reconstruction yields
$h_k=\sum_l\langle h_k,\widetilde f_l\rangle_{\mathcal H_{\vx}}f_l$, with norm convergence,
and Fubini's theorem identifies its coefficients as
$c_{l,k}=\langle s^{(\rvt)},\widetilde f_l\otimes\widetilde g_k\rangle_{\mathcal H_{\vx,t}}$.
For this application of Fubini, absolute integrability follows from
$\int_{\mathcal X}\int_0^1
|s^{(\rvt)}(\vx,t)\widetilde f_l(\vx)\widetilde g_k(t)|\,\mathrm dt\,\mathrm d\vx
\le\|s^{(\rvt)}\|_{\mathcal H_{\vx,t}}
\|\widetilde f_l\|_{\mathcal H_{\vx}}\|\widetilde g_k\|_{\mathcal H_t}$.
Hence this $h_k$ coincides with the one defined in the theorem statement,
$h_k=\sum_l c_{l,k}f_l$, with convergence in $\mathcal H_{\vx}$.
Let $S_f$ and $S_g$ denote the spatial and temporal frame operators. The tensor product
of the canonical dual frames is the canonical dual of the tensor product frame:
the frame operator of $\{f_l\otimes g_k\}_{l,k}$ acts on an elementary tensor
$a\otimes b$ as $(S_fa)\otimes(S_gb)$. This follows by factoring
$\langle a\otimes b,f_l\otimes g_k\rangle=
\langle a,f_l\rangle\langle b,g_k\rangle$ in the frame-operator sum;
the bounded analysis and synthesis operators justify the resulting sums.
Equality extends from finite sums of elementary tensors to the completed product
space by boundedness and density. Thus the product frame operator is
$S_f\otimes S_g$, with inverse $S_f^{-1}\otimes S_g^{-1}$, and its canonical dual
elements are $\widetilde f_l\otimes\widetilde g_k$.
Frame reconstruction therefore gives the stated double expansion. Its coefficient
array is square-summable: applying the spatial dual-frame bound to the expansion of
$h_k$ and then \cref{eq:temporal-coefficient-summability} gives
$\sum_{l,k}|c_{l,k}|^2\le B_{\widetilde f}\sum_k\|h_k\|_{\mathcal H_{\vx}}^2
\le B_{\widetilde f}B_{\widetilde g}\|s^{(\rvt)}\|_{\mathcal H_{\vx,t}}^2$.
Since the synthesis operator of the tensor frame is bounded, the net of finite partial
sums of $\sum_{l,k}c_{l,k}f_l\otimes g_k$ converges to $s^{(\rvt)}$ independently of
order, so the expansion is unconditionally convergent.

{For completeness, the single temporal sum also converges in the product space.
Temporal reconstruction gives $\sum_{k=1}^n h_k(\vx)g_k\to s^{(\rvt)}(\vx,\cdot)$
in $\mathcal H_t$ for almost every $\vx$. The bounded synthesis operator for $\{g_k\}$
has norm at most $\sqrt{B_g}$. Applying its tail bound to each such slice and then
using Tonelli's theorem gives
\begin{equation}
\label{eq:temporal-frame-tail-bound}
\left\|s^{(\rvt)}-\sum_{k=1}^n h_k g_k\right\|_{\mathcal H_{\vx,t}}^2
\le B_g\sum_{k>n}\|h_k\|_{\mathcal H_{\vx}}^2
\longrightarrow0,
\end{equation}
where the convergence follows from \cref{eq:temporal-coefficient-summability}.}
Only the temporal frame enters this bound; the spatial frame is used solely to
identify $h_k$ with $\sum_l c_{l,k}f_l$ in the double expansion above.

Define $(TF)(\vx)=\int_0^1F(\vx,t)\,\mathrm dt$. Cauchy--Schwarz gives
\begin{equation}
\label{eq:product-integration-bound}
\|TF\|_{L^2(\mathcal X)}\le\|F\|_{\mathcal H_{\vx,t}}.
\end{equation}
Indeed, $|\int_0^1F(\vx,t)\,\mathrm dt|^2\le\int_0^1|F(\vx,t)|^2\,\mathrm dt$
for almost every $\vx$, and integration in $\vx$ gives the claim. This also shows
that $T$ is well-defined on $L^2$ equivalence classes.
Apply $T$ to each finite temporal partial sum and pass to the $L^2(\mathcal X)$ limit.
Since $Ts^{(\rvt)}=\log r$ by the path identity, this proves
\cref{eq:log-ratio-frame-decomposition-single-sum} in the stated norm.
\end{proof}

\subsection{Proof of Corollary~\ref{corollary:weak-derivative-of-time-score}}
\label{proof:weak-derivative-of-time-score}


\begin{proof}
Set $S_n(\vx,t)=\sum_{k=1}^n h_k(\vx)g_k(t)$ and
$V_n(\vx,t)=\sum_{k=1}^n h_k(\vx)g_k'(t)$, and let $v\in\mathcal H_{\vx,t}$ denote the
limit of the derivative series. {The theorem and the hypothesis give
$S_n\to s^{(\rvt)}$ and $V_n\to v$ in $\mathcal H_{\vx,t}$, respectively.} For every
$\psi\in C_c^\infty((0,1))$,
finite-sum weak integration by parts gives the identity in $L^2(\mathcal X)$
\begin{equation}
\label{eq:finite-sum-weak-derivative}
\int_0^1S_n(\cdot,t)\psi'(t)\,\mathrm dt
=-\int_0^1V_n(\cdot,t)\psi(t)\,\mathrm dt.
\end{equation}
Each term is justified by $h_k\in L^2(\mathcal X)$ and
$g_k,g_k'\in L^2([0,1])$. The scalar weak-derivative identity for $g_k$ can therefore
be multiplied by $h_k$ and summed over the finite index set. No boundary term
appears because $\psi$ is compactly supported in $(0,1)$.
{For any $a\in L^2([0,1])$, the map
$F\mapsto\int_0^1F(\cdot,t)a(t)\,\mathrm dt$ is bounded from
$\mathcal H_{\vx,t}$ to $L^2(\mathcal X)$ with norm at most $\|a\|_{L^2}$.
This follows by Cauchy--Schwarz in $t$ and integration in $\vx$, exactly as in
\cref{eq:product-integration-bound}, with $1$ replaced by $a$.}
Taking limits on both sides of \cref{eq:finite-sum-weak-derivative} therefore yields
\begin{equation}
\label{eq:limit-weak-derivative}
\int_0^1s^{(\rvt)}(\cdot,t)\psi'(t)\,\mathrm dt
=-\int_0^1v(\cdot,t)\psi(t)\,\mathrm dt.
\end{equation}
{Under the standard identification
$\mathcal H_{\vx,t}\cong L^2((0,1);\mathcal H_{\vx})$, this is precisely the defining
identity for the weak derivative of the $\mathcal H_{\vx}$-valued function
$t\mapsto s^{(\rvt)}(\cdot,t)$, with value $v$. Hence $\partial_t s^{(\rvt)}=v$
in $\mathcal H_{\vx,t}$.}

The convergence hypothesis holds, for example, when $\{g_k'\}$ is a Bessel sequence
with bound $B_{g'}$: for almost every $\vx$ the tail bound
$\|\sum_{k>n}h_k(\vx)g_k'\|_{\mathcal H_t}^2
\le B_{g'}\sum_{k>n}|h_k(\vx)|^2$ holds, and integrating over $\mathcal X$ gives
$\|\sum_{k>n}h_kg_k'\|_{\mathcal H_{\vx,t}}^2
\le B_{g'}\sum_{k>n}\|h_k\|_{\mathcal H_{\vx}}^2$,
which vanishes as $n\to\infty$ by \cref{eq:temporal-coefficient-summability}.
\end{proof}

\subsection{Proof of Lemma~\ref{lemma:rbf-density}}
\label{proof:rbf-density}

\begin{proof}
Fix $f\in L^2([0,1])$ and $\varepsilon>0$. Density of $C([0,1])$ in
$L^2([0,1])$ gives a function $q\in C([0,1])$ such that
$\|f-q\|_{L^2}<\varepsilon/3$. {By $c$-universality, there exist finitely many centers
$y_j\in[0,1]$ and coefficients $a_j$ such that
$p=\sum_{j=1}^m a_j\kappa_{\phi,\sigma_0}(\cdot,y_j)$ satisfies
$\|q-p\|_{C([0,1])}<\varepsilon/3$.} Since $[0,1]$ has measure one, the same bound
holds in $L^2$.

{For each $y_j$, choose $c_{j,K}\in\mathcal C_K$ with
$|c_{j,K}-y_j|\le h_K$, and define
\begin{equation}
\label{eq:finite-grid-approximant}
v_K=\sum_{j=1}^m a_j\kappa_{\phi,\sigma_0}(\cdot,c_{j,K})
\in\mathcal V_K(\boldsymbol\sigma_{0,1:K}).
\end{equation}

If several $y_j$ are mapped to the same grid center, their coefficients are added;
all other grid coefficients are set to zero. Thus $v_K$ is in the displayed
$K$-term space even when the chosen replacement centers are not distinct.}
Let $A=\sum_{j=1}^m|a_j|$. If $A=0$, then $p=v_K=0$.
{Otherwise, continuity of $\kappa_{\phi,\sigma_0}$ on the compact set $[0,1]^2$ implies
uniform continuity. Because $h_K\to0$ and $m$ is fixed, for all
sufficiently large $K$ each center replacement changes its kernel function by less than
$\varepsilon/(3A)$ uniformly on $[0,1]$.} Hence
$\|p-v_K\|_{C([0,1])}<\varepsilon/3$ and the triangle inequality gives
$\|f-v_K\|_{L^2}<\varepsilon$. This construction works for every sufficiently large $K$
and does not require nested spaces. {Finally, $\boldsymbol\sigma_{0,1:K}\in\mathcal A_K$, so the
fixed-scale space $\mathcal V_K(\boldsymbol\sigma_{0,1:K})$ is a member of the
trainable family, and hence the constructed $v_K$ belongs to the practical family with learnable
shape parameters and satisfies \cref{eq:rbf-finite-approximation}.}
\end{proof}

{\sloppy
\begin{corollary}[Joint Expressive Capacity of the Temporal Dictionary]
\label{corollary:joint-expressive-capacity}
Assume the hypotheses of Lemma~\ref{lemma:rbf-density}, with the same center sequence
and fixed reference scale $\sigma_0$. Fix a probability measure $\mu$ on $\mathcal X$, and let $\mu\otimes\mathrm dt$ be the product measure on
$\mathcal X\times[0,1]$. For each $K$, let $P_K$ be the orthogonal projector in
$\mathcal H_t=L^2([0,1])$ onto
$\operatorname{span}\{\kappa_{\phi,\sigma_0}(\cdot,c_k)\}_{k=1}^{K}$.
Then for every $s^{(\rvt)}\in L^2(\mu\otimes\mathrm dt)$,
\begin{equation}
\|s^{(\rvt)}-P_Ks^{(\rvt)}\|_{L^2(\mu\otimes\mathrm dt)}\longrightarrow0,
\end{equation}
and, with $T$ denoting temporal integration,
$\|T(s^{(\rvt)}-P_Ks^{(\rvt)})\|_{L^2(\mu)}\to0$.
In the Sobolev regime of Proposition~\ref{proposition:error-bound}, if additionally the squared
temporal Sobolev norm is $\mu$-integrable, then for all sufficiently large $K$,
\begin{equation}
\label{eq:joint-sobolev-capacity}
\|s^{(\rvt)}-P_Ks^{(\rvt)}\|_{L^2(\mu\otimes\mathrm dt)}^2
\le C^2K^{-2\beta}
\int\|s^{(\rvt)}(\vx,\cdot)\|_{\mathcal W^{\beta,2}([0,1])}^2\,\mu(\mathrm d\vx).
\end{equation}
Because $\boldsymbol\sigma_{0,1:K}\in\mathcal A_K$, the RBF family with global shape
parameters and unrestricted measurable coefficient functions has best error no larger
than this projection error. Restricting the coefficients to a neural class introduces
an additional approximation term, treated in the finite-sample extension.
\end{corollary}

\begin{proof}
For any $s^{(\rvt)}\in L^2(\mu\otimes\mathrm dt)$, Fubini's theorem gives
$s^{(\rvt)}(\vx,\cdot)\in\mathcal H_t$ for $\mu$-a.e.\ $\vx$. The slice map is
strongly measurable as an $\mathcal H_t$-valued function. To see this explicitly,
choose a countable orthonormal basis of the separable space $\mathcal H_t$.
Its scalar slice coefficients are measurable parameter integrals by Fubini;
the finite basis expansions are strongly measurable and converge in
$\mathcal H_t$ to the slice for $\mu$-a.e.\ $\vx$. As a bounded linear operator,
$P_K$ preserves this measurability. It is well-defined because its range is
finite-dimensional and hence closed.

The fixed-scale construction in the proof of Lemma~\ref{lemma:rbf-density}, applied to
$f=s^{(\rvt)}(\vx,\cdot)$, gives, for $\mu$-a.e.\ $\vx$,
$\|s^{(\rvt)}(\vx,\cdot)-P_Ks^{(\rvt)}(\vx,\cdot)\|_{\mathcal H_t}
=\operatorname{dist}(s^{(\rvt)}(\vx,\cdot),\operatorname{span}\{\kappa_{\phi,\sigma_0}(\cdot,c_k)\}_{k=1}^{K})\to0$.
The orthogonality of $P_Ku$ and $(I-P_K)u$ gives
$\|u\|_{\mathcal H_t}^2=\|P_Ku\|_{\mathcal H_t}^2+
\|(I-P_K)u\|_{\mathcal H_t}^2$. Consequently,
$\|s^{(\rvt)}(\vx,\cdot)-P_Ks^{(\rvt)}(\vx,\cdot)\|_{\mathcal H_t}^2\le\|s^{(\rvt)}(\vx,\cdot)\|_{\mathcal H_t}^2$.
By the convention $\mathcal H_t=L^2([0,1])$, the right-hand side equals
$\|s^{(\rvt)}(\vx,\cdot)\|_{L^2([0,1])}^2$, which is $\mu$-integrable because
$s^{(\rvt)}\in L^2(\mu\otimes\mathrm dt)$. Dominated convergence yields
$\|s^{(\rvt)}-P_Ks^{(\rvt)}\|_{L^2(\mu\otimes\mathrm dt)}^2\to0$.

Cauchy--Schwarz on the unit interval gives
$|Tf(\vx)|\le\|f(\vx,\cdot)\|_{L^2([0,1])}$, and hence $\|Tf\|_{L^2(\mu)}\le\|f\|_{L^2(\mu\otimes\mathrm dt)}$. Applying this
to $f=s^{(\rvt)}-P_Ks^{(\rvt)}$ transfers the joint $L^2$ consistency to the
log-density-ratio error.

For the family with unrestricted measurable coefficients, restricting the outer infimum to the single admissible
configuration $\boldsymbol\sigma_{0,1:K}\in\mathcal A_K$ and the inner infimum to the
fixed-scale coefficient functions gives the claimed upper bound; the per-$\vx$
$L^2([0,1])$ orthogonal projection achieves $\|s^{(\rvt)}-P_Ks^{(\rvt)}\|_{L^2(\mu\otimes\mathrm dt)}$,
so the bound is sharp within the fixed-scale span.
For completeness, the projection has measurable coefficient functions even if the
sections are linearly dependent. Choose an orthonormal basis
$e_1,\ldots,e_r$ of their span and write each $e_j$ as a fixed finite combination
of the sections. The expansion
$P_Ks^{(\rvt)}(\vx,\cdot)=\sum_{j=1}^r
\langle s^{(\rvt)}(\vx,\cdot),e_j\rangle_{\mathcal H_t}e_j$
then produces the required coefficients by finite linear combinations of measurable
scalar functions.

For the Sobolev rate, the fixed-scale interpolation bound of
\cref{eq:sobolev-interpolation-estimate}, proved below, first bounds the error of
the fixed-scale interpolant. Optimality of the $L^2$ projection makes its error
no larger. After substituting $h_K\le \rho_\ast K^{-1}$, where $\rho_\ast$ bounds the
mesh ratios uniformly in $K$, this gives
$\|s^{(\rvt)}(\vx,\cdot)-P_Ks^{(\rvt)}(\vx,\cdot)\|_{L^2([0,1])}^2\le C^2K^{-2\beta}\|s^{(\rvt)}(\vx,\cdot)\|_{\mathcal W^{\beta,2}([0,1])}^2$.
The constant and the threshold on $K$ depend only on the kernel, smoothness and
center geometry, so the same bound holds for all eligible slices.
Integrating under the assumed $\mu$-integrability of the squared Sobolev norm yields
\cref{eq:joint-sobolev-capacity}. By Cauchy--Schwarz in $t$ the same bound controls
$\|T(s^{(\rvt)}-P_Ks^{(\rvt)})\|_{L^2(\mu)}^2$.
\end{proof}
}

\subsection{Proof of Proposition~\ref{proposition:rbf-approximation-scheme}}
\label{proof:rbf-approximation-scheme}

\begin{proof}
Denseness follows from Lemma~\ref{lemma:rbf-density}.
Each fixed $\mathcal V_K(\boldsymbol\sigma_{1:K})$ is finite-dimensional and therefore
closed in $\mathcal H_t$, so its orthogonal projection exists and is unique.
Two coefficient vectors represent the same element precisely when their difference
belongs to the null space of the synthesis map $a\mapsto\sum_k a_k g_k$.
This map is injective if and only if the selected functions are linearly independent.

{For a common scale $\sigma$, distinct centers and a continuous strictly positive-definite
kernel, the Gram matrix $G_{ij}=\kappa_{\phi,\sigma}(c_i,c_j)$ is positive definite.
If $\sum_k a_k\kappa_{\phi,\sigma}(\cdot,c_k)=0$ in $L^2([0,1])$, continuity makes
this equality hold everywhere on $[0,1]$. Evaluation at the centers gives $Ga=0$,
so $a=0$. This proves the sufficient condition.}
\end{proof}

\subsection{Proof of Lemma~\ref{lemma:qualitative-approximation-consistency}}
\label{proof:qualitative-approximation-consistency}

\begin{proof}
Lemma~\ref{lemma:rbf-density} shows that the best attainable error over the $K$-term family
tends to zero. {For any $\varepsilon>0$, the lemma bounds that infimum by
$\varepsilon$ for every sufficiently large $K$. The error of the chosen
approximation is nonnegative and no larger than this infimum plus $\eta_K$.}
The squeeze theorem and $\eta_K\to0$ then prove
\cref{eq:qualitative-approximation-consistency}.
Since $\mathcal H_t=L^2([0,1])$, Cauchy--Schwarz yields
\begin{equation}
|\log r(\vx)-\log r_K(\vx)|
\le
\|s^{(\rvt)}(\vx,\cdot)-s_K^{(\rvt)}(\vx,\cdot)\|_{\mathcal H_t},
\end{equation}
which proves \cref{eq:qualitative-logratio-consistency}.
\end{proof}

\subsection{Proof of Proposition~\ref{proposition:error-bound}}
\label{proof:error-bound}

\begin{proof}
Fix $\vx$ and write $f=s^{(\rvt)}(\vx,\cdot)$. We first bound the error at the
admissible common scale $\sigma_0$, then transfer the bound to the trainable family.

\textit{Step 1: Reduce to fixed-scale interpolation.}
Let $f_{K,\sigma_0}$ be the $L^2([0,1])$ orthogonal projection of $f$ onto
$\mathcal V_K(\boldsymbol\sigma_{0,1:K})$. {Quasi-uniformity gives
$\operatorname{sep}_K>0$, so the centers are distinct, and the interpolation matrix of a
strictly positive definite kernel at distinct centers is invertible
\citep{schaback1995error}; the interpolant
$I_{\mathcal C_K}^{\sigma_0}f$ therefore exists and belongs to the same space.} Projection
optimality gives
\begin{equation}
\label{eq:sobolev-best-approx}
\|f-f_{K,\sigma_0}\|_{\mathcal H_t}
\le\|f-I_{\mathcal C_K}^{\sigma_0}f\|_{\mathcal H_t}.
\end{equation}

\textit{Step 2: Apply the Sobolev interpolation estimate.}
The extension theorem on the bounded Lipschitz interval $[0,1]$
\citep{adams2003sobolev,dinezza2012hitchhiker} provides a bounded linear
$E:\mathcal W^{\beta,2}([0,1])\to\mathcal W^{\beta,2}(\mathbb R)$. Set
$\widetilde f=Ef$, so
$\|\widetilde f\|_{\mathcal W^{\beta,2}(\mathbb R)}
\le C_{\mathrm{ext}}\|f\|_{\mathcal W^{\beta,2}([0,1])}$.
{Since $\beta>1/2$, continuous representatives exist; equality almost everywhere
between $f$ and $\widetilde f$ on $[0,1]$ is therefore equality everywhere,
including at the centers. Their interpolants consequently coincide.}

Use \citet[Theorem~4.2]{narcowich2006sobolev} with its sampling step refined by
\citet[Theorem~4.1]{arcangeli2007extension}, as restated in
\citet[Section~4, Eq.~(13) and the following discussion]{tuo2020kriging}.
{This refinement replaces the additional condition $\lfloor\beta\rfloor>d/2$
by $\beta>d/2$. Its hypotheses hold here: $d=1$, $1/2<\beta\le\tau$,
the interval satisfies the interior cone condition, and the kernel's native
norm is equivalent to $\mathcal W^{\tau,2}(\mathbb R)$.} Taking error-norm
order $\mu=0$ and absorbing the extension and norm-equivalence constants gives,
whenever $h_K$ is sufficiently small,
\begin{equation}
\label{eq:sobolev-interpolation-estimate}
\|f-I_{\mathcal C_K}^{\sigma_0}f\|_{L^2([0,1])}
\le A h_K^\beta
\rho_K^{\tau-\beta}
\|f\|_{\mathcal W^{\beta,2}([0,1])},
\end{equation}
where $A$ may depend on $\phi$, $\sigma_0$, $\beta$, $\tau$ and $[0,1]$,
but not on $K$ or $f$.

\textit{Step 3: Convert fill distance to dictionary size.}
Let $\rho_\ast$ bound the mesh ratios uniformly in $K$. Among $K\ge2$ ordered
centers in $[0,1]$, at least one consecutive gap is at most $1/(K-1)$, so
$\operatorname{sep}_K\le1/[2(K-1)]$. Consequently,
$h_K=\rho_K\operatorname{sep}_K\le\rho_\ast/[2(K-1)]\le\rho_\ast K^{-1}$.
Thus $h_K\to0$, so Step 2 applies for all sufficiently large $K$.
Combining \cref{eq:sobolev-best-approx,eq:sobolev-interpolation-estimate} gives
$\|f-f_{K,\sigma_0}\|_{\mathcal H_t}\le
A\rho_\ast^\tau K^{-\beta}\|f\|_{\mathcal W^{\beta,2}([0,1])}$.
Finally, $\boldsymbol\sigma_{0,1:K}\in\mathcal A_K$ makes this projection an
admissible competitor for the infimum in \cref{eq:sobolev-algebraic-rate}.
The claimed bound follows with $C=A\rho_\ast^\tau$, independent of $K$ and $f$.
\end{proof}

\subsection{Finite-Sample Error Bounds}
\label{appendix:finite-sample-oracle}
We give a conditional oracle inequality that separates representation, generalization
and optimization errors for a calibrated scalar time-score objective. The assumptions
below specify the statistical setting.

\paragraph{Statistical Setting and Assumptions.}
Let $P$ be the population measure defining the score-matching error norm for the time score
$s^{(\rvt)}(\vx,t)=\partial_t\log p_t(\vx)\in L^2(P)$, and let $\mu$ be the evaluation
measure for density-ratio queries on samples. For a coefficient class $\mathcal F_K$, define
\begin{equation}
\label{eq:composed-model-class}
\mathcal S_{K,\mathcal F}
\triangleq
\left\{
\sum_{k=1}^{K}h_k(\vx)\kappa_{\phi,\sigma_k}(t,c_k):
h=(h_1,\ldots,h_K)\in\mathcal F_K,\ 
\boldsymbol\sigma_{1:K}\in\mathcal A_K
\right\}.
\end{equation}

The shape parameters are shared across $\vx$. Let $\mathcal L$ and $\mathcal L_N$ be the
population and empirical scalar time-score objectives based on $N$ i.i.d. training tuples
on $[0,1]$. Assume that all model scores belong to $L^2(P)$, that the model class is
nonempty, and that the population objectives below are finite. The two conditions below
connect scalar excess risk to the log-ratio error.

\begin{assumption}[Population Risk Calibration]
\label{assumption:risk-calibration}
The population objective is two-sided calibrated on $\mathcal S_{K,\mathcal F}$: there exist
constants $0<c_-\le c_+<\infty$ such that for every $f\in\mathcal S_{K,\mathcal F}$,
\begin{equation}
\label{eq:risk-calibration}
c_-\|f-s^{(\rvt)}\|_{L^2(P)}^2
\le
\mathcal L(f)-\mathcal L(s^{(\rvt)})
\le
c_+\|f-s^{(\rvt)}\|_{L^2(P)}^2.
\end{equation}
\end{assumption}

{For the scalar score-matching objectives, the standard excess-risk identity
\citep{choi2022density,yu2025density} yields this condition from the loss weights.
The paragraph ``Calibrated Scalar Objectives'' below gives the identity and its
constants explicitly.}

\begin{assumption}[Integration Norm Compatibility]
\label{assumption:integration-norm}
{Let $(Tf)(\vx)\triangleq\int_0^1f(\vx,t)\,\mathrm dt$.
For every $f\in\mathcal S_{K,\mathcal F}$, temporal integration of the score error satisfies
$\|T(f-s^{(\rvt)})\|_{L^2(\mu)}\le C_T\|f-s^{(\rvt)}\|_{L^2(P)}$
for a constant $C_T<\infty$.}
\end{assumption}

{A sufficient condition gives an explicit integration constant. Suppose
$P$ has a positive density $q$ relative to $\mu\otimes\mathrm dt$.
If the following essential supremum is finite, one may take
\begin{equation}
\label{eq:integration-density-constant}
C_T^2=\operatorname*{ess\,sup}_{\vx\sim\mu}
\int_0^1\frac{1}{q(\vx,t)}\,\mathrm dt.
\end{equation}
For any eligible error $e$, weighted Cauchy--Schwarz gives
\begin{equation}
\label{eq:weighted-integration-error}
|Te(\vx)|^2\le
\left(\int_0^1|e(\vx,t)|^2q(\vx,t)\,\mathrm dt\right)
\left(\int_0^1\frac{1}{q(\vx,t)}\,\mathrm dt\right).
\end{equation}
Integration in $\mu$ proves the required inequality. In particular, $q\equiv1$ gives
$C_T=1$, and $q\ge1/\kappa$ gives $C_T\le\sqrt\kappa$.
The reciprocal-integral criterion also permits $q$ to approach zero where its reciprocal
remains integrable.}

{\paragraph{Empirical Optimization Error.}
For a returned estimator $\widehat s\in\mathcal S_{K,\mathcal F}$ and a finite empirical
infimum, define
\begin{equation}
\label{eq:empirical-optimization-gap}
\varepsilon_{\rm opt}\triangleq\mathcal L_N(\widehat s)
-\inf_{f\in\mathcal S_{K,\mathcal F}}\mathcal L_N(f)\ge0.
\end{equation}
Thus $\mathcal L_N(\widehat s)\le\mathcal L_N(f)+\varepsilon_{\rm opt}$ for every
$f\in\mathcal S_{K,\mathcal F}$. The optimization term measures the remaining empirical
loss gap within the chosen model class and enters the total error bound additively.}

\paragraph{Oracle Inequality.}
The error decomposes into four terms. Define the temporal approximation term
\begin{equation}
\label{eq:oracle-bias-terms}
A_K\triangleq
\inf_{\boldsymbol\sigma_{1:K}\in\mathcal A_K}
\inf_{h_1,\ldots,h_K}
\left\|
s^{(\rvt)}(\vx,t)-\sum_{k=1}^{K}h_k(\vx)\kappa_{\phi,\sigma_k}(t,c_k)
\right\|_{L^2(P)},
\end{equation}
where the inner infimum ranges over unrestricted measurable coefficient functions for which
the displayed function lies in $L^2(P)$. For any $\eta>0$, choose an admissible global vector
$\boldsymbol\sigma^{(\eta)}_{1:K}$ and unrestricted coefficients with error at most $A_K+\eta$.
Denote their resulting time-score function by $s^{(\rvt)}_{K,\eta}$, and define the neural
coefficient approximation term at those shape parameters:
\begin{equation}
\label{eq:neural-coefficient-bias}
B_{K,\mathcal F,\eta}\triangleq
\inf_{h\in\mathcal F_K}
\left\|
s^{(\rvt)}_{K,\eta}(\vx,t)-\sum_{k=1}^{K}h_k(\vx)
\kappa_{\phi,\sigma_k^{(\eta)}}(t,c_k)
\right\|_{L^2(P)}.
\end{equation}

{Write $\ell_f$ for the per-tuple loss of model $f$.
Fix a reference model $f_0\in\mathcal S_{K,\mathcal F}$ independently of the
training tuples. Define the random uniform deviation of loss differences
\begin{equation}
\label{eq:empirical-process-deviation}
\Gamma_N(\mathcal S_{K,\mathcal F},\delta)
\triangleq
\sup_{f\in\mathcal S_{K,\mathcal F}}
\left|\{\mathcal L_N(f)-\mathcal L_N(f_0)\}
-\{\mathcal L(f)-\mathcal L(f_0)\}\right|
\end{equation}
where $\delta$ indexes the confidence bounds below. Subtracting the same reference loss
preserves every empirical and population loss comparison. Under
Assumptions~\ref{assumption:risk-calibration}--\ref{assumption:integration-norm}, we obtain
\cref{eq:oracle-inequality} as follows.}

The four terms correspond to temporal approximation, neural coefficient approximation,
score-matching generalization and empirical optimization error.

\begin{proof}
For any $f\in\mathcal S_{K,\mathcal F}$, the calibration condition,
\cref{eq:empirical-optimization-gap} and the definition of $\Gamma_N$ give
\begin{align}
\label{eq:score-error-derivation}
c_-\|\widehat s-s^{(\rvt)}\|_{L^2(P)}^2
&\le \mathcal L(\widehat s)-\mathcal L(s^{(\rvt)}) \nonumber\\
&=\{\mathcal L(\widehat s)-\mathcal L_N(\widehat s)-\mathcal L(f_0)+\mathcal L_N(f_0)\}
+\{\mathcal L_N(\widehat s)-\mathcal L_N(f)\} \nonumber\\
&\quad
+\{\mathcal L_N(f)-\mathcal L(f)-\mathcal L_N(f_0)+\mathcal L(f_0)\}
+\{\mathcal L(f)-\mathcal L(s^{(\rvt)})\} \nonumber\\
&\le
2\Gamma_N+\varepsilon_{\rm opt}
+c_+\|f-s^{(\rvt)}\|_{L^2(P)}^2.
\end{align}
For any
$\zeta>0$, the definition of $B_{K,\mathcal F,\eta}$ gives
$f_\zeta\in\mathcal S_{K,\mathcal F}$ at the chosen global shape parameters with
$\|f_\zeta-s^{(\rvt)}_{K,\eta}\|_{L^2(P)}\le B_{K,\mathcal F,\eta}+\zeta$.
The triangle inequality and $(a+b)^2\le2a^2+2b^2$ imply
\begin{equation}
\label{eq:oracle-competitor-error}
\|f_\zeta-s^{(\rvt)}\|_{L^2(P)}^2
\le 2(A_K+\eta)^2+2(B_{K,\mathcal F,\eta}+\zeta)^2.
\end{equation}
Substituting $f_\zeta$ into \cref{eq:score-error-derivation} and letting
$\zeta\downarrow0$ proves
\begin{equation}
\label{eq:score-error-bound}
\|\widehat s-s^{(\rvt)}\|_{L^2(P)}^2
\le
\frac{2c_+\{(A_K+\eta)^2+B_{K,\mathcal F,\eta}^2\}
+2\Gamma_N(\mathcal S_{K,\mathcal F},\delta)
+\varepsilon_{\rm opt}}{c_-}.
\end{equation}
Assumption~\ref{assumption:integration-norm} then transfers this score-error bound to the log-ratio error because
$\log\widehat r-\log r=T(\widehat s-s^{(\rvt)})$.
\end{proof}

\paragraph{Temporal Approximation Rate.}
The fixed-scale theory bounds $A_K$ because
$\boldsymbol\sigma_{0,1:K}\in\mathcal A_K$ is one admissible global choice. To integrate the
pointwise rates of Propositions~\ref{proposition:error-bound} and~
\ref{proposition:error-bound-infinitely-smooth} into an $L^2(P)$ bound on $A_K$, assume
$P(\mathrm d\vx,\mathrm dt)\le p_{\max}\mu(\mathrm d\vx)\,\mathrm dt$ and integrability
of the squared temporal Sobolev or native-space norm, respectively.
To make this transfer explicit, let $Q_Ks^{(\rvt)}$ denote the fixed-scale
temporal $L^2([0,1])$ projection, applied slice by slice. In the Sobolev case,
\begin{equation}
\label{eq:population-temporal-bias-bound}
\begin{aligned}
A_K^2
&\le\|s^{(\rvt)}-Q_Ks^{(\rvt)}\|_{L^2(P)}^2\\
&\le p_{\max}\|s^{(\rvt)}-Q_Ks^{(\rvt)}\|_{L^2(\mu\otimes\mathrm dt)}^2\\
&\le p_{\max}C^2K^{-2\beta}
\int\|s^{(\rvt)}(\vx,\cdot)\|_{\mathcal W^{\beta,2}([0,1])}^2\,\mu(\mathrm d\vx).
\end{aligned}
\end{equation}
The first inequality uses the admissible common scale and the measurable projection
coefficients constructed in the joint-capacity proof. The projection belongs to
$L^2(P)$ because the displayed error is finite and $s^{(\rvt)}\in L^2(P)$.
For the infinitely smooth cases, the same argument uses the fixed-scale
interpolant as a competitor for this projection and integrates its squared
native-space error bound. This constructs a single global scale choice before
integration; no exchange of a pointwise infimum with an integral is needed.

\paragraph{Sample Size and Temporal Capacity.}
The remaining term is the generalization error. We bound it by a parameter-covering
argument at a fixed confidence.
\begin{corollary}[Sample Size and Temporal Capacity]
\label{corollary:sample-size-capacity}
Let Assumptions~\ref{assumption:risk-calibration}--\ref{assumption:integration-norm} hold.
Fix $N\ge1$ and $0<\delta<1$, and fix the model class before drawing the training
tuples. Assume the per-tuple loss is measurable in the tuple for every parameter vector.
Let $d_K$ count the scalar parameters of $\mathcal S_{K,\mathcal F}$, including the $K$
learned log-shape parameters. Assume the parameter vectors lie in $[-R_K,R_K]^{d_K}$.
The per-tuple loss is uniformly $L_K$-Lipschitz in the parameters under the
$\ell^\infty$ norm, with $L_K\ge1$.
Write $M_K$ for a common interval length containing the loss differences
$\ell_f-\ell_{f_0}$ over all models and training tuples. Such an interval exists with
$M_K=4R_KL_K$; a tighter bound may be used when available.
Then, with probability at least $1-\delta$,
\begin{equation}
\label{eq:parameter-generalization}
\Gamma_N(\mathcal S_{K,\mathcal F},\delta)
\le \frac{2}{N}
+M_K\sqrt{\frac{d_K\log(1+2R_K L_K N)+\log(2/\delta)}{2N}}.
\end{equation}
In the Sobolev setting of \cref{eq:joint-sobolev-capacity}, assume also
$P\le p_{\max}\mu\otimes\mathrm dt$. If $p_{\max}$, the calibration and integration
constants, and the integrated squared Sobolev norm are independent of $K,N$, then
\begin{equation}
\label{eq:sample-size-capacity-rate}
\begin{aligned}
\|\log\widehat r-\log r\|_{L^2(\mu)}^2
\le C\biggl[&K^{-2\beta}+B_{K,\mathcal F,\eta}^2+\eta^2
+\varepsilon_{\rm opt}+\frac{1}{N}\\
&+M_K\sqrt{\frac{d_K\log(1+2R_K L_K N)+\log(2/\delta)}{N}}\biggr],
\end{aligned}
\end{equation}
where $C$ is independent of $K,N,\eta,\delta$.
\end{corollary}

\begin{proof}
Since both parameter vectors lie in $[-R_K,R_K]^{d_K}$, the Lipschitz condition gives
$|\ell_f-\ell_{f_0}|\le2R_KL_K$ for every tuple. Thus the loss differences lie in
$[-2R_KL_K,2R_KL_K]$, establishing the stated interval length from parameter regularity.
Partition the parameter box into cells of side length $1/(L_KN)$ and choose one
admissible representative from each nonempty cell. There are at most
$(1+2R_K L_KN)^{d_K}$ representatives: each coordinate interval can be partitioned
into at most $\max\{1,\lceil2R_KL_KN\rceil\}$ intervals of the prescribed maximum
length. Let $J$ denote the number of nonempty cells. The $\ell^\infty$ distance
to a representative is at most $1/(L_KN)$, so the Lipschitz assumption bounds the
loss difference by $1/N$ for every tuple, and hence for both its empirical and
population averages. Apply Hoeffding's inequality \citep{hoeffding1963probability}
to $\ell_{f_j}-\ell_{f_0}$ for each representative $f_j$. It gives,
when $M_K>0$,
$\Pr(|\{\mathcal L_N(f_j)-\mathcal L_N(f_0)\}-\{\mathcal L(f_j)-\mathcal L(f_0)\}|>u)
\le2\exp(-2Nu^2/M_K^2)$.
A union bound therefore yields
\begin{equation}
\label{eq:finite-net-deviation}
\Pr\!\left(\max_{1\le j\le J}
\left|\{\mathcal L_N(f_j)-\mathcal L_N(f_0)\}
-\{\mathcal L(f_j)-\mathcal L(f_0)\}\right|>u\right)
\le2J\exp(-2Nu^2/M_K^2).
\end{equation}
Set $u=M_K\sqrt{\log(2J/\delta)/(2N)}$ and use the bound on $J$.
For any model $f$ and its representative $f_j$,
the deviation of its loss difference is bounded by the corresponding deviation at $f_j$
plus $2/N$.
Taking the supremum proves \cref{eq:parameter-generalization}. If $M_K=0$,
all loss differences are zero and so is the deviation, giving the same bound.
The supremum is measurable: the admissible parameter set is a separable subset
of a finite-dimensional box, and the Lipschitz assumption makes the deviation
continuous in its parameters, so its supremum equals one over a countable dense subset.
Substituting this bound and
$A_K^2\le CK^{-2\beta}$ from \cref{eq:population-temporal-bias-bound} into
\cref{eq:oracle-inequality}, and using $(A_K+\eta)^2\le2A_K^2+2\eta^2$,
proves \cref{eq:sample-size-capacity-rate}.
\end{proof}

\paragraph{Neural Class and Choice of $K$.}
For a shared backbone with $d_{\mathrm{back}}$ trainable parameters, a linear
$d_{\mathrm{feat}}$-to-$K$ coefficient head with biases, and $K$ learned shape
parameters, the count in the corollary is
\begin{equation}
\label{eq:capacity-parameter-count}
d_K=d_{\mathrm{back}}+(d_{\mathrm{feat}}+2)K.
\end{equation}
The coefficient class enters through $d_K$ and $B_{K,\mathcal F,\eta}$. The same admissible
shape range, including $\sigma_0$, is used in $A_K$ and the model class. Neural coefficient
approximation and empirical minimization both use $\mathcal F_K$.
The radius $R_K$, loss-difference interval length $M_K$ and parameter Lipschitz constant $L_K$ may grow
with $K$; their dependence is retained in \cref{eq:parameter-generalization}.

For fixed backbone size and feature dimension, suppose $R_KL_K$ grows at most
polynomially with $K$ and $M_K=\mathcal O(K^a)$ for some $a\ge0$. For $2\le K\le N$
and fixed confidence, the complexity term is then
$\mathcal O(K^a\sqrt{K\log N/N})$. Balancing this term with $K^{-2\beta}$ gives
the relation $K^{2\beta+a+1/2}\asymp\sqrt{N/\log N}$, and hence
\begin{equation}
\label{eq:capacity-sample-balance}
K\asymp\left(\frac{N}{\log N}\right)^{1/(4\beta+2a+1)},
\qquad
K^{-2\beta}+K^a\sqrt{\frac{K\log N}{N}}
=\mathcal O\left(\left(\frac{\log N}{N}\right)^{2\beta/(4\beta+2a+1)}\right).
\end{equation}
For uniformly bounded loss differences ($a=0$), this reduces to
$K\asymp(N/\log N)^{1/(4\beta+1)}$. Consistency of the complete estimator also requires
$B_{K,\mathcal F,\eta}\to0$, $\eta\to0$, and
$\varepsilon_{\rm opt}\to0$. More generally, a growing coefficient network is allowed
provided, for some $\delta_N\to0$,
\begin{equation}
\label{eq:capacity-consistency-condition}
K\to\infty,\qquad
\frac{M_K^2\{d_K\log(1+2R_K L_KN)+\log(2/\delta_N)\}}{N}\to0,
\end{equation}
with the same uniform calibration and integration constants and vanishing residual terms.
This is a sufficient balance of two terms in the upper bound.

\paragraph{Calibrated Scalar Objectives.}
The scalar score-matching excess-risk identity follows the standard equivalence in
\citet{choi2022density}. For conditional time-score regression \citep{yu2025density},
let $P$ be the law of $(\vx_t,t)$ and $Y$ the conditional score target, with
$\mathbb E[Y\mid\vx_t=\vx,t]=s^{(\rvt)}(\vx,t)$ and finite second moments.
For a weight
$0<w_-\le w(\vx,t)\le w_+<\infty$ measurable in $(\vx,t)$,
the identity, with $s=s^{(\rvt)}$, is
\begin{equation}
\label{eq:conditional-calibration-identity}
\begin{aligned}
\mathcal L(f)-\mathcal L(s)
&=\mathbb E\bigl[w\{(f-Y)^2-(s-Y)^2\}\bigr]\\
&=\mathbb E[w(f-s)^2].
\end{aligned}
\end{equation}
Finite second moments and the bounded weight justify the expectations and the
identity. The bounds on $w$ give the two calibration
inequalities with $c_-=w_-$ and $c_+=w_+$.
For the unweighted objective, $c_-=c_+=1$. Thus the calibration constants are obtained
from the scalar objective rather than imposed as additional model properties.
For the scalar time-score-matching objective with uniform time sampling
\citep{choi2022density}, this calibration identity can also be derived directly.
Let $\lambda\in C^1([0,1])$ be a positive deterministic weight. Suppose that
$p_t(\vx)$ and each eligible $f(\vx,t)$, including $s(\vx,t)$, are absolutely
continuous in $t$ for almost every $\vx$, that the derivative and boundary terms
below are absolutely integrable, and that the weighted square integrals are finite.
Write $\mathcal B_\lambda(f)=\int_{\mathcal X}
[\lambda(t)p_t(\vx)f(\vx,t)]_{t=0}^{t=1}\,\mathrm d\vx$.
The scalar part of \cref{eq:joint-score-matching-loss} is then
\begin{equation}
\label{eq:scalar-time-score-matching-objective}
\mathcal L_{\rm SM}(f)
=\int_{\mathcal X}\int_0^1
p_t(\vx)\{\lambda f^2+2\lambda\partial_t f+2\lambda' f\}
\,\mathrm dt\,\mathrm d\vx-2\mathcal B_\lambda(f).
\end{equation}
The product rule and $\partial_t p_t=p_ts$ imply
$\partial_t(\lambda p_tf)=\lambda'p_tf+\lambda p_tsf+
\lambda p_t\partial_t f$. Integrating this identity in $t$ and then in $\vx$
is justified by the stated absolute integrability and gives
\begin{equation}
\label{eq:time-score-calibration-by-parts}
\begin{aligned}
\mathcal L_{\rm SM}(f)
&=\int_{\mathcal X}\int_0^1\lambda p_t(f^2-2fs)\,\mathrm dt\,\mathrm d\vx,\\
\mathcal L_{\rm SM}(f)-\mathcal L_{\rm SM}(s)
&=\int_{\mathcal X}\int_0^1\lambda p_t(f-s)^2\,\mathrm dt\,\mathrm d\vx.
\end{aligned}
\end{equation}
Thus $0<\lambda_-\le\lambda(t)\le\lambda_+<\infty$ gives calibration with
$c_-=\lambda_-$ and $c_+=\lambda_+$. The boundary correction in the objective
is essential to this cancellation unless its boundary functional vanishes.

\subsection{Proof of Proposition~\ref{proposition:error-bound-infinitely-smooth}}
\label{proof:error-bound-infinitely-smooth}

\begin{proof}
Fix $\vx$, write $f\triangleq s^{(\rvt)}(\vx,\cdot)$, and let $f_{K,\sigma_0}$ be its best
approximation in $\mathcal V_K(\boldsymbol\sigma_{0,1:K})$. Since
$I_{\mathcal C_K}^{\sigma_0}f\in\mathcal V_K(\boldsymbol\sigma_{0,1:K})$,
\begin{equation}
\label{eq:inf-smooth-best-approx}
\|f-f_{K,\sigma_0}\|_{\mathcal H_t}
\le\|f-I_{\mathcal C_K}^{\sigma_0}f\|_{\mathcal H_t}.
\end{equation}
Moreover, $\boldsymbol\sigma_{0,1:K}\in\mathcal A_K$, so the fixed-scale comparison space is one member
of the practical family and
\begin{equation}
\label{eq:inf-smooth-learnable-comparison}
\mathcal E_K(\vx)\le\|f-f_{K,\sigma_0}\|_{\mathcal H_t}.
\end{equation}

{The domain $[0,1]$ is a compact one-dimensional cube, as required by
\citet[Theorem~6.1]{rieger2010sampling}. Its restricted native-space norm agrees
with the minimum-extension norm in \cref{eq:restricted-native-space-manuscript}.
Quasi-uniformity gives $\operatorname{sep}_K>0$, so the centers are distinct, and the
interpolation matrices of the strictly positive definite Gaussian and IMQ kernels are
invertible \citep{schaback1995error}; setting the regularization parameter to
zero in the cited theorem therefore produces the interpolant used here.}

{For the Gaussian kernel, take $q=2$, $\Omega=[0,1]$, and Gaussian scale
$c=\sigma_0^{-1}$ in the kernel convention of that reference. Its Theorem~6.1,
with regularization parameter zero, gives
\begin{equation}
\label{eq:gaussian-interpolation-detail}
\|f-I_{\mathcal C_K}^{\sigma_0}f\|_{L^2([0,1])}
\le 2\exp(c_G\log(h_K)/h_K)
\|f\|_{\mathcal N_{\phi,\sigma_0}([0,1])}
\end{equation}
for a constant $c_G>0$ and sufficiently small $h_K$. Decrease the threshold to
$h_0<1$ if necessary; then $\log h_K=-|\log h_K|$.} Absorbing the factor $2$
into $C$ gives the right-hand side of \cref{eq:gaussian-h-bound}. Together with
\cref{eq:inf-smooth-best-approx,eq:inf-smooth-learnable-comparison}, this gives
\begin{equation}
\label{eq:gaussian-h-bound}
\mathcal E_K(\vx)
\le C\exp\left(-c\frac{|\log h_K|}{h_K}\right)
\|s^{(\rvt)}(\vx,\cdot)\|_{\mathcal N_{\phi,\sigma_0}([0,1])}.
\end{equation}
If
$h_K\le c_\ast K^{-1}$, {choose $K$ large enough that $h_K<1$ and
$\log(K/c_\ast)\ge\tfrac12\log K$. Then
$|\log h_K|\ge\log(K/c_\ast)\ge\tfrac12\log K$ and
$1/h_K\ge K/c_\ast$. Multiplication gives
$|\log h_K|/h_K\ge K\log K/(2c_\ast)$,} which yields
\cref{eq:gaussian-k-bound} with $\widetilde C=C$ and $\widetilde c=c/(2c_\ast)$.

For standard IMQ, the verification in \cref{appendix:imq-rbfs} establishes
Assumption~\ref{assumption:fixed-scale-exponential-interpolation} for every target in
the specified restricted native space. Combining that interpolation estimate with
\cref{eq:inf-smooth-best-approx,eq:inf-smooth-learnable-comparison} bounds
$\mathcal E_K(\vx)$. Since $\lambda\in(0,1)$ and $1/h_K\ge K/c_\ast$,
$\lambda^{1/h_K}\le\exp[-|\log\lambda|K/c_\ast]$, proving
\cref{eq:imq-k-bound}.

For generalized IMQ, \citet[Theorem~6.1]{rieger2010sampling} gives the corresponding
interpolation bound with $q=2$ and its regularization parameter set to zero. Its kernel
condition is $\alpha>d/2$, which becomes $\alpha>1/2$ on the one-dimensional temporal
domain. {Specifically, the reference uses $(\sigma_0^2+|s-t|^2)^{-\alpha}$,
whereas our kernel is $\sigma_0^{2\alpha}$ times this kernel. Multiplying a kernel
by a fixed positive factor preserves its interpolation operator and rescales its
native norm by the inverse square root of that factor. The resulting fixed factors
are absorbed into $C$.} Thus the cited bound reads
$\|f-I_{\mathcal C_K}^{\sigma_0}f\|_{L^2([0,1])}
\le C\exp(-c/h_K)\|f\|_{\mathcal N_{\phi,\sigma_0}([0,1])}$.
Combining this bound with
\cref{eq:inf-smooth-best-approx,eq:inf-smooth-learnable-comparison} proves
the fill-distance estimate
\begin{equation}
\label{eq:gimq-h-bound}
\mathcal E_K(\vx)
\le C\exp(-c/h_K)
\|s^{(\rvt)}(\vx,\cdot)\|_{\mathcal N_{\phi,\sigma_0}([0,1])}.
\end{equation}
The relation $1/h_K\ge K/c_\ast$ then gives
\cref{eq:gimq-k-bound} with $\widetilde C=C$ and $\widetilde c=c/c_\ast$.
\end{proof}

\vskip 0.2in
\bibliography{sample}

\end{document}